\ifdefined\useorstyle

\documentclass[mnsc,blindrev]{informs3}

\usepackage{packages-or}
\usepackage{editing-macros}
\usepackage{statistics-macros-or}
\usepackage{formatting}
\usepackage[normalem]{ulem} % for \sout

\DoubleSpacedXI % Made default 4/4/2014 at request
\usepackage{endnotes}
\let\footnote=\endnote

\usepackage{natbib}
 \bibpunct[, ]{(}{)}{,}{a}{}{,}%
 \def\bibsep{\smallskipamount}%
\TheoremsNumberedThrough     % Preferred (Theorem 1, Lemma 1, Theorem 2)
\ECRepeatTheorems

\EquationsNumberedThrough    % Default: (1), (2), ...
\usepackage{hyperref}
\usepackage{./statistics-macros-or}
\usepackage{soul} % Make sure this is in your

\usepackage{pgfplotstable}
\usepackage{graphicx}
\usepackage{subcaption}
\usepackage{float}
 \usepackage{bm}

\usepackage{algorithm}
\usepackage{algpseudocode}
\usepackage{tabularx}

\usepackage{overpic}
\usepackage{tikz}
\usepackage{rotating}
\usepackage{psfrag}
\usepackage{enumitem}

\usepackage{listings}
\usepackage{xcolor} % optional but useful

\begin{document}
%%%%%%%%%%%%%%%%

%\setcounter{page}{1}
%\DOI{10.1287/xxxx.0000.0000}%


\RUNTITLE{A Practical Minimax Approach to Causal Inference with Limited Overlap}

% \TITLE{Bundling Information Goods of Decreasing Value}
\TITLE{A Practical Minimax Approach to Causal Inference with Limited Overlap}

%} % end of the block

\ABSTRACT{Limited overlap between treated and control groups is a key challenge in
observational analysis.  Standard approaches like trimming importance weights
can reduce variance but introduce a fundamental bias.  We propose a
sensitivity framework for contextualizing findings under limited overlap, where
we assess how irregular the outcome function has to be in order for the main
finding to be invalidated. Our approach is based on worst-case confidence
bounds on the bias introduced by standard trimming practices, under explicit
assumptions necessary to extrapolate counterfactual estimates from regions of
overlap to those without.  Empirically, we demonstrate how our sensitivity
framework protects against spurious findings by quantifying uncertainty in
regions with limited overlap.

%%% Local Variables:
%%% mode: latex
%%% TeX-master: "main"
%%% End:
}

%\KEYWORDS{deterministic inventory theory; infinite linear programming duality;
%  existence of optimal policies; semi-Markov decision process; cyclic schedule}

% \HISTORY{This paper was first submitted on July,
%   2020.}

\maketitle
%%%%%%%%%%%%%%%%%%%%%%%%%%%%%%%%%%%%%%%%%%%%%%%%%%%%%%%%%%%%%%%%%%%%%%

% Samples of sectioning (and labeling) in MNSC
% NOTE: (1) \section and \subsection do NOT end with a period
%       (2) \subsubsection and lower need end punctuation
%       (3) capitalization is as shown (title style).
%
%\section{Introduction.}\label{intro} %%1.
%\subsection{Duality and the Classical EOQ Problem.}\label{class-EOQ} %% 1.1.
%\subsection{Outline.}\label{outline1} %% 1.2.
%\subsubsection{Cyclic Schedules for the General Deterministic SMDP.}
%  \label{cyclic-schedules} %% 1.2.1
%\section{Problem Description.}\label{problemdescription} %% 2.

% Text of your paper here

\else

\documentclass[11pt]{article}
\usepackage[numbers]{natbib}
\usepackage{packages}
\usepackage{editing-macros}
\usepackage{formatting}
\usepackage{./statistics-macros}

% \onehalfspacing
% \renewcommand{\baselinestretch}{1.35}

\begin{document}

% Control whitespace around equations
\abovedisplayskip=8pt plus0pt minus3pt
\belowdisplayskip=8pt plus0pt minus3pt

% ------------------------------------------------------------------------
% Main Paper Body
% ------------------------------------------------------------------------

% ------------------------------------------------------------------------
% Default title and authorship
% ------------------------------------------------------------------------
\begin{center}
  {\LARGE A Sensitivity Approach to Causal Inference Under Limited Overlap} \\
  \vspace{.5cm} {\large Yuanzhe Ma \qquad Yian Huang \qquad  Hongseok Namkoong} \\
  \vspace{.2cm}
  {\normalsize   Columbia University} \\
  \vspace{.2cm} \texttt{\{ym2865, yh3209, hongseok.namkoong\}@columbia.edu}

  \end{center}

% ------------------------------------------------------------------------
% Abstract
% ------------------------------------------------------------------------

\begin{abstract}%
Limited overlap between treated and control groups is a key challenge in
observational analysis.  Standard approaches like trimming importance weights
can reduce variance but introduce a fundamental bias.  We propose a
sensitivity framework for contextualizing findings under limited overlap, where
we assess how irregular the outcome function has to be in order for the main
finding to be invalidated. Our approach is based on worst-case confidence
bounds on the bias introduced by standard trimming practices, under explicit
assumptions necessary to extrapolate counterfactual estimates from regions of
overlap to those without.  Empirically, we demonstrate how our sensitivity
framework protects against spurious findings by quantifying uncertainty in
regions with limited overlap.

%%% Local Variables:
%%% mode: latex
%%% TeX-master: "main"
%%% End:

\end{abstract}

\fi

%\tableofcontents % This generates the table of contents automatically

\section{Introduction}
\label{section:introduction}

Observational data is widely utilized when randomized experiments are
infeasible or fail to adequately represent target populations.  A key
challenge in observational analysis is the lack of overlap between treatment
and control groups.  Even when a nominally large dataset is collected, the
effective sample size may be prohibitively small when there is a region with
little overlap between treated and control populations.  
As an example, if the
treatment of interest is rarely observed among
older citizens, estimating their
counterfactual (treated) outcome becomes inherently unreliable.  This
challenge is further exacerbated in modern operational contexts, where
high-dimensional covariate representations~\citep{DAmourDiFeLeSe21} increase
data sparsity, making causal identification particularly difficult in regions
of the covariate space with small effective sample size. 
Standard inferential
methods relying on
asymptotic normal approximations
tend to fail silently when the
effective sample size is limited. 
Empirically, different observational
datasets with lack of overlap may lead to contradictory conclusions on the
treatment effects~\citep{NIH16, LiReShetal18, LandgrenSiAuetal18,
  HussainObShSo22, HussainShObDeSo23}.

Theoretically, several authors have quantified how the lack of overlap
deteriorates the convergence rates of typical confidence
intervals~\citep{KhanTa10, BussoDiMc14, Rothe17, YangDi18,HongLeLi19, MaWa20,
  HeilerKa21, KhanNe22, MaSaSaUr23}.  In particular,~\citet{MaWa20} show that
in general, the limiting distribution of the inverse probability weighting
(IPW) estimator may not be Gaussian, calling into question the validity of
standard central limit theorems-based confidence intervals.  Practically, limited overlap results
in large importance weights in typical estimators, leading to high
variances~\citep{KangSc07, KhanTa10}.  To stabilize these estimators,
researchers have developed various trimming and reweighting
methodologies~\citep{CrumpHoImMi06, CrumpRiHoImMi09, LiThLi18,
  JuScVa19,SasakiUr22}, which are commonly applied in fields including
economics and biomedical research~\citep{LaLonde86, HeckmanIcTo97,
  DehejiaWa99, Frolich04, SmithTo05, PetersenPoGrWaVa12, GlynnLuRoPoScSt19}.
Trimming, in particular, has been shown to improve the accuracy of estimators
in certain scenarios~\citep{ColeHe08, Sturmer10, LeeLeSt11}, and has become a
standard practice in cases of limited overlap.  While these approaches are
effective in reducing variance, they change the estimand and introduce
biases that are often difficult to quantify.

\begin{figure}[t] 
  \begin{minipage}[b]{0.49\textwidth}
    \centering \includegraphics[width=   \textwidth, height=6cm]{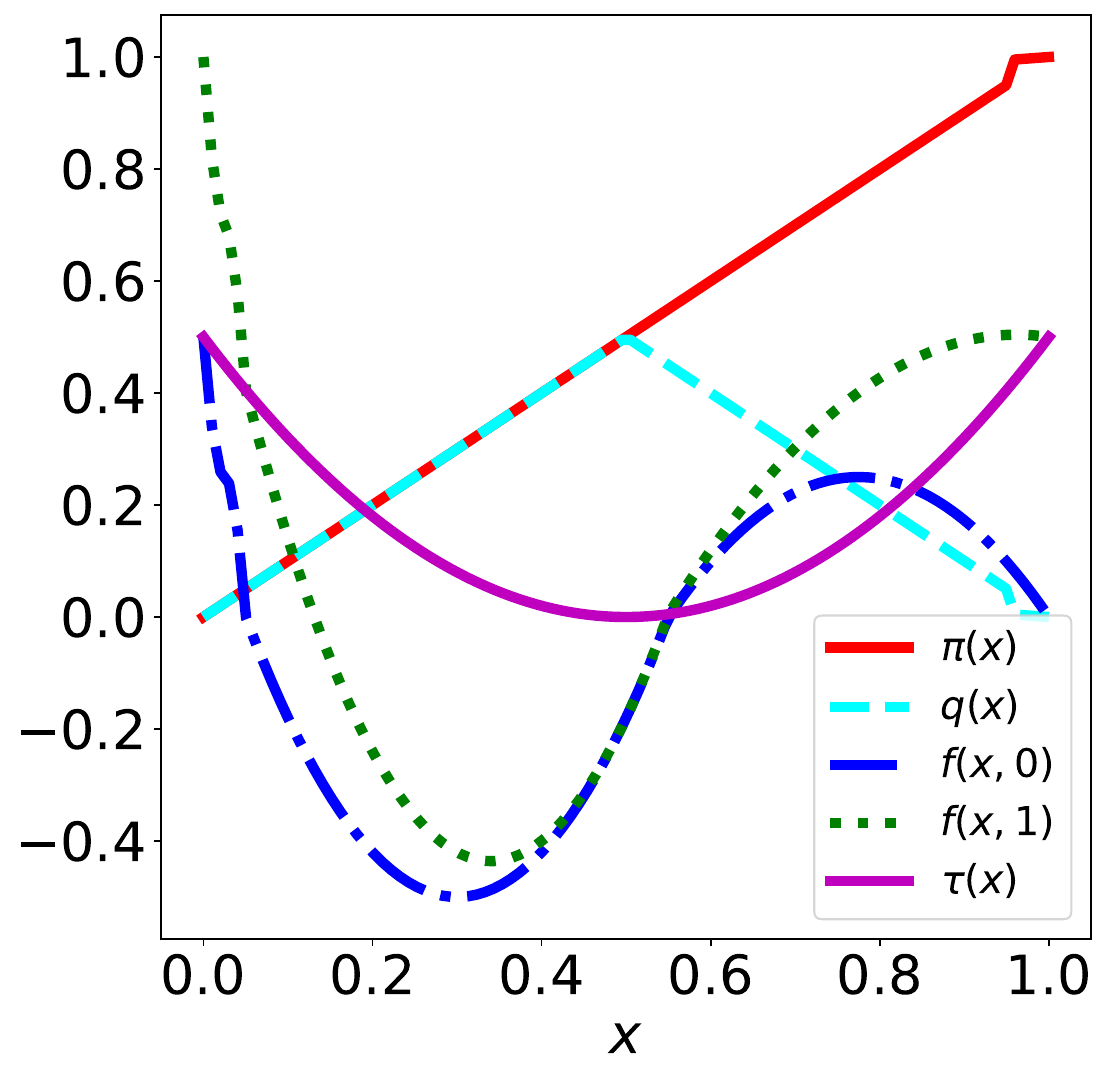}
  \end{minipage}
  \hfill
  \begin{minipage}[b]{0.49\textwidth}
    \centering \includegraphics[width=\textwidth, height=6cm]{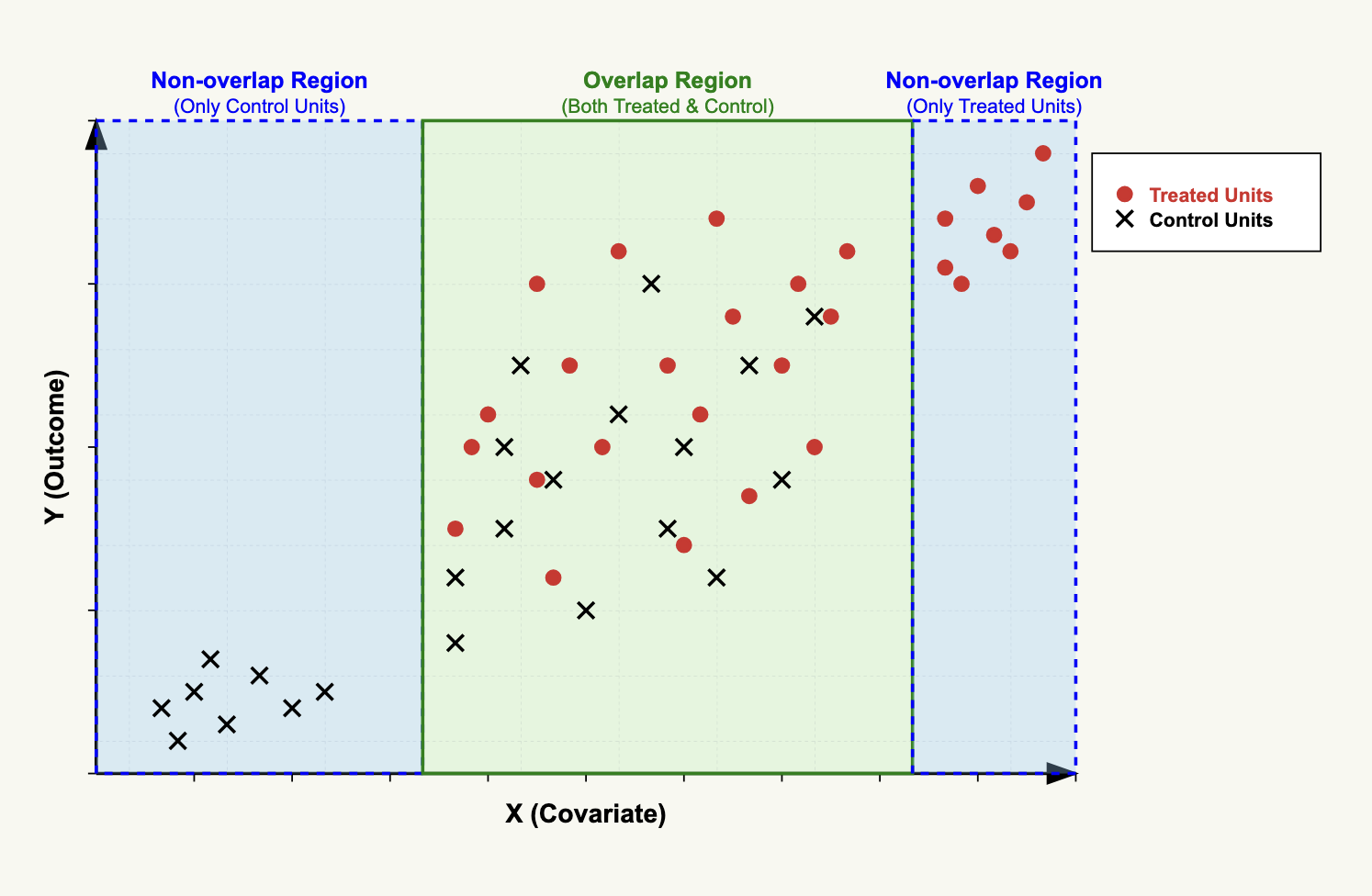}
  \end{minipage}
  \caption{\textbf{Left}: data-generation process used in the simulation setup where
    $\pi(x) = \mathbb{P}(Z = 1 \mid X = x)$ denotes the propensity score,
    $q(x) = \min \set{\pi(x),1-\pi(x)}$ measures whether
    a point has sufficient
    overlap, and $f(x, z)$ represents the potential outcome for a unit with
    covariates $x$ under treatment assignment $z \in \set{0, 1}$.  The individual
    treatment effect is defined as $\tau(x) = f(x, 1) - f(x, 0)$.
    \textbf{Right}: Visualization of one simulated observational dataset.
  }\label{fig:simulation}
\end{figure}

To illustrate, we consider the following simple example:

\begin{example}
  We illustrate the impact of limited overlap in a problem with
  one-dimensional covariate $X \sim \Uni[0,1]$.  We defer full details of the
  data-generating process to Section~\ref{sec:experiment-sim}, and instead
  provide a visual illustration in Figure~\ref{fig:simulation}.  Here, the
  propensity scores take extreme values when $x$ is close to 0 or 1.
  In Figure~\ref{fig:simulation-AIPW-viz}a, we observe how the augmented
  inverse propensity weighting ($\AIPW$) estimator exhibits wide confidence
  intervals that fluctuate around the true treatment effect, reflecting its
  instability under limited overlap.  To remedy this, we consider standard
  heuristics that truncate the importance weights based on 
  a threshold from the
  set
  \begin{align}
    \mc{E} = \set{0.01,0.02,0.03,0.04,0.05} \label{eqn:sim-eps-set}
  \end{align}
  by selecting $\epsilon$ that minimizes the length of the truncated
  $\mathsf{AIPW}$ confidence interval, denoted by
  $\mathsf{AIPW}_{\mathsf{partial}, \epsilon}$:
  \begin{align}
    \epsilon\opt \in \argmin_{\epsilon \in \mc{E}} 
    |\AIPW_{\mathsf{partial}, \epsilon}|. 
    \label{eqn:eps-opt-def}
  \end{align} 
  In Figure~\ref{fig:simulation-AIPW-viz}b, we observe $\AIPWP$ gradually
  deviates from the target estimand as the non-overlap region increases due to
  the bias introduced by trimming.
  \label{example:simulation}
\end{example}

\begin{figure}[t]
  \centering
    \begin{minipage}[b]{0.49\textwidth}
    \includegraphics[width=\textwidth, height=6.5cm]{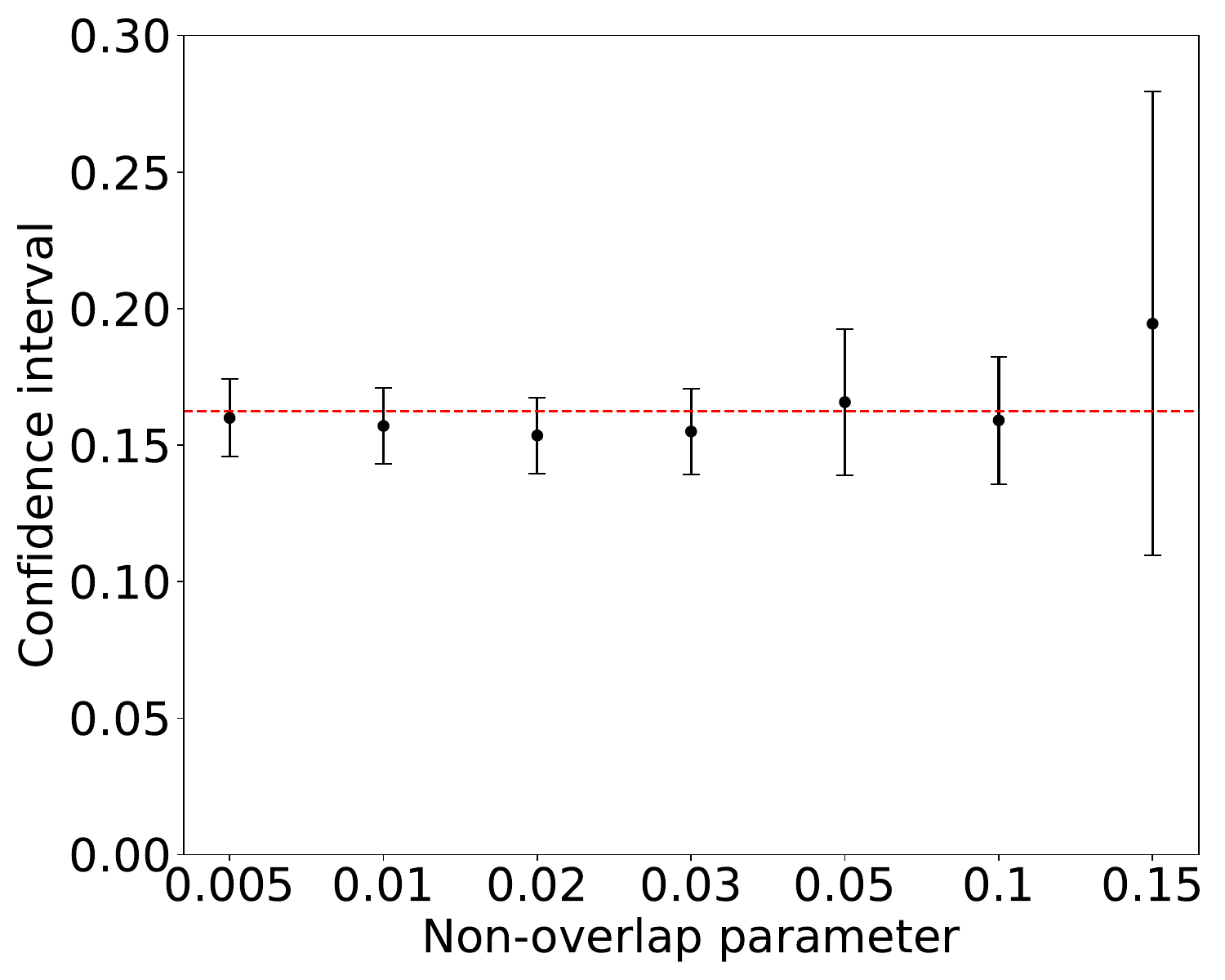}
  \end{minipage}
  \begin{minipage}[b]{0.49\textwidth}
    \includegraphics[width=\textwidth, height=6.5cm]{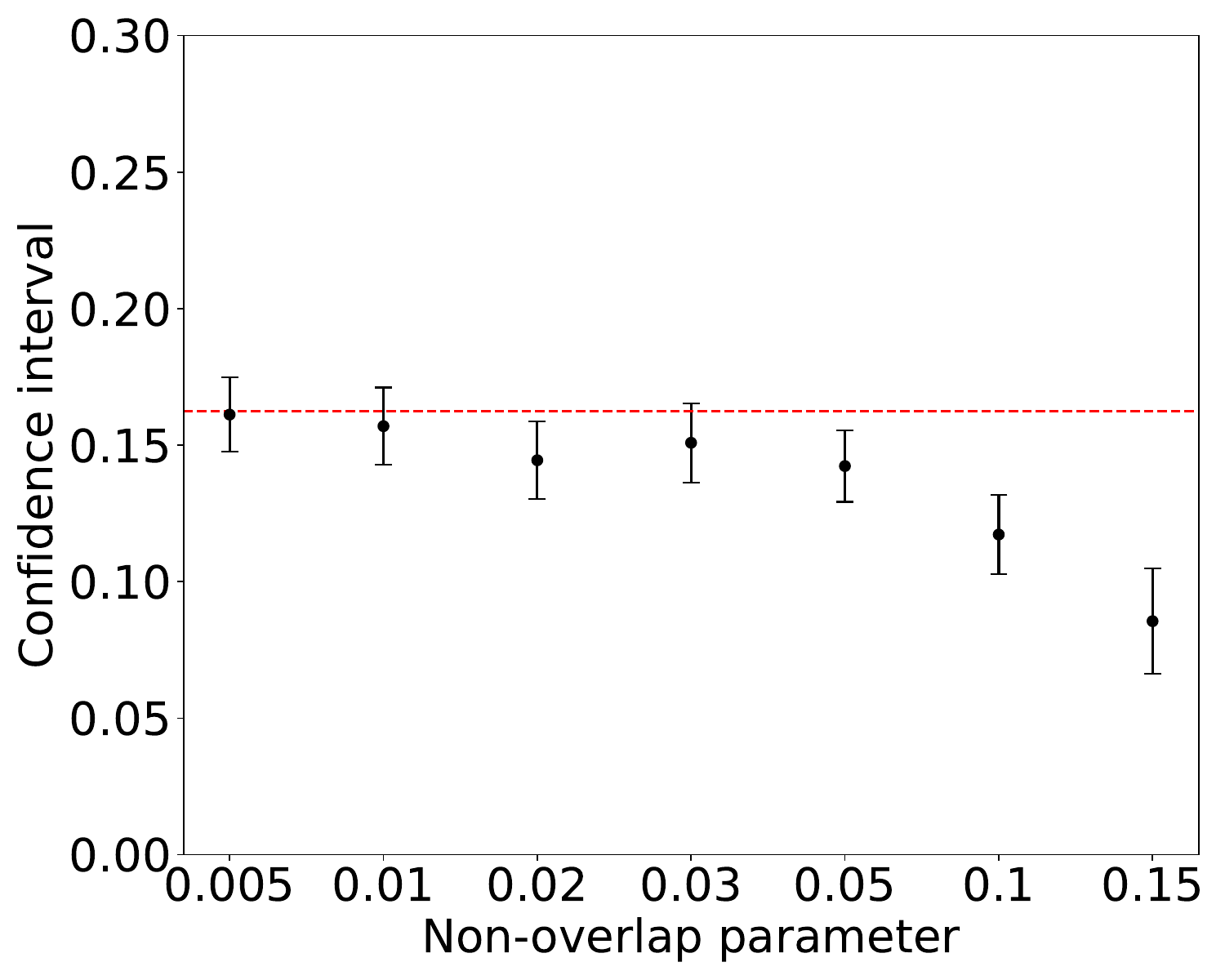}
  \end{minipage}
  % \hfill
  \caption{Confidence intervals from $\AIPW$ (left) and its trimmed variant
    $\AIPWP$ (right) across different overlap levels, with the dotted red line
    representing the true estimand value; higher values on the $x$-axis mean
    more limited overlap.  \textbf{Left}: $\AIPW$ yields very wide confidence
    intervals.
    \textbf{Right}: We follow standard heuristics to truncate data
    in a way such that $\AIPWP$'s confidence interval has the smallest
    length. }
  \label{fig:simulation-AIPW-viz}
\end{figure}

To assess the sensitivity of findings under limited overlap, we propose an
inferential method that ensures always-valid uncertainty quantification under
explicit assumptions on the smoothness of the outcome function.  Unlike
asymptotic methods, our approach allows us to quantify instance-specific
uncertainty that accurately scales with the level of overlap between the
treated and control groups.  Consider binary treatments $z_i \in \set{0,1}$,
covariates $x_i \in \mc{X}$ and potential outcomes $f(x_i,z_i)$ that generate
realized outcome data $y_i = f(x_i,z_i) + \epsilon_i$ for $i = 1,\ldots,n$,
where $\epsilon_i$ is the noise.

To analyze the bias introduced by truncating based on propensity scores, we
allow the modeler to specify a threshold for the \emph{overlap} region (e.g.,
based on a fitted propensity score), and use $S_i$ to denote the indicator for
whether unit $i$ belongs in this region. We separate the average treatment
effect into two estimands
\begin{subequations}
\begin{align}
\tau_+ & \defeq \sum_{i=1}^n \I{S_i = 1} \frac{1}{n} (f(x_i,1) - f(x_i,0)),
\label{eqn:tau-overlap} \\
\tau_- & \defeq \sum_{i=1}^n \I{S_i = 0} \frac{1}{n} (f(x_i,1) - f(x_i,0)),
\label{eqn:tau-non-overlap}
\end{align}
\end{subequations}
and assume that standard inferential tools are appropriate for estimating
$\tau_+$.

Since asymptotic normal approximations
are invalid in regions with limited overlap, we
use a worst-case approach to bound $\tau_-$, where we consider all outcome
functions that satisfy a smoothness condition.  We leverage~\citet{Donoho94}'s
minimax framework to find the tightest possible confidence interval accounting
for the worst-case bias across all admissible outcome functions. Intuitively,
this minimax method finds an optimal affine estimator while adding a bias
correction term to ensure coverage of the resulting confidence interval.

\begin{figure}[h] 
  \centering \includegraphics[width=0.7\textwidth, height=8cm]{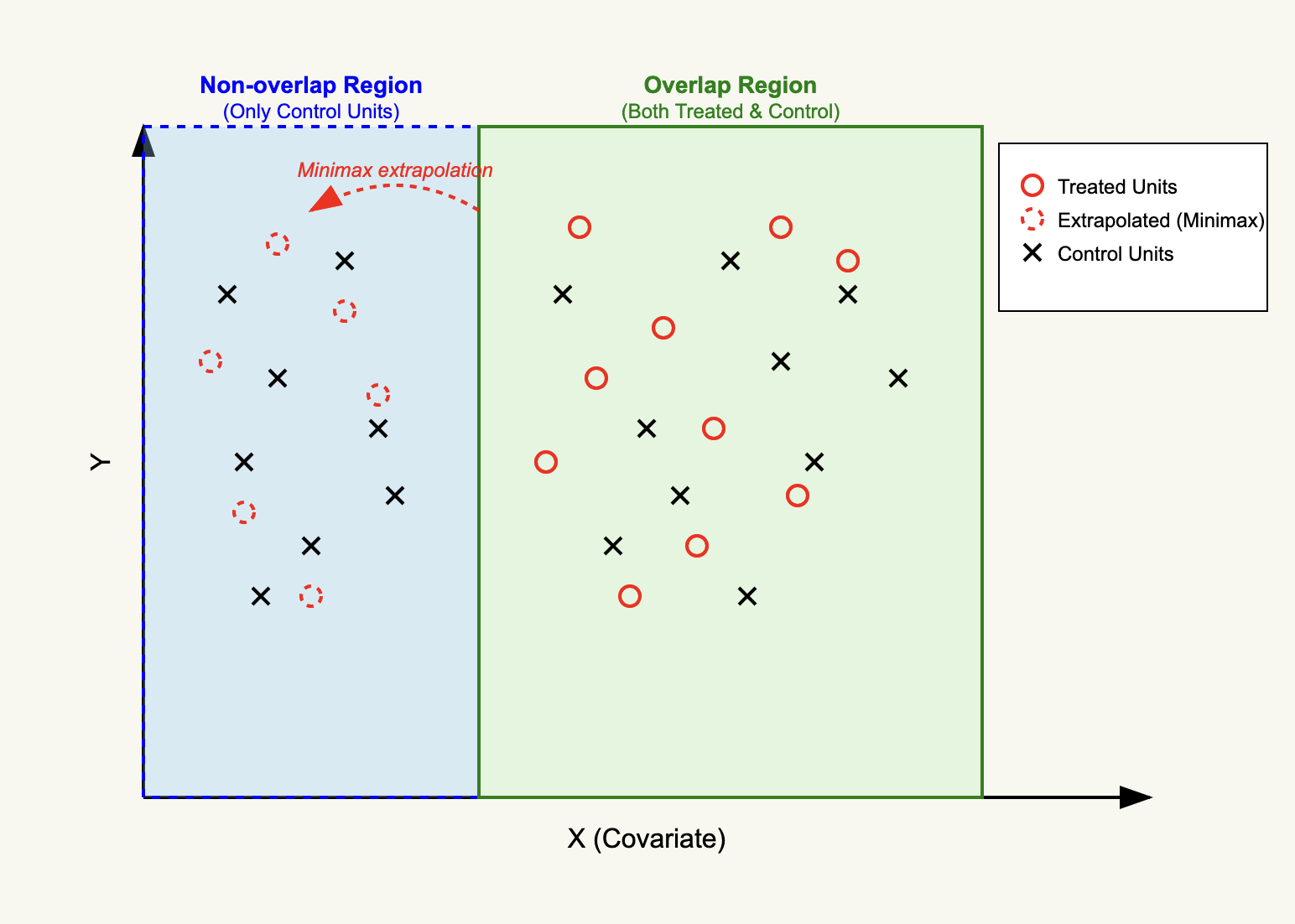}
  \caption{
    Visualization of our method.
    In the overlap region, we use 
    typical asymptotic confidence intervals.
    In the non-overlap region, we use the minimax approach to extrapolate from the overlap region.
    Our method allows the analyst to analyze the potential bias caused by ignoring samples with extreme propensity scores
    and see how 
    this depends on the
    extrapolability
    of data from the non-overlap region to the overlap region.
  }
  \label{fig:viz-our-method}
\end{figure}

~\citet{ArmstrongKo21} consider the set of Lipschitz outcome functions with
smoothness constant $L$ and apply~\citet{Donoho94}'s
minimax inferential
framework to \emph{both} $\tau_+$ and $\tau_-$, leading to overly conservative
confidence intervals.  Instead, we only apply the worst-case confidence region
over $\tau_-$.
Our approach (depicted in Figure~\ref{fig:viz-our-method}) separates the
estimand to make full use of the standard normal approximation, and only
resorts to conservative worst-case bounds when   traditional asymptotic
assumptions   break down.
This decomposition allows practitioners to explicitly quantify the potential
bias introduced by standard trimming or reweighting practices and understand
when their conclusions may be driven by extrapolation beyond well-supported
regions.

Instead of requiring the modeler to commit to a particular level
of smoothness that is fundamentally untenable,
we take a sensitivity
analysis perspective where we consider increasing values of $L$ until the
confidence bounds cross a certain threshold, % (e.g., 0), 
and then contextualize
this level of smoothness by using data from regions of overlap. 
This allows us
to translate abstract smoothness assumptions into interpretable diagnostics
and help analysts validate assumptions necessary for reliable inference when
overlap is limited.  This form of sensitivity analysis provides a continuous
view on model sensitivity, rather than a binary accept/reject judgment.
Ultimately, this enables more informed decisions about whether to trust,
adjust, or reinterpret conclusions drawn from trimmed versions of asymptotic
methods.

We begin by reviewing~\citet{Donoho94}'s minimax inferential framework in
Section~\ref{section:approach}, then introduce our sensitivity analysis
approach in Section~\ref{sec:proposed-approach}. We provide some analytical
insights in Section~\ref{section:analytic} and demonstrate our framework using
both simulated and real datasets in Section~\ref{section:experiment}. Our
results illustrate how our sensitivity framework provides a simple diagnostic
against asymptotic methods that silently fail. When compared to the naive
minimax method~\cite{ArmstrongKo21}, our approach yields significantly shorter
intervals with comparable coverage properties. 
Finally, to inspire future work, 
we discuss potential extensions of our framework to adaptive data collection
settings in Section~\ref{sec:continual-sampling-framework}.

%%% Local Variables:
%%% mode: latex
%%% TeX-master: "main"
%%% End:

\paragraph{Related work}

There is a substantial body of work studying observational analysis under
limited overlap. These methods aim to rigorously characterize the statistical
uncertainty associated with estimators such as IPW, AIPW, and matching-based
methods, often under relaxed or irregular identification conditions.
\citet{KhanTa10} and~\citet{Rothe17} show that in the presence of limited or
vanishing overlap, the rate at which estimators converge to their true values
can be significantly slower than the parametric $\sqrt{n}$ rate, complicating
standard inference. \citet{YangDi18} and~\citet{HeilerKa21} extend this line
of work by proposing robust inferential procedures that adjust for the heavy
tails and large variances caused by extreme weights, which are especially
problematic in regions of limited overlap.

Similarly,~\citet{HongLeLi19} and~\citet{MaWa20} highlight how lack of overlap
can lead to non-Gaussian asymptotic distributions, and they develop corrected
inference procedures that better account for these irregularities.  For
instance,~\citet{MaWa20} explicitly demonstrate the failure of conventional
confidence intervals when the IPW estimator’s asymptotic distribution departs
from normality due to limited overlap.  More recently, using a bias correction
approach,~\citet{MaSaSaUr23} propose doubly robust estimators that achieve
valid asymptotic inference under weaker conditions, including near-zero
overlap and high-dimensional covariates.  On the other hand, our framework
requires domain knowledge on the data-generating function (i.e., specifying
the function class $\mc{F}$ in a parametric form), but it provides
non-asymptotic guarantees grounded in the minimax estimation literature,
allowing for near-optimal finite-sample inference with explicit control over
bias and variance trade-offs within a structured function class.  This
approach enables reliable inference even in challenging regions of the
covariate space, without resorting to overly conservative bounds.

Another related line of work by~\citet{KhanSaUg23} addresses the challenge of
extrapolating treatment effects from the region of covariate overlap to the
non-overlap region.  Their method, which provides a point estimate of the
treatment effect, builds on the partial identification
literature~\citep{Manski90, ImbensMa04, Stoye09, Cui21, KallusZh21}, aiming to
conservatively bound treatment effects in the non-overlap region.  However,
their approach depends critically on the accuracy of the fitted outcome model
$\hat{\mu}$ in the overlap region, which is then extrapolated to the
non-overlap region.  This reliance means their guarantees can degrade when
model fit is imperfect, especially in finite samples.

Our framework is closely related to the literature on minimax estimators,
particularly the foundational work by~\citet{Donoho94} that introduces affine
estimators that achieve near-optimal worst-case performance over specified
function classes.  Building on this,~\citet{ArmstrongKo18, ArmstrongKo21}
apply the minimax approach to causal estimands, particularly without
assuming strong overlap.  Their approach yields estimators and confidence
intervals with optimal worst-case performance, even when standard overlap
assumptions fail.   
However, a key limitation of these methods is that the corresponding
confidence intervals tend to be overly conservative, particularly in finite
samples, due to their focus on worst-case performance across broad function
classes.

In contrast, our method maintains the spirit of minimax optimality, but is
more directly aimed at practical inference in regions of non-overlap, where
conventional estimators perform poorly.  As we detail in
Section~\ref{section:approach}, our framework avoids the excess
conservativeness of prior minimax-based intervals by applying the minimax
approach only to analyze the non-overlap region.

% %%% Local Variables:
%%% mode: latex
%%% TeX-master: "main"
%%% End:

\section{Finite-sample minimax inference on regions of non-overlap}
\label{section:approach}

In this work, we assume that the observed units comprise the entire population
of interest.  This perspective is well-established in the causal inference
literature~\citep{Imbens04, AbadieIm11}, and implies that the sample size $n$
is fixed rather than drawn from a superpopulation.  This modeling choice is
particularly reasonable in many applied policy and healthcare settings, where
the goal is to estimate causal effects for a specific, finite group---such as
all patients in a registry, participants in a program, or residents of a
region.  In these scenarios, inference does not concern some hypothetical
infinite population, but rather the actual individuals in the dataset. By
treating the population as fixed, we avoid assumptions about the
data-generating process beyond what is necessary for causal identification,
and we can directly quantify uncertainty due to lack of overlap or sparsity in
observed units.  This approach also aligns well with our focus on evaluating
uncertainty in finite samples and identifying regions with poor support, where
asymptotic approximations may fail.  We would also like to point out that
though many estimators in the literature focus on the population-level
treatment effects, they can be used to estimate the sample average treatment
effect~\citep{Imbens04}.

Formally, we consider the standard potential outcome framework where the observed tuples 
$\set{x_i,y_i,z_i}_{i=1}^n$ consist of 
treatment $z_i \in \set{0,1}$ alongside covariates $x_i$.
Our goal is to estimate the sample average treatment effect  (ATE)
\begin{align}
\tau \defeq \frac{1}{n} \sum_{i=1}^n \E[Y_i(1) - Y_i(0)|X_i = x_i].
\label{eqn:ATE}
\end{align} 
under the assumption of unconfoundedness
$(Y_i(1), Y_i(0)) \perp Z_i \mid X_i$.  This is also sometimes called
conditional average treatment effect~\citep{Imbens04, AbadieIm11}.  More
generally, for known weights $w_i \ge 0$, we are interested in the weighted
average treatment effect
\begin{align}
  \tau_{\bm{w}}(f) \defeq  \sum_{i=1}^n w_i \tau(f, x_i)
  ~~\mbox{where}~~\tau(f, x) \defeq f(x,1) - f(x,0).
\label{eqn:tau-weighted-ate-def}
\end{align}
Since $f$ is unknown, we omit its explicit dependence in $\tau_{\bm{w}}(f)$
and $\tau(f, x_i)$.

To obtain reliable estimates of the ATE, the strong overlap
assumption~\citep{RosenbaumRu83b,Imbens04} states that the propensity score
$\pi(x) \defeq \P(z = 1 | x = x)$ is bounded away from 0 and 1:
\begin{align*}
\varepsilon < \pi(x_i) < 1 - \varepsilon \text{  for some } \varepsilon > 0 \text{ for all } i.
\end{align*}  
 When facing limited overlap,
researchers truncate/winsorize the fitted propensity score at some level that
makes the resulting estimator stable. This implicitly changes the estimand to
units with sufficient overlap.  As we observed in
Example~\ref{example:simulation}, without truncation, the confidence intervals
typically blow up.  On the other hand, as we show in
Section~\ref{sec:experiment-sim}, truncation could lead to substantial
undercoverage under limited overlap.

Throughout, we assume the propensity score $\pi(x)$ is known and show that
standard asymptotic estimators can still fail under this idealized setting,
highlighting the inherent challenges posed by limited overlap.  As we will
see, our framework requires minimal knowledge of the propensity score and thus
can be easily generalized to the setting with unknown propensity scores.  In
addition, for most parts of the paper, we assume the analyst follows a
particular trimming procedure that divides the data into non-overlap and
overlap regions.

\subsection{Background}

Before we introduce our approach, we first provide background on an
inferential framework built on the theory of minimax estimation for linear
functions~\citep{Donoho94, CaiLo04, ArmstrongKo21}.  Instead of relying on
central limit theorems that assume an infinite stream of data generated by a
fixed data-generating distribution, this framework generates a worst-case
confidence interval over all data-generating distributions that could have
generated the sample the researcher has access to.

We assume there is an unknown data-generating function $f$ with
\begin{align*}
    y_i = f(x_i, z_i) + \epsilon_i, 
\end{align*} 
where $\epsilon_i \sim N(0, \sigma^{2}(x_{i}, z_{i}))$ is the  independent
Gaussian noise  with known variance.  Since we are not in the asymptotic
regime, we impose an explicit distributional assumption on the noise.  This
framework can be generalized to other distributional assumptions as discussed
in~\citet{JuditskyNe09}.  As discussed in Section~\ref{section:experiment}, we
can estimate $\sigma$ from data and the framework remains applicable.

The unknown function $f$ is assumed to be within a known function class
$\mc{F}$, which is prespecified and reflects the researcher's prior knowledge
on all possible functions that could have generated the observed data.  We
focus on generating confidence intervals $\mc{C}$ with guaranteed coverage of
the desired estimand $\tau_{\bm{w}}(f)$ when the true generating function
satisfies $f \in \mc{F}$:
\begin{align}
  \inf_{f \in \mc{F}}\P_f (\tau_{\bm{w}}(f) \in \mc{C}) \ge 1-\alpha.\label{eqn:coverage-def}
\end{align} 
We usually assume $\alpha$ is a fixed quantity, e.g., $\alpha = 0.05$.  
Recalling
that we assume the population and sample size $n$ are fixed, we omit the
dependence of the estimand on $n$ in the paper.  We focus on centrosymmetric
sets ($f \in \mc{F}$ implies $-f \in \mc{F}$) as it simplifies the expression
of the minimax estimator to be introduced (see, e.g.~\cite{ArmstrongKo21}).

After fixing the function class $\mc{F}$ and data $\mc{D}$, the minimax
confidence interval of $\tau_{\bm{w}}$~\eqref{eqn:tau-weighted-ate-def} is
governed by the following quantity that is known as the \textit{modulus of
  continuity}~\citep{Donoho94}:
\begin{align}
\omega_{\mc{F};\mc{D}}(\delta; \bm{w}) \defeq \sup_{f \in\mathcal{F}}\left\{
\sum_{i=1}^n 2 w_i \paran{f(x_i,1)-f(x_i,0)}
:
\sum_{i=1}^{n}\frac{f^2(x_i, z_i)}{\sigma^{2}(x_i, z_i)}\leq
\frac{\delta^{2}}{4}\right\}.  \label{eqn:def-omega}
\end{align}
The parameter $\delta$ controls the weights as we visualize through an example
in the next subsection. We discuss how to compute $\omega_{\mc{F};\mc{D}}$ in
Appendix~\ref{sec:algo}.

Letting the optimal solution of~\eqref{eqn:def-omega} be
$f\opt_\delta(x_i,z_i)$, one then computes $\omega(\delta; \bm{w})$ for every
$\delta > 0$, and each $\delta$ yields a corresponding minimax
estimator~\citep{ArmstrongKo18,ArmstrongKo21} $\hat{\tau}_{\delta}(\bm{w})$ given by
\begin{align}
  \hat{\tau}_{\delta}(\bm{w})  = \sum_{i=1}^n     
  \frac{2\omega'(\delta)}{\delta} \frac{f\opt_\delta(x_i,z_i)}{\sigma^2(x_i,z_i)}
  y_i
  = \paran{\sum_{i=1}^n w_i}
  \sum_{i=1}^n \frac{f\opt_\delta(x_i,z_i)/{\sigma^2(x_i,z_i)}}
  {\sum_{j=1}^n z_j f\opt_\delta(x_j,z_j)/{\sigma^2(x_j,z_j)}} y_i,
\label{eqn:minimax-w-interlval-expression}
\end{align} 
where $f\opt$ is  an 
 optimal solution of~\eqref{eqn:def-omega}.
 In addition,
 \begin{align}
 \mc{C}_\delta(\bm{w}) = 
    \hat{\tau}_{\delta}(\bm{w})  \pm  \cv_\alpha\paran{\frac{
\maxbias(\hat{\tau}_{\delta}(\bm{w}))
}{
\sd(\hat{\tau}_\delta(\bm{w})) 
}} 
\sd(\hat{\tau}_\delta(\bm{w})) 
\label{eqn:CI-expression}
 \end{align}
 is a valid confidence interval that applies a bias correction procedure 
 on top of $\hat{\tau}_{\delta}(\bm{w})$,
where  $\cv_\alpha(b)$ is the $1-\alpha$ 
quantile of a $|N(b,1)|$ 
random variable.
 In~\eqref{eqn:CI-expression},  
\begin{align}
\maxbias(\hat{\tau}_{\delta}(\bm{w}))
\defeq \sup_{f \in \mc{F}} \E_f \Big[\hat{\tau}_{\delta}(\bm{w})  - \tau_{\bm{w}}(f)
\Big]
\label{eqn:max-bias}
\end{align}
is the maximum possible 
bias of $\hl_{\delta}(\bm{w})$, 
which can be shown~\citep{ArmstrongKo21} 
to be equal to
\begin{align}
\maxbias(\hat{\tau}_{\delta}(\bm{w}))
=  \half (\omega(\delta; \bm{w})  - \delta \omega'(\delta; \bm{w})).
\label{eqn:max-bias-equiv-formula}
\end{align}
In addition, the standard error of the estimator can be computed by
\begin{align}
\sd(\hat{\tau}_\delta(\bm{w})) = \omega'(\delta;\bm{w}).
\label{eqn:std-formula}
\end{align}

After obtaining~\eqref{eqn:CI-expression} for every $\delta$,
we construct
the minimax
CI   using  $\DFLCI$,  where $\DFLCI$
is chosen by solving the following  
problem to 
minimize the CI length balancing bias and variance  
of $\hat{\tau}_{\delta}(\bm{w})$:
\begin{align}
\DFLCI(\bm{w}) \in \argmin_{\delta > 0} \set{\cv_\alpha\paran{\frac{
\maxbias(\hat{\tau}_{\delta}(\bm{w}))
}{\sd(\hat{\tau}_\delta(\bm{w}))}} 
\sd(\hat{\tau}_\delta(\bm{w})) 
},
\label{eqn:opt-prob-delta}
\end{align}
which yields the minimax CI, denoted as $\CFLCI(\bm{w})$.
This construction determines the shortest possible length for fixed-length
confidence procedures under a given function class, and the associated
modulus-based confidence interval attains this minimax criterion. Importantly,
the optimal modulus procedure admits an implementation that is affine in the
outcomes (an affine estimator together with a calibrated bias
correction).
As~\citet[Theorem A.3]{ArmstrongKo21} show, this confidence
interval is minimax-optimal among all fixed-length procedures.

In this work, we follow~\citet{ArmstrongKo21} and focus on Lipschitz outcome
functions
\begin{align}
  \mc{F}_{L, \norm{\cdot}} \defeq \set{f: |f(x,d) - f(\Tilde{x},d)| \le L \norm{x-\Tilde{x}}, \forall x,\Tilde{x} \in \mc{X}, d \in \set{0,1}}.  \label{eqn:def-F-L-Lip} 
\end{align}
It is well known that we can replace the constraint $\mc{F}$ with its
finite-sample counterpart and only impose Lipschitz constraints on the
observed $n$ data points~\cite{Beliakov06}, which we denote as
\begin{align}
  \mc{F}_{L, \norm{\cdot},n} \defeq \set{f: |f(x_i,d) - f(x_j,d)| \le L \norm{x_i - x_j}, \forall i,j \in [n], d \in \set{0,1}}.  \label{eqn:def-F-L-Lip-finite-sample} 
\end{align}
Henceforth, we focus on the $\ell_2$ norm
and
use $\mc{F}_{L}$ to represent $\mc{F}_{L, \ell_2,n}$.
 
\subsection{Proposed approach}

\newcommand{\wip}{\red{work in progress}}

Applying the finite-sample minimax approach over the entire
dataset~\cite{ArmstrongKo21} can be pessimistic, as it focuses on the
worst-case function within a given function class.  For the Lipschitz function
class, Theorem 2.3 of~\citet{ArmstrongKo21} demonstrates that the minimax
estimator relies only on information from nearby points---especially when the
Lipschitz constant $L$ is large---for counterfactual imputation.  This occurs
even when distant points, whose outcomes are observed and potentially
informative, are available.  As a result, the minimax method may fail to fully
utilize all available information, leading to practically inefficient
inference, as we demonstrate in the next section. Another notable feature of
the minimax approach is that the resulting confidence interval length does not
depend on the observed outcomes $y_1,\ldots,y_n$.  This further highlights its
conservative nature, particularly in regions where outcome data alone are
insufficient to justify narrower intervals.

This conservatism is well-suited for analyzing regions of
\textit{non-overlap}, where no comparable treated or control units are
observed and traditional estimation methods become unreliable. In such
settings, the inherently cautious nature of minimax inference provides an
honest reflection of the uncertainty stemming from a lack of overlap.

We thus study minimax confidence intervals $\CFLCI(\bm{w})$ on regions of
non-overlap (abbreviated as $\MW$).
Intuitively, $\MW$ can be used to bound the bias of any estimator targeting
\begin{align*}
    \tau_{\bar{\bm{w}}} \defeq \sum_{i=1}^n \left( \frac{1}{n} - w_i \right) \tau(x_i),
\end{align*}
i.e., the complement of $\tau_{\bm{w}}$ relative to the full-sample ATE.  In
addition to reweighted estimands proposed by~\citet{CrumpHoImMi06}, the family
of estimators $\hat{\tau}_{\bm{w}}$ includes standard truncated estimators
such as $\AIPWP$ as discussed before.  By ignoring regions with poor overlap,
these estimators aim to estimate $\tau_{\bar{\bm{w}}}$---the ATE over the
region with overlap--with greater accuracy, at the cost of potentially larger
bias $\tau_{\bar{\bm{w}}} - \tau$ even when these estimators are unbiased.  Our
framework, $\MW$, provides a principled framework to analyze this bias.
 
Though our framework is general, we focus on using it as an auxiliary tool for
typical asymptotic estimators such as $\AIPWP$.  Our proposed method, minimax
partial ($\MP_\epsilon$), is parameterized by the truncation threshold
$\epsilon \in (0,\half)$ on the propensity score chosen by the practitioner
when they truncate the data in $\AIPWP$, and only applies the minimax approach
on the non-overlap region.  Throughout, we define
\begin{align}
    q(x) = \min\set{\pi(x),1-\pi(x)} \label{eqn:q-def}
\end{align}
as a measure of overlap for each point $x$.  Given a fixed $\epsilon$, we
decompose the ATE $\tau$ as follows:
\begin{align}
\tau = \P(q(X) \ge \epsilon) \E[ \tau(X) | q(X) \ge \epsilon] + \P(q(X) < \epsilon) \E[ \tau(X) | q(X) < \epsilon] \defeq  \tau_{+}  + \tau_{-},
    \label{eqn:ATE-decomposition}
\end{align} 
where we assume $X \sim \emp$, the empirical distribution formed by data
$\set{x_i}_{i \in [n]}$.  Based on the above expression, we define
$\MP_\epsilon$ as
\begin{align}
  \MP_\epsilon \defeq  \M_{\bm{w}_\epsilon},
\label{eqn:m-partial-def}
\end{align}   
where $(\bm{w}_\epsilon)_i = \frac{1}{n} \I{q(x_i) < \epsilon}.$

Sometimes we also write $\MP_{\epsilon,L}$ to emphasize that we assume the
function class is $\mc{F}_L$~\eqref{eqn:def-F-L-Lip}.  We omit the dependence
on $\epsilon$ and simply write $\MP$ when the dataset is fixed, assuming that
the truncation threshold $\epsilon$ is selected using a standard procedure.
Using the interval $\MP_\epsilon$, we can construct a metric to evaluate the
underlying uncertainty in the non-overlap region.  For example, letting
$[\MP_{\epsilon}^-,\MP_{\epsilon}^+]$ be the interval $\MP_\epsilon$, we
consider the following metric as a way to estimate $\tau_-$
\begin{align}
  T(\MP_\epsilon)
  \defeq
  \max \set {|\MP_{\epsilon}^-| , |\MP_{\epsilon}^+| },
  \label{eqn:our-metric}
\end{align}
which is the largest absolute value of the interval end points generated by
$\MP_\epsilon$.  The larger this value is, the more bias $\AIPWP$ can lead to.
Similarly, we can use $T(\MP_\epsilon) = |\MP_\epsilon|$ to estimate the 
uncertainty of the bias estimate.

%%% Local Variables:
%%% mode: latex
%%% Tex-master: "main"
%%% End:

\begin{figure}[t] 
\centering \includegraphics[width=0.7\textwidth, height=8cm]{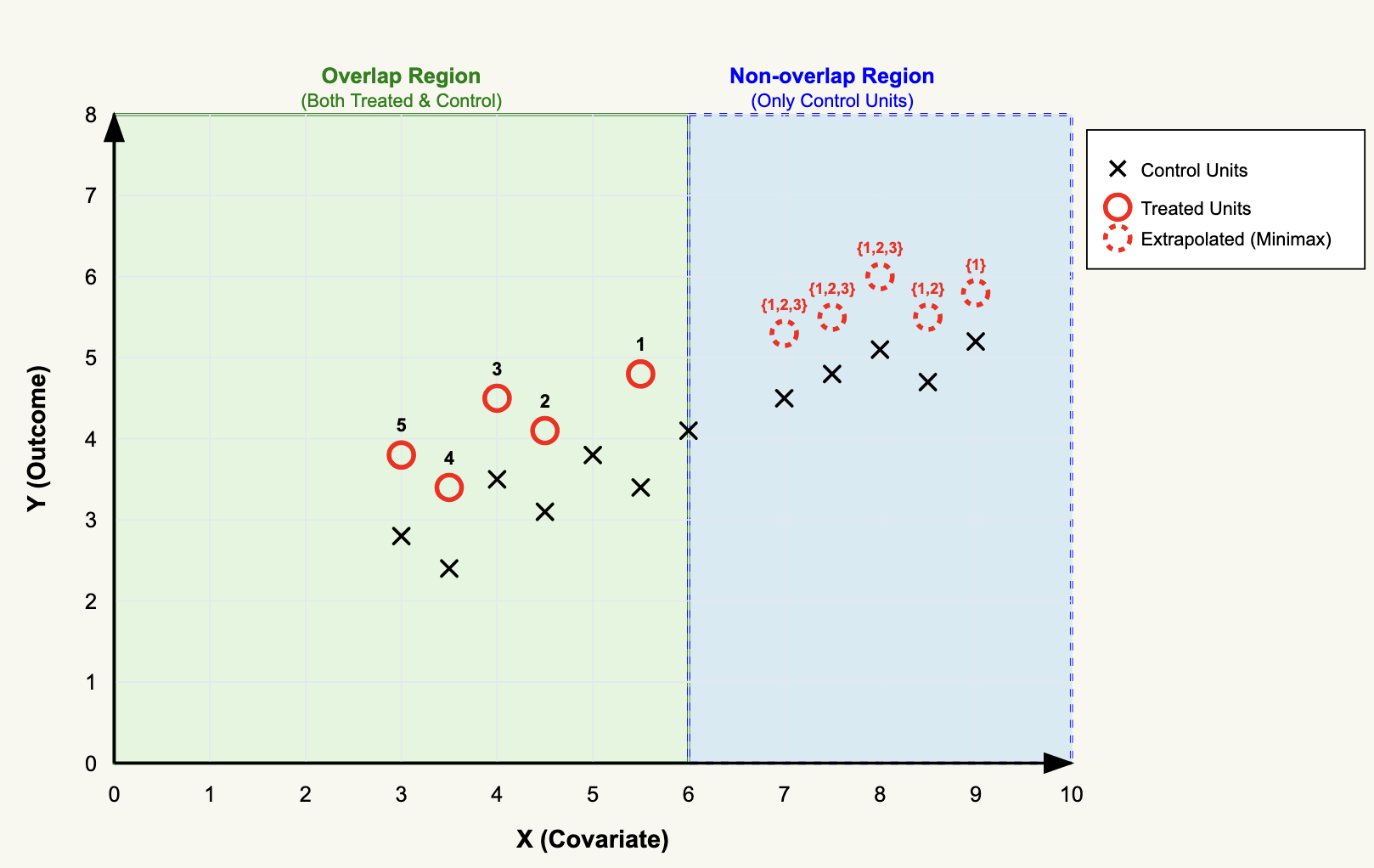}
\caption{  For each point in
  the non-overlap region, we list the set of treated points from the overlap
  region used in its extrapolation.  For example, the leftmost point $i$ in
  the non-overlap region uses points 1, 2, and 3 for extrapolation.  This
  means that point pairs $(i,1), (i,2), (i,3)$ have binding Lipschitz
  constraints for the program that defines the minimax estimator
  $\LFLCI(\bm{w})$~\eqref{eqn:minimax-w-interlval-expression}.  }
\label{fig:viz-counterfactual}
\end{figure}

\subsection{Analytic insights}
\label{section:analytic}

Since the derivation of the minimax estimator is involved, we provide some
analytical intuition.

\paragraph{Matching interpretation} We start by interpreting the minimax
estimator as a nearest-neighbor estimator.
\begin{lemma}
  Let $\mu \ge 0$ and $\Lambda \ge 0$ be the optimal dual variables corresponding
  to the Lipschitz constraints~\eqref{eqn:def-F-L-Lip-finite-sample}.  The
  minimax estimator imputes counterfactual values as
  \begin{align}
    \label{eqn:impute-conterfactual-more-specific}
    w_k \hat{f}(x_k,1 -z_k) = \sum_{j: z_j = 1- z_k} W_{jk} y_j
    ~~~~\mbox{where}~~~~W_{jk} = 
    \begin{cases}
      \frac{\Lambda_{jk}^1}{\mu}, & \text{if } z_k = 0 \\ 
      \frac{\Lambda_{kj}^0}{\mu}, & \text{if } z_k = 1.
    \end{cases}
  \end{align} 
\label{lemma:minimax-as-matching}
\end{lemma}
\noindent See Appendix~\ref{sec:additional-details-minimax} for the proof.

\begin{figure}[t]
 \begin{minipage}[b]{0.49\textwidth}
\centering \includegraphics[width=\textwidth, height=6cm]{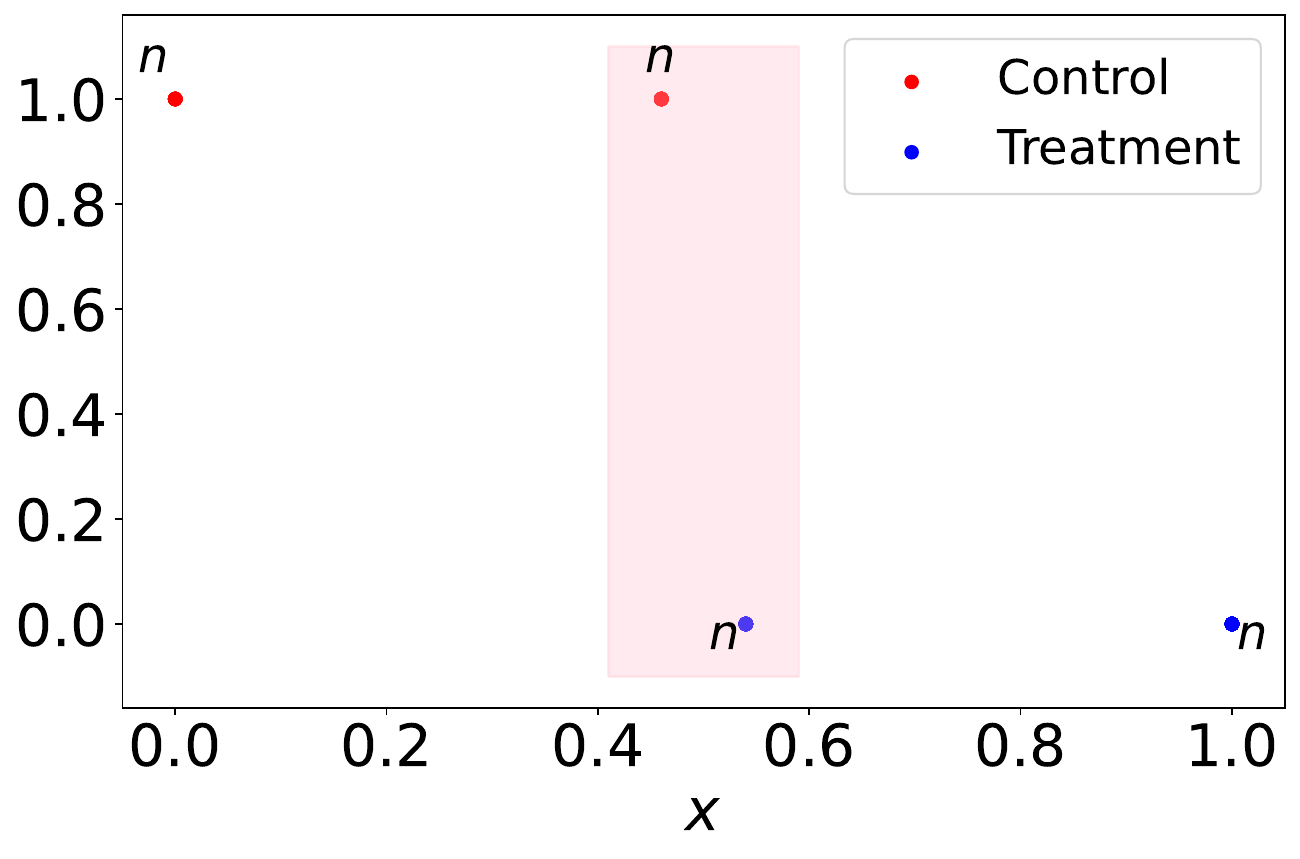}
 \end{minipage}
 \hfill
  \begin{minipage}[b]{0.49\textwidth}
 \centering \includegraphics[width=\textwidth, height=6cm]{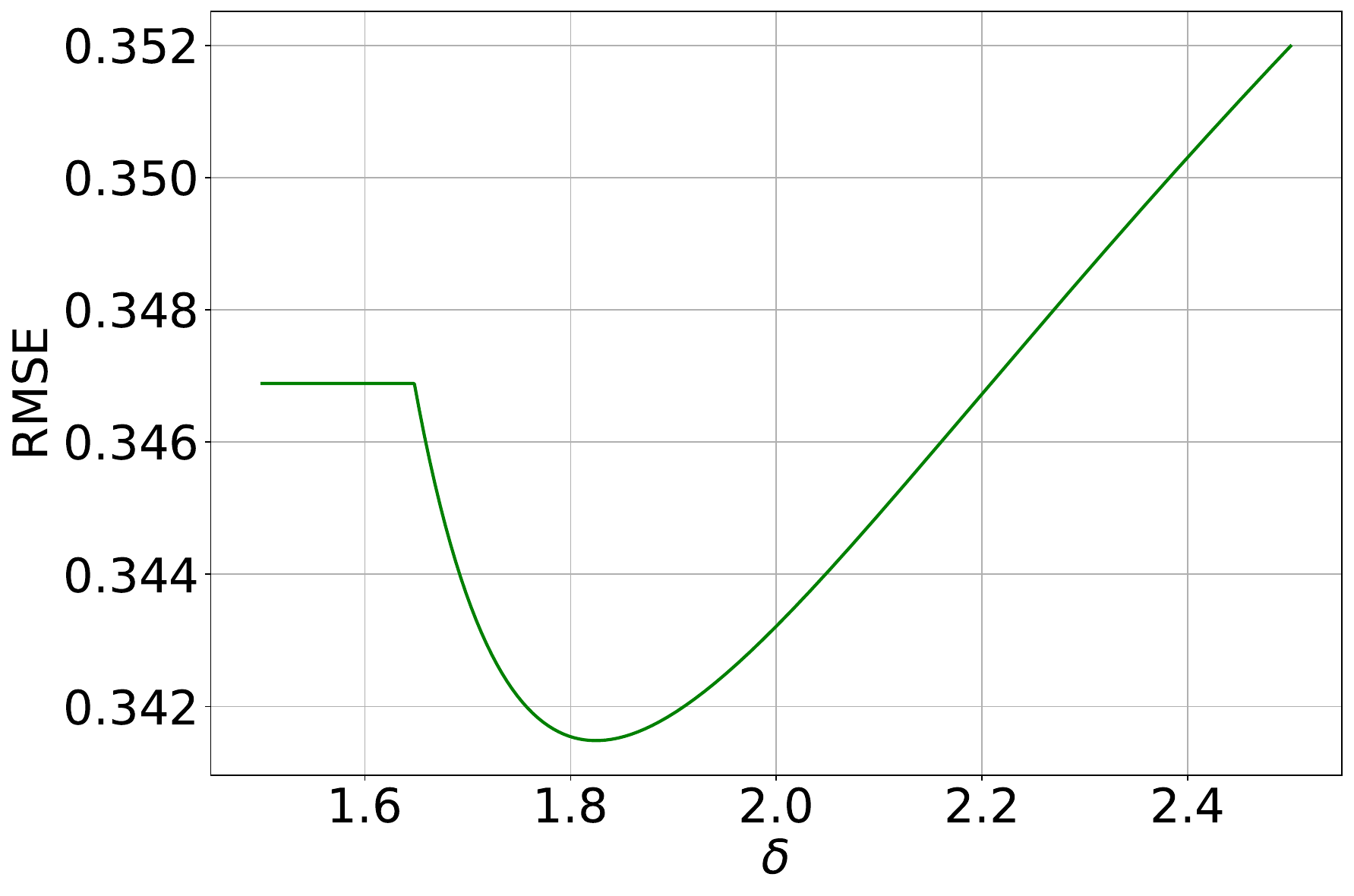}
 \end{minipage}
 \caption{\textbf{Left}: There are $2(k+1)n$ samples in total and the middle region
   in pink is the overlap region. See
   Appendix~\ref{sec:analytic-example-details} for details. \textbf{Right}:
   RMSE of the estimator $\hat{\tau}_{\delta}(\bm{w})$ vs $\delta$ with
   $n=25, k = 10, L = 1, \overrange = 0.1, \xi = 0.01$.  }
\label{fig:example-toy-analytic}
\end{figure}

To illustrate, consider a simple dataset where the data is divided into
overlap and non-overlap regions; in the non-overlap region, no samples receive
treatment and the minimax estimator relies on the treatment data from the
overlap region to perform extrapolation.  Figure~\ref{fig:viz-counterfactual}
shows how extrapolation is being conducted by $\MP$: matching weights
$W_{jk}$~\eqref{eqn:impute-conterfactual-more-specific} are computed only
using $(X,Z)$ data and nonzero only if the corresponding Lipschitz
constraints~\eqref{eqn:def-F-L-Lip-finite-sample} for point $(j,k)$ are
binding.

\paragraph{Extrapolation in regions of non-overlap}
In particular, consider a one-dimensional setting with two clusters of
covariates $x \in \R$ ($n$ points in total): one with positive weights and the other without. That
is, there exists $t \in \R$ such that for all $x_i \le t$, we have $w_i = 0$
and for all $x_i >t$, we have $w_i = \frac{1}{n}$.  We can prove that our
minimax procedure respects the geometry of the covariate space when
extrapolating: if it chooses to extrapolate using a unit $j$ farther from the
non-overlap region, then it must also be using all closer points.
Recall the modulus of
continuity on the non-overlap region (assuming $\sigma^2 = 1$ for simplicity)
  \begin{align}
    \begin{split} 
      \omega(\delta) =  \max_{f \in \mc{F}_L} 
      & ~2 \sum_{i=1}^n 
      w_i
      (f(x_i,1) - f(x_i,0))   \\
      \mathrm{s.t.}  & \quad \sum_{i=1}^n f(x_i,z_i)^2
      \le \frac{\delta^2}{4},
    \end{split} \label{eqn:toy-program}
                       ~~~~\mbox{where}~~~~  w_i = \frac{1}{n} \I{q_i < \epsilon}
  \end{align}
  and for all $i$ with $w_i = 0$, let
  $\overrange_i = \min_j \set{|x_i-x_j| \mid w_j > 0}$ denote its distance to the
  region with positive weights.  The following result shows that extrapolation
  starts from the boundary of the overlap region and proceeds inward.
\begin{lemma}
  Consider two units $i$ and $j$ with $w_i=w_j=0$ and $\overrange_i \ge \overrange_j$. Then for every
  $\delta > 0$, there exists an optimal solution $f\opt_\delta$ to the optimization
  problem~\eqref{eqn:toy-program} with
  $f\opt_\delta(x_j,1) \ge f\opt_\delta(x_i,1)$.
\label{lemma-extrapolate-start-from-near}
\end{lemma}

\begin{figure}[t]
 \begin{minipage}[b]{0.49\textwidth}
\centering \includegraphics[width=\textwidth, height=6cm]{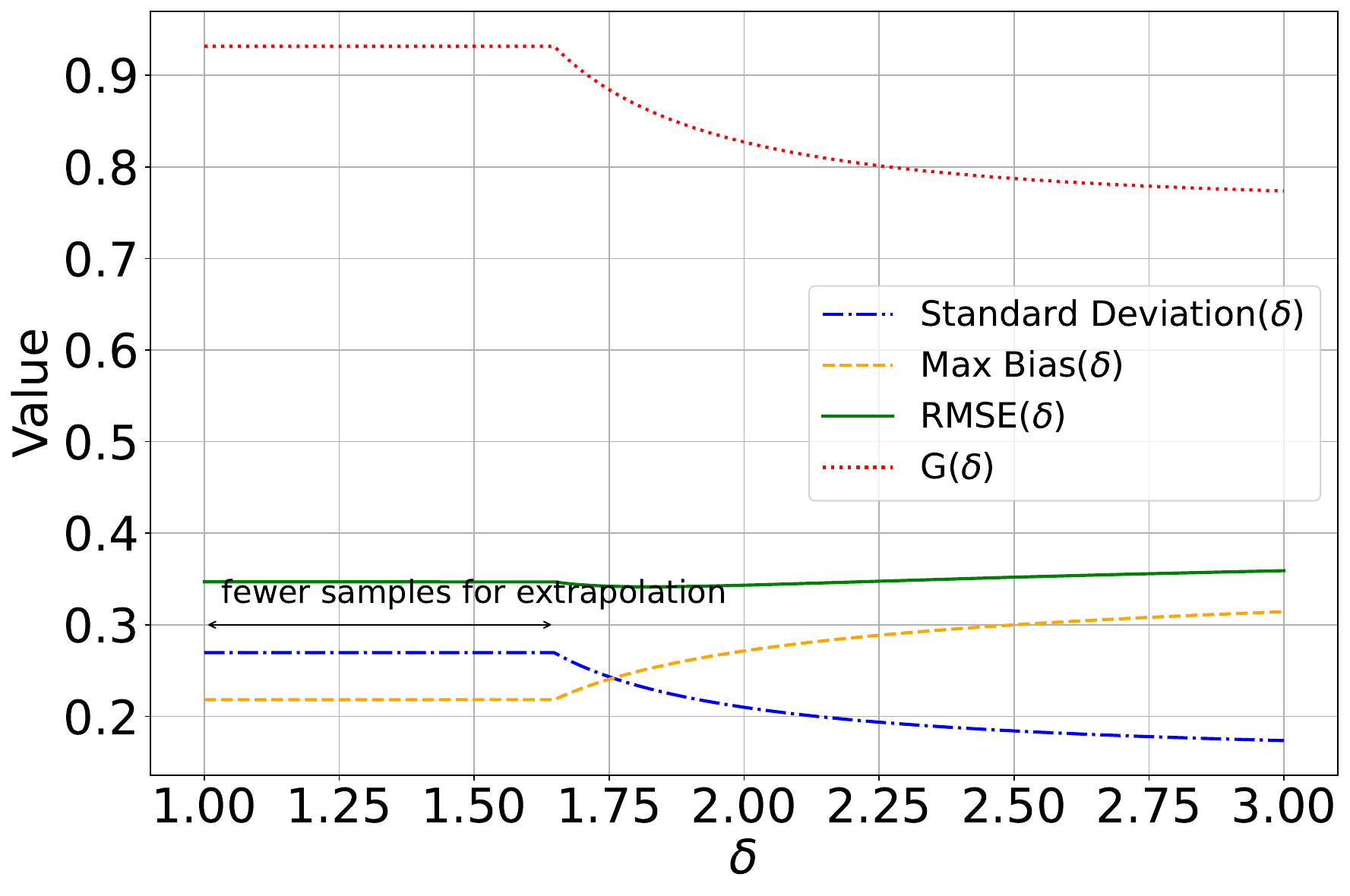}
\end{minipage}
\hfill
 \begin{minipage}[b]{0.49\textwidth}
 \centering \includegraphics[width=\textwidth, height=6cm]{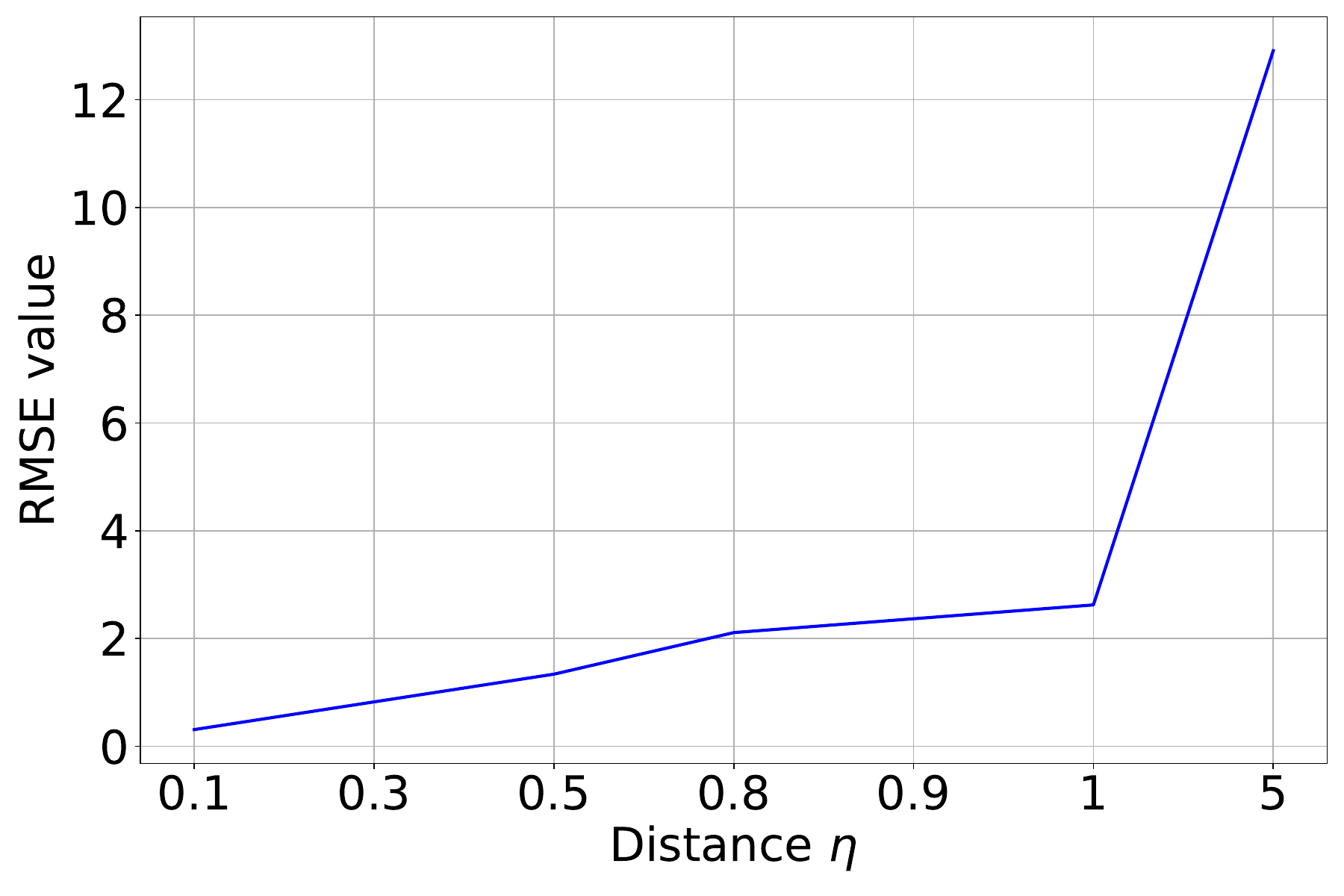}
 \end{minipage}
 \caption{\textbf{Left}: Bias, variance, and the length of the confidence
   interval~\eqref{eqn:CI-expression}
   $G(\delta) = \cv_\alpha\left(\frac{ \maxbias(\hat{\tau}_{\delta}(\bm{w}))
     }{\sd(\hat{\tau}_\delta(\bm{w}))}\right) \cdot
   \sd(\hat{\tau}_\delta(\bm{w}))$.
   \textbf{Right}: RMSE as we vary the
   distance parameter $\overrange$; for each $\overrange$, we compute the optimal (lowest)
   RMSE with respect to $\delta \ge 0$. }
\label{fig:example-toy-vs-eta}
\end{figure}

\paragraph{Estimation error as a function of $\delta$} Finally, we shed light
on the bias~\eqref{eqn:max-bias-equiv-formula} and
variance~\eqref{eqn:std-formula} of our minimax approach as a function of
$\delta$~\eqref{eqn:def-omega}, which controls the total weight the estimator
is allowed to place.  On the left panel of
Figure~\ref{fig:example-toy-analytic}, we consider $2n$ samples with treatment
assignments $z = 1$ and $z = 0$ located within the \emph{overlap region} (with distance $\xi$), and
an additional $2k \cdot n, k > 1$ samples located at a distance $\overrange$ outside of this region
(i.e., without overlap).  This structure allows us to compute the modulus of
continuity $\omega(\delta)$ in closed form for all $\delta > 0$, which we then
use to analyze the worst-case bias and variance.

To understand the bias-variance trade-off more closely, in the left panel of
Figure~\ref{fig:example-toy-vs-eta}, we plot the \emph{worst-case bias}
$\maxbias(\hat{\tau}_{\delta}(\bm{w}))$ and the \emph{standard deviation}
$\sd(\hat{\tau}_{\delta}(\bm{w}))$ separately.  
\begin{itemize}
\item When $\delta$ is \emph{small}, the estimator 
uses fewer samples for extrapolation, 
so it has lower bias and higher variance.
\item When $\delta$ is \emph{large}, 
the estimator uses more samples for extrapolation, 
reducing variance at the cost of increased bias.
\end{itemize}
In the right panel of Figure~\ref{fig:example-toy-vs-eta}, we further explore how the optimal RMSE obtained by
optimizing over $\delta$ scales with the lack of overlap $\overrange$. As $\overrange$
increases, extrapolation becomes less reliable, leading to higher RMSE.

\begin{figure}[t] 
  \begin{minipage}[b]{0.49\textwidth}
    \centering
   \includegraphics[width=\textwidth, height=7cm]{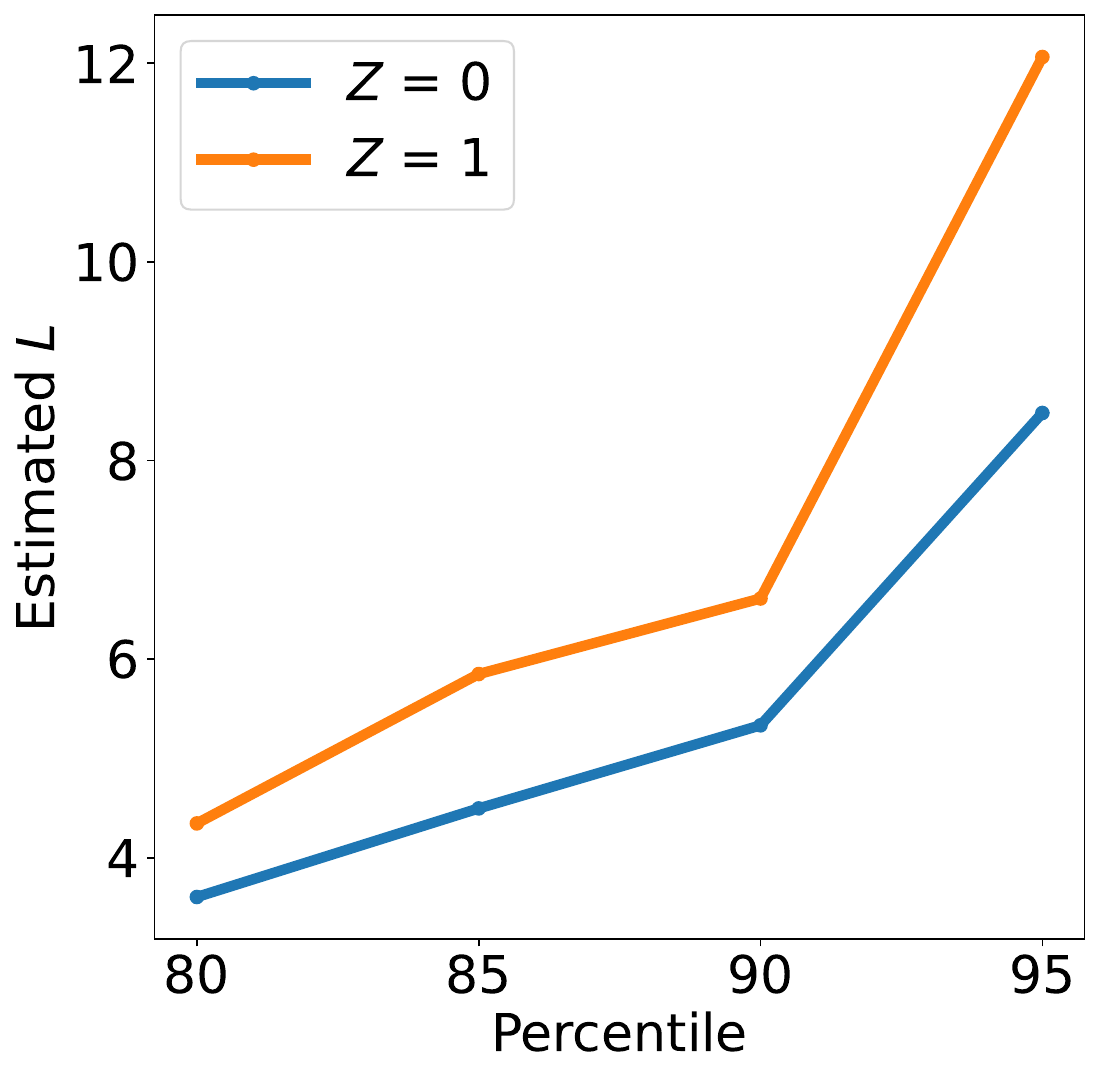}
  \end{minipage}
  \begin{minipage}[b]{0.49\textwidth}
    \centering 
    \includegraphics[width=\textwidth, height=7cm]{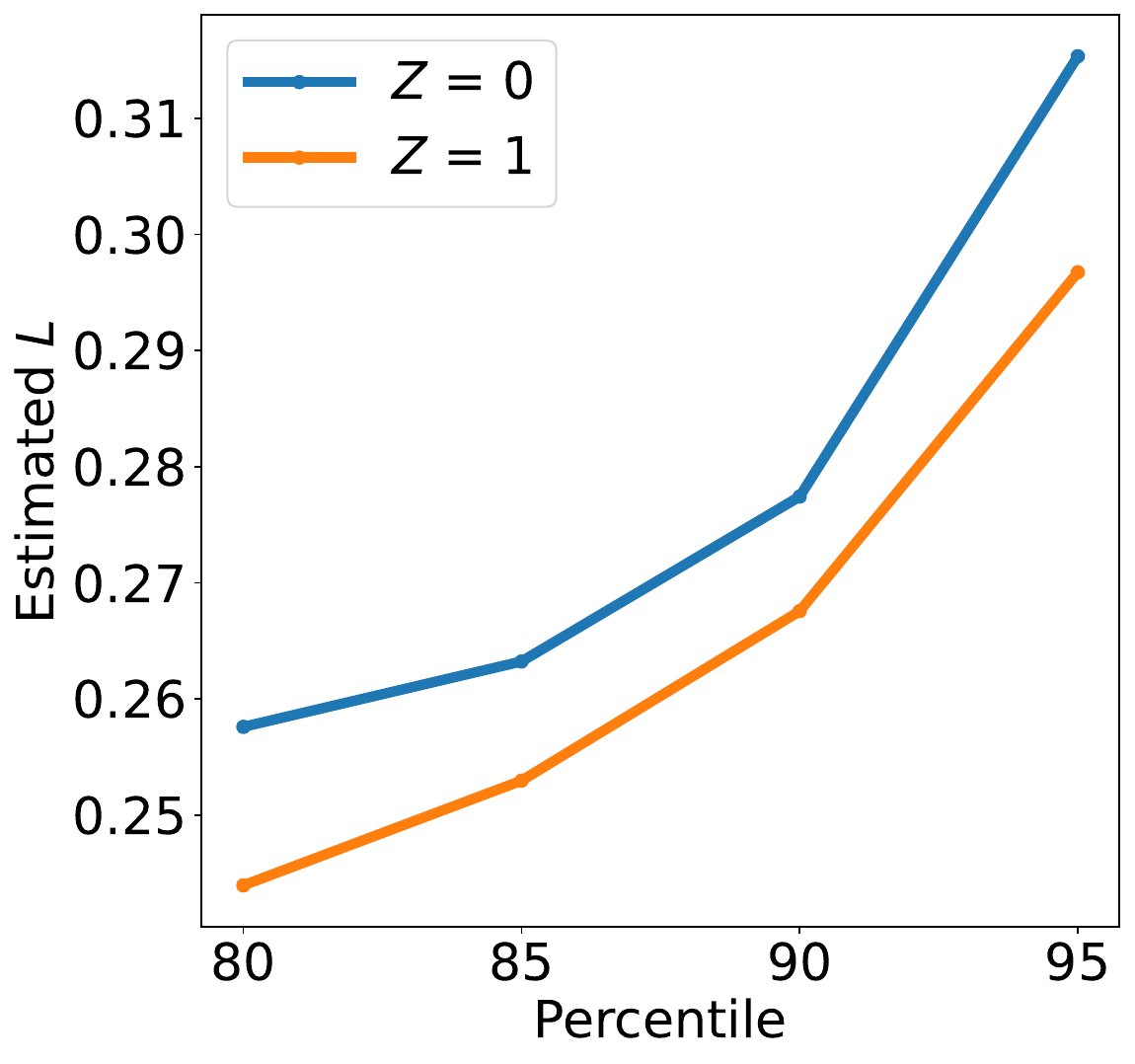}
  \end{minipage}
  \caption{
    Contextualization of the Lipschitz constant 
    $\Tilde{L}_{Z,p}$~\eqref{eqn:contextualize-L-def} 
    vs $p$. \textbf{Left}: Example~\ref{example:simulation}.
    \textbf{Right}: PennUI dataset in Section~\ref{section:experiment}
  }
  \label{fig:sim_contextualize}
\end{figure}

\section{Sensitivity analysis} \label{sec:proposed-approach}
 
We use the $\MP$ framework to facilitate a diagnostic analysis assessing how
different levels of assumed smoothness affect the estimate of bias due to
trimming.  Our framework asks
\emph{how strong the extrapolation assumption must be
 for the induced bias of the asymptotic estimator to remain
  negligible?}  Instead of committing to a level of $L$---which is inherently
challenging---our sensitivity framework varies the Lipschitz constant $L$ to
assess how the minimax interval $\MP_{\epsilon,L}$ evolves in length and
location.  This allows analyzing the relationship between the assumed
smoothness of the outcome function and the potential bias due to trimming.
Small values of $L$ correspond to strong smoothness assumptions—implying high
extrapolability, while larger values represent weaker assumptions and lead
to wider, more conservative intervals.

\begin{figure}[t] 
 \begin{minipage}[b]{0.49\textwidth}
\centering \includegraphics[width=\textwidth, height=7cm]{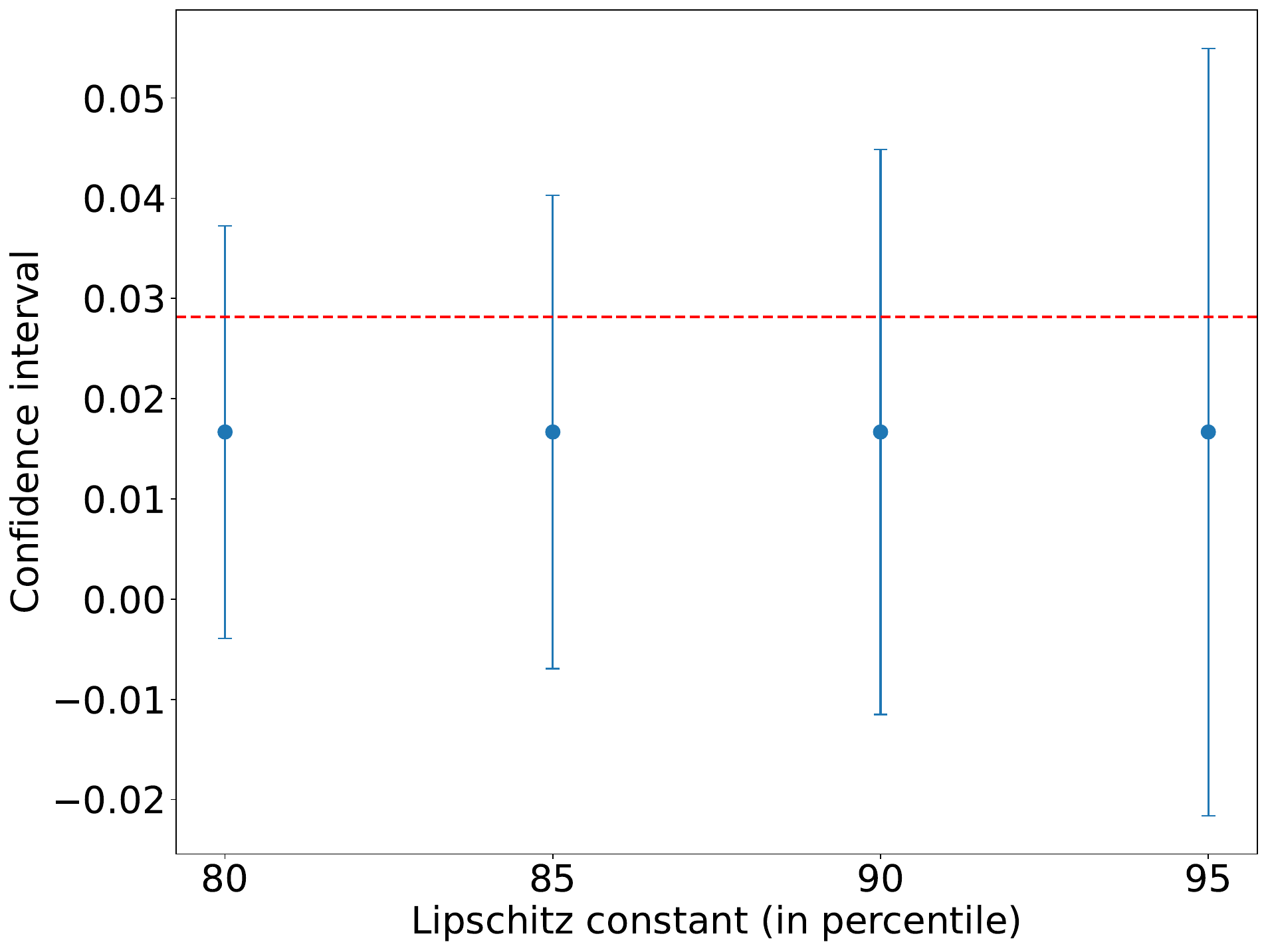}
\end{minipage}
 \begin{minipage}[b]{0.49\textwidth}
\centering \includegraphics[width=\textwidth, height=7cm]{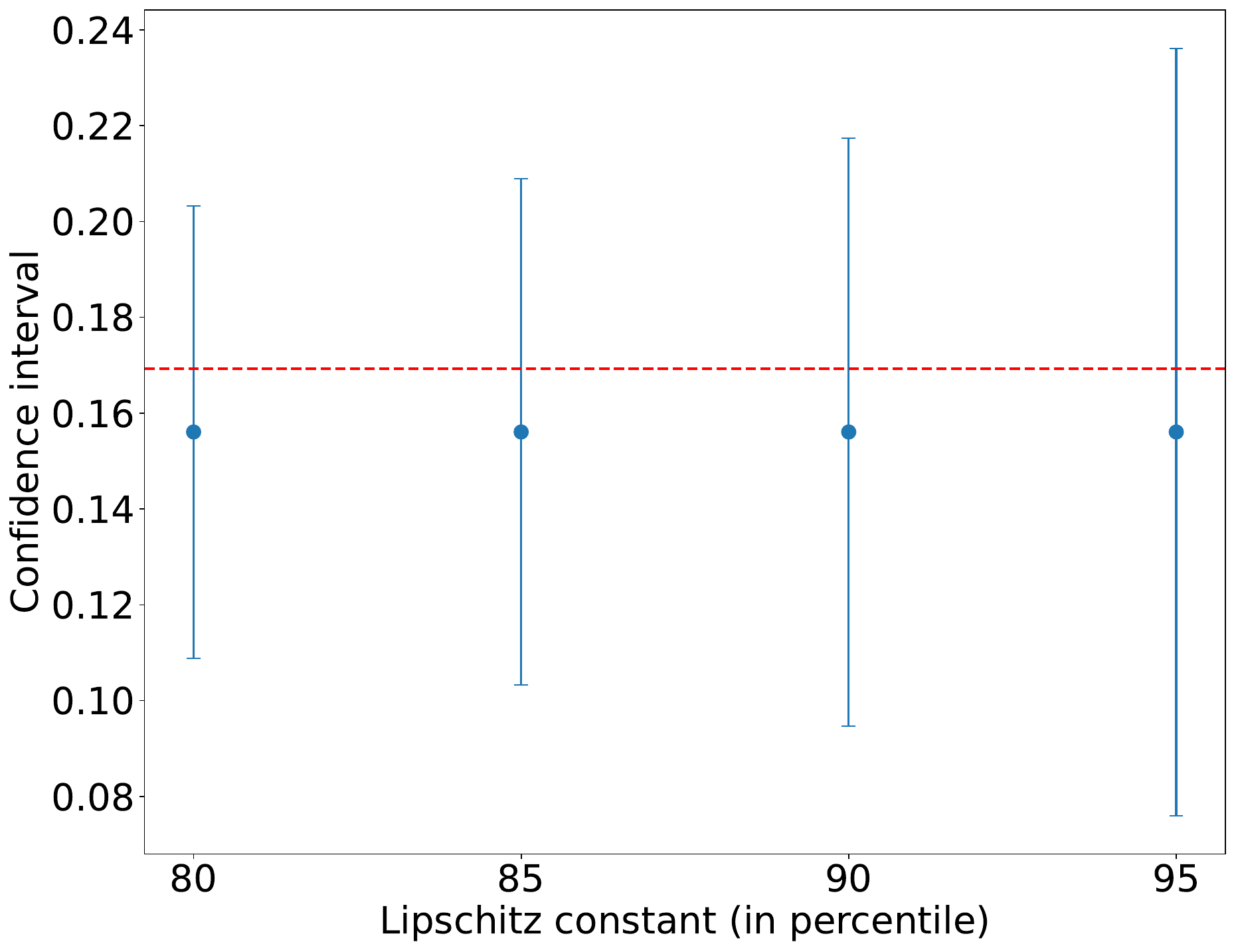}
\end{minipage}
 \caption{
 Data comes from Example~\ref{example:simulation}.
 On the left, we have confidence intervals generated by $\MP_\epsilon$
 at different  values of 
 $L$ with $\epsilon = 0.01$.
  We see how increasing the value of $L$ covers the estimand 
 $\tau_-$ and this enables the sensitivity analysis.
 On the right, we have results 
 for $\M$. 
 As we can see, the confidence intervals $\M$ are much wider than 
 $\MP$.
  }
\label{fig:simulation-M-sensitivity}
\end{figure}

To contextualize the sensitivity curve, we use the overlap region (where we
trust the predicted outcome regressor) to estimate the smoothness of the
outcome function, and use this estimate to guide how to properly model the
outcome function in the non-overlap region.  A uniform Lipschitz constant over
all data points (i.e., taking the maximum instead of the quantile
in~\eqref{eqn:contextualize-L-def}) may be overly conservative due to data
noise.
We consider the $p$-th quantile of
$\frac{\hat{\mu}_Z(x_i) - \hat{\mu}_Z(x_j)}{\norm{x_i- x_j}}$ (where
$\hat{\mu}$ is the fitted outcome regressor using data in the overlap region)
\begin{align}
  \Tilde{L}_{Z,p}= \max_{i: q(x_i) \ge \epsilon} \set{p-\mathsf{th~quantile } \set{j \ne i, q(x_j) \ge \epsilon:
  \left|
  \frac{\hat{\mu}_Z(x_i) - \hat{\mu}_Z(x_j)}{\norm{x_i- x_j}}
  \right|
  }}, Z \in \set{0,1}.
  \label{eqn:contextualize-L-def}
\end{align}

Since contextualizing the Lipschitz constant $L$ is challenging, the analyst
can instead think of the more intuitive quantity $p \in (0,1)$ that controls
the functional class. For simplicity, we consider the uniform Lipschitz
constant estimate for $Z \in \set{0,1}$ and consider
\begin{align}
\Tilde{L}_p = \max_{Z \in \set{0,1}} \set{\Tilde{L}_{Z,p}}.
\label{eqn:contextualize-L-def-uniform}
\end{align}
Using $\Tilde{L}_p$ in our framework, the analyst can better articulate the
function spaces parameterized by $L$.  In addition, the choice of $p$ to make
the confidence intervals valid also forms useful domain knowledge and can be
potentially transferred to other datasets.
In Figure~\ref{fig:sim_contextualize}, we use Example~\ref{example:simulation}
to numerically illustrate this contextualization approach.  We observe that
the percentile-based contextualization procedure offers an interpretable and
data-driven way to vary the extrapolation assumptions and examine their impact
on inference.

\begin{figure}[t]  
  \centering \includegraphics[width=0.5\textwidth, height=7cm]{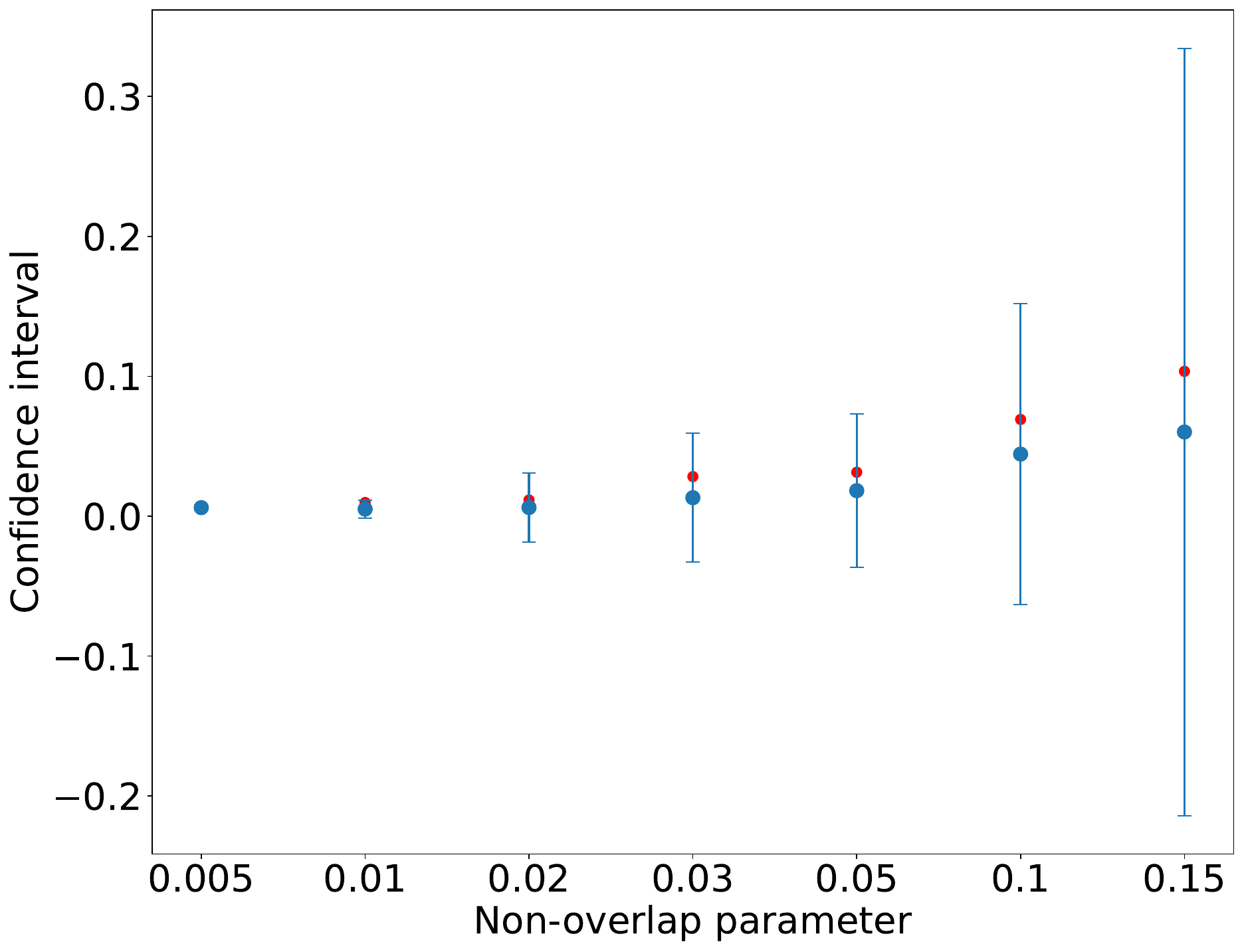}
  \caption{ 
    We plot 
    $\MP_\epsilon$
    with a fixed Lipschitz constant $L = 14$
    with different values of non-overlap parameter $\overrange$. 
    In this plot, red points represent the non-overlap ATE at each value of the non-overlap parameter. Blue points/intervals represent
    the point estimate/confidence interval 
    of $\MP$.
    At every point, we assume 
    the truncation threshold is chosen to minimize the length of the $\AIPWP$ interval.
    This shows the $\MP$ interval adapts to 
    data with different levels of overlap.
  }
  \label{fig:simulation-M-sensitivity-vary-a}
\end{figure}

\paragraph{Illustrating our sensitivity framework using
  Example~\ref{example:simulation}}
Figure~\ref{fig:simulation-M-sensitivity} shows $\MP_{\epsilon,L}$ across
Lipschitz constants $L$ for a fixed $\epsilon = 0.01$, which is
the threshold that
minimizes the length of the confidence interval $\AIPWP$ as
in~\eqref{eqn:eps-opt-def}; we consider $L$ ranging from
$\hat{L}_p$~\eqref{eqn:contextualize-L-def-uniform} with $p=0.80$ to $p=0.95$,
which roughly corresponds to $L$ ranging from $L = 5$ to $L = 14$ from
Figure~\ref{fig:sim_contextualize}.
 
Our analysis shows how gradually increasing $L$---and thus weakening the
smoothness assumption---leads to wider confidence intervals from
$\MP_{\epsilon,L}$, but also improves coverage of the ATE in the non-overlap
region $\tau_-$.  This trend reflects that under stronger smoothness
assumptions (lower $L$), the minimax procedure may underestimate uncertainty,
whereas more conservative values of $L$ yield intervals that reliably contain
$\tau_-$.  The sensitivity framework clearly allows the analyst to  
better understand the assumptions.  For example, at the 80th percentile, we
see the   $\MP_\epsilon$ interval is almost entirely above 0, suggesting
a larger chance of having a positive bias.  For a larger percentile, or being
more conservative, we see $\MP$ becomes much wider, and the potential
magnitude of the bias also increases.  Concretely, in
Figure~\ref{fig:simulation-M-sensitivity-vary-a} we summarize our minimax
partial estimator $\MP$ at the 95th percentile $L$ ($L = 14$) where we observe
reliable coverage while maintaining reasonably short confidence intervals.

\paragraph{Combining confidence intervals}
 
If the analyst deems the asymptotic approach to be   untrustworthy, they can
either modify the estimand so that excluding the uncertain subpopulation is
acceptable, or adopt a hybrid confidence interval that employs a conservative
union bound to combine the intervals from the minimax and asymptotic
approaches.  Specifically, in
Figure~\ref{fig:simulation-M-sensitivity-vary-a}, when the non-overlap
parameter is $0.03$, the length of $\MP$ is even wider than $\AIPWP$ and the
analyst can consider improving coverage while mitigating the excessive length
often associated with purely minimax-based intervals.  Using the union bound
$\P(X + Y \ge x+ y) \le \P(X \ge x) + \P(Y \ge y)$, we can add the upper
bounds of $\AIPWP$ and $\MP$ at levels $\alpha/2$ (i.e., coverage with
probability $1- \alpha/2$) to obtain an $\alpha$-confidence interval of $\MC$:
the combined $\alpha$-confidence interval of $\tau$ is
\begin{align}
  \MC^{(\alpha)}
  = \big[ a_1 +b_1, a_2 + b_2
  \big],~~~\mbox{where}~~\AIPWP^{(\alpha/2)} = \left[ a_1, a_2\right], ~~
  \MP^{(\alpha/2)} = \left[ b_1,b_2 \right].
  \label{eqn:MC-def}
\end{align}

\begin{figure}[t]
  \centering
  \begin{minipage}[b]{0.46\textwidth}
    \centering
    \includegraphics[width=\textwidth]{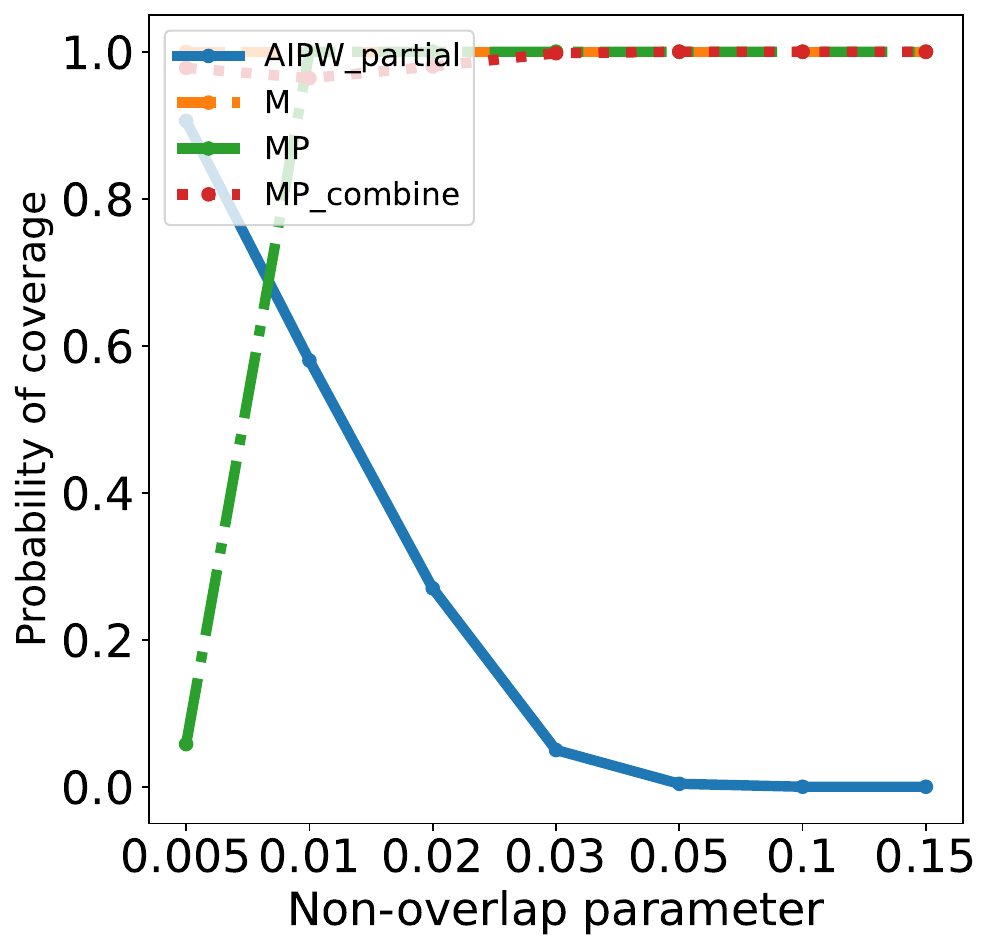}
    % \caption*{Prob of coverage}
    % \label{fig:sim-prob-coverage}
  \end{minipage}
  \hfill
  \begin{minipage}[b]{0.49\textwidth}
    \centering
    \includegraphics[width=\textwidth]{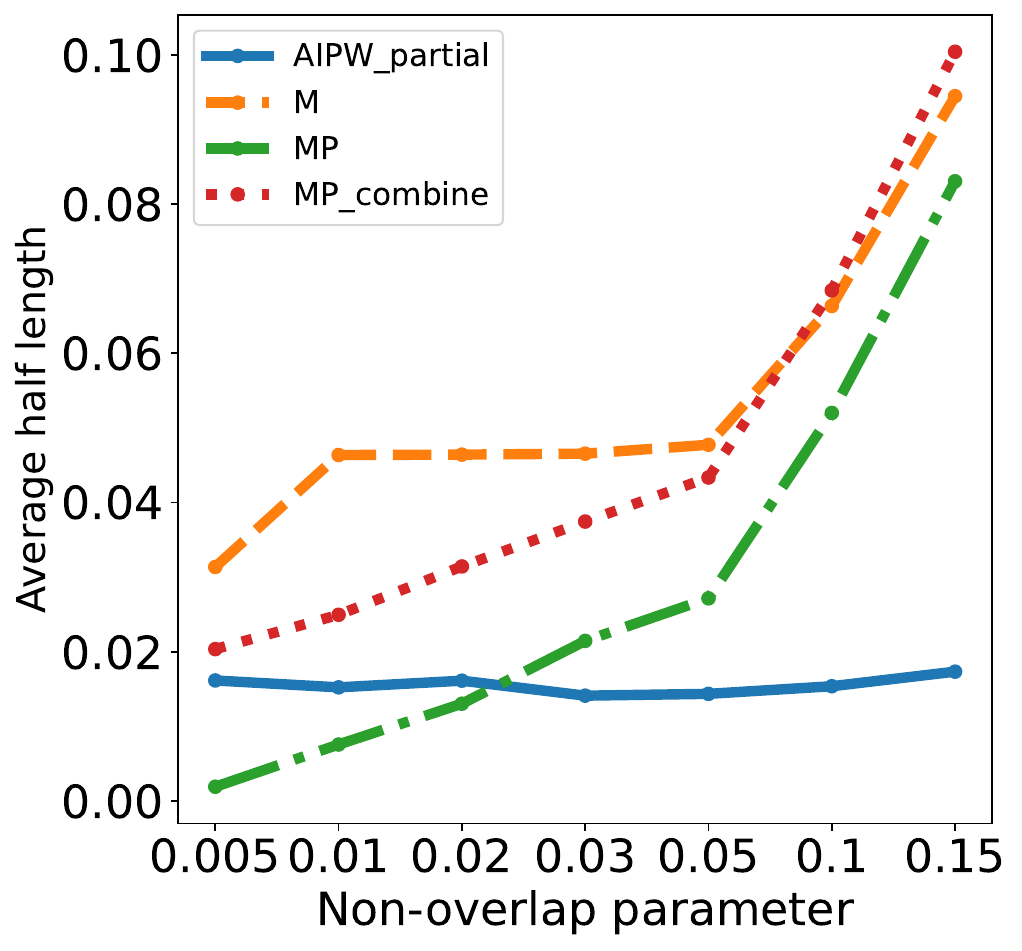}
    % \caption*{Half length of interval}
    % \label{fig:sim-length}
  \end{minipage}
  \caption{Probability of coverage and half length for $L=14$, as we vary the
    non-overlap parameter $\overrange$~\eqref{eqn:sim-prop}.  The truncation threshold
    is always chosen to be the one such that $\AIPWP$ has the smallest
    length. Our method $\MP$ helps analyze the unreliability of $\AIPWP$. }
  \label{fig:simulation-results}
\end{figure}

\paragraph{Coverage probability}
To assess the coverage probability of different inferential approaches, we
evaluate the performance of different confidence interval procedures across
repeated trials in Figure~\ref{fig:simulation-results}.  We observe a sharp
decline in the coverage probability for $\AIPWP$ as $\overrange$ increases from $0.005$
to $0.03$.  This highlights the fragility of asymptotic methods under
violation of the overlap assumption and underscores the utility of our
framework in providing reliable inference where standard estimators fail.

Despite achieving the desired coverage, the full minimax procedure $\M$ is
overly conservative: its length is significantly larger than necessary and
does not meaningfully reflect the uncertainty specifically attributable to the
non-overlap region. In contrast, the interval from $\MP$, which targets
$\tau_-$, is substantially narrower and better calibrated to the uncertainty
in the trimmed region.  The combined interval $\MC$ achieves valid coverage
for the full ATE $\tau$, with a shorter length (higher power) than the
conservative interval from $\M$ for small values of $\overrange$.

%%% Local Variables:
%%% mode: latex
%%% Tex-master: "main"
%%% End:

\begin{figure}[t]
  \begin{minipage}[b]{0.49\textwidth}
    \centering 
    \includegraphics[width=0.85\textwidth]{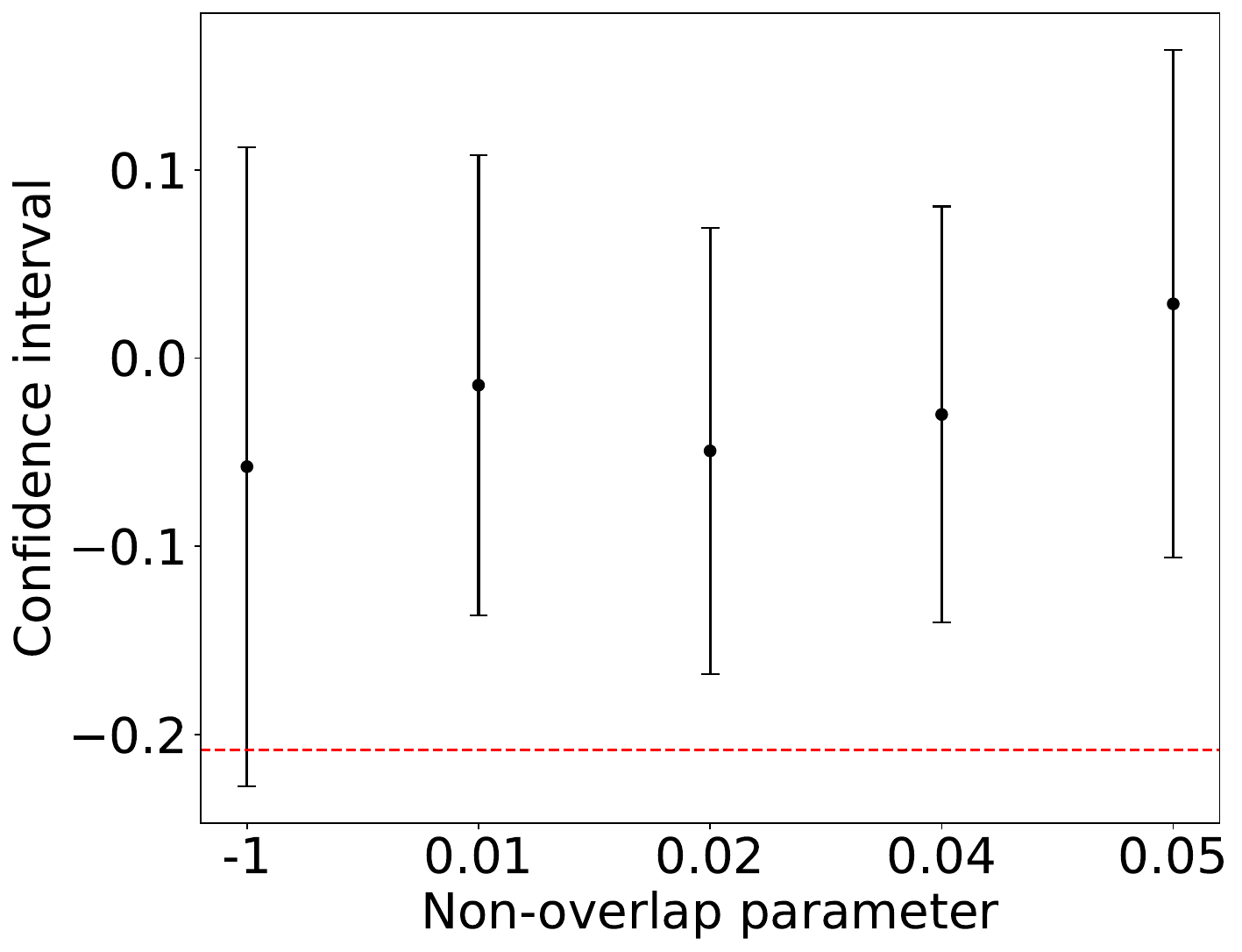}
  \end{minipage}
  \hfill
  \begin{minipage}[b]{0.49\textwidth}
    \centering \includegraphics[width=0.8\textwidth]{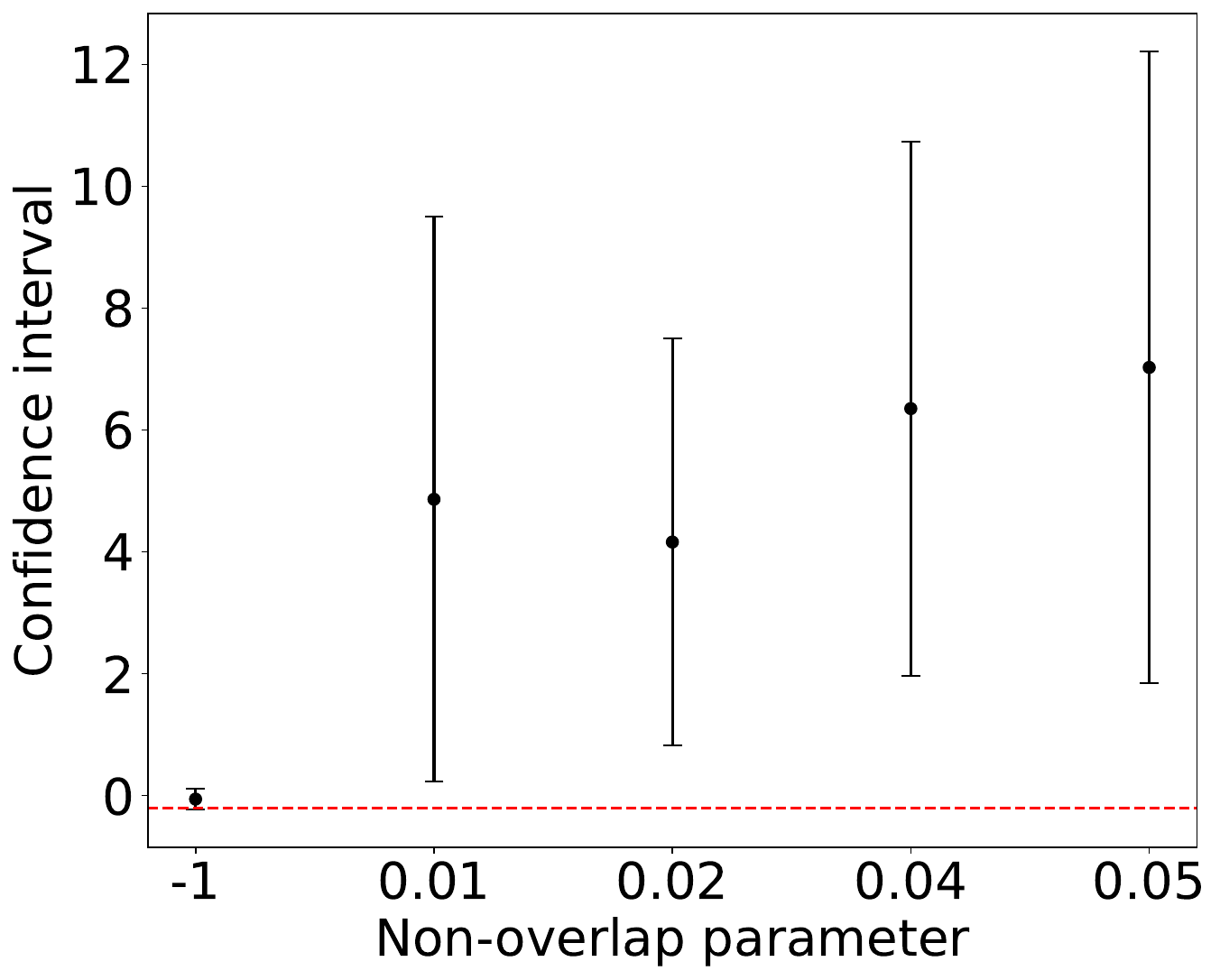}
  \end{minipage}
  \caption{Failure of $\AIPWP$ on observational data simulated from the PennUI
    dataset.  Similar to Figure~\ref{fig:simulation-AIPW-viz}, we compare the
    confidence interval from $\AIPWP$ for data generated with different levels
    of $\overrange$ in~\eqref{eqn:RCT-prop-E-k-def}, and $E_{-1}(x) = \half$ for all
    $x$.  \textbf{Left}: With truncation. \textbf{Right}: Without truncation.
  }
  \label{fig:RCT-AIPW-viz}
\end{figure}

\section{Case study}
\label{section:experiment}
 
We now demonstrate our framework through a case study on the Pennsylvania
Reemployment Bonus Demonstration (PennUI) dataset, a significant experiment
conducted from 1988 to 1989 to estimate the impact of financial incentives on
reemployment~\citep{BiliasKo02}.

\paragraph{Preliminaries}
In the PennUI dataset, approximately 15,000 eligible claimants were randomly
assigned to one of the six treatment groups or a control group.  The data
provides valuable insights for policymakers aiming to enhance reemployment
programs and UI (unemployment insurance) systems.  In this dataset, the outcome
variable $Y$ is the log of the duration of unemployment for the UI claimant
and covariates $X \in \R^{60}$ include demographics and employment service
program participation.  We take a subset of this data, considering only
treatment variable taking values in $\set{4,6}$ and further conduct a random
subsampling with $n = 539$ samples with roughly $48.3\%$ of the samples
received treatments.  As this data is a randomized experiment, we can use a
standard difference-in-means estimator to obtain the treatment effect
$\tau = -0.2$ on the entire population.

We simulate various observational datasets using a
sampling method with more details in Section~\ref{sec:experimental-details}.
To define the propensity score used in the sampling procedure, we first train
a T-learner using \texttt{Random Forest} regressors on the entire dataset to
obtain $\hat{\tau}(x_i)$ for all points $x_i$, then we map the value of
$\hat{\tau}(x_i)$ to obtain
$\pi_{\overrange}(x_i) = E_{\overrange}(\mathsf{percentile}(\hat{\tau}(x_i)))$ for each sample $i$
with  
\begin{align}
E_{\overrange}(x) = 
\begin{cases} 
\Uni(0.005, 0.03) & \text{if } x \leq 0.075+\overrange \text{ or } 1 - x \leq 0.075+\overrange \\
\Uni(0.03, 0.05) & \text{elif } x \leq 0.1+\overrange \text{ or } 1 - x \leq 0.1+\overrange \\
\Uni(0.05, 0.1) & \text{elif } x \leq 0.15+\overrange \\
0.5 & \text{otherwise}
\end{cases}
\label{eqn:RCT-prop-E-k-def}
\end{align}
We assign extreme values of $\pi_{\overrange}(x_i)$ to those with extreme values of
$\hat{\tau}$ and it leads to observational datasets with limited overlap.
Unless otherwise specified, we assume $\overrange = 0.01$.  Now, for $\AIPW$ or
$\AIPWP$, we use a \texttt{Random Forest} regressor for the outcome
prediction.

Recall that the standard deviation $\sigma$ is an important input to our
procedure $\MP$.  To estimate it, we utilize the same estimator as in
\cite{AbadieIm06,ArmstrongKo21}, which is given by
\begin{align}
  \what{\sigma}^2\defeq \avgn  \widehat{\sigma}^2(x_i,z_i)~~~\mbox{where}~~~
  \widehat{\sigma}^2(x_i,z_i) =\frac{J}{J+1} \paran{Y_i - \frac{1}{J} \sum_{m=1}^J Y_{\ell_m(i)}}^2.
\end{align} 
Here, $\ell_m(i)$ is the closest unit to unit $i$ among the units with the
same treatment as unit $i$ and $J$ is a tuning parameter which we choose to
set as $J = 2$ following~\citep{ArmstrongKo21}.

\begin{figure}[t]   
\begin{minipage}[b]{0.49\textwidth}
\centering \includegraphics[width=\textwidth]{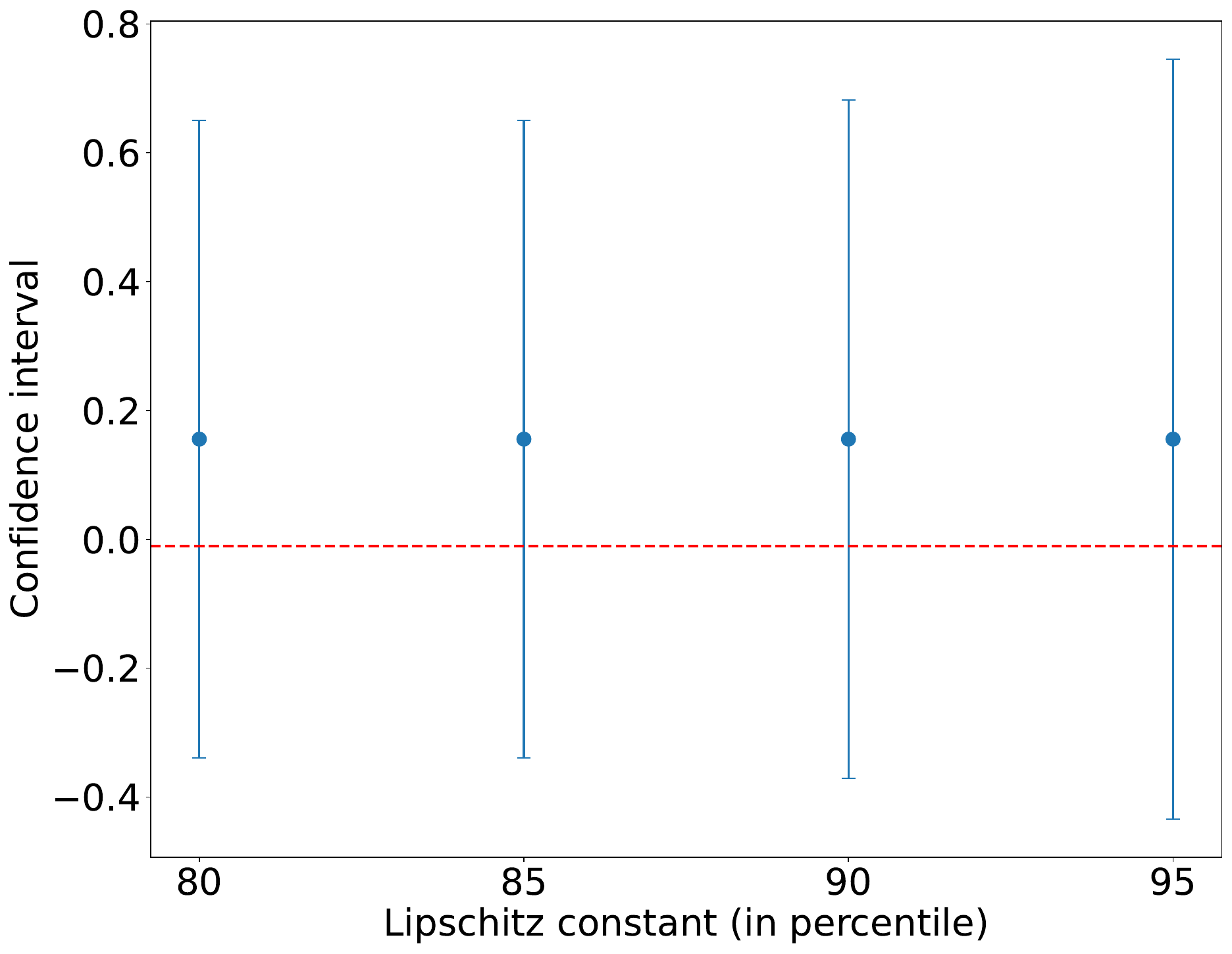}
\end{minipage}
\hfill
\begin{minipage}[b]{0.49\textwidth}
\centering \includegraphics[width=\textwidth]{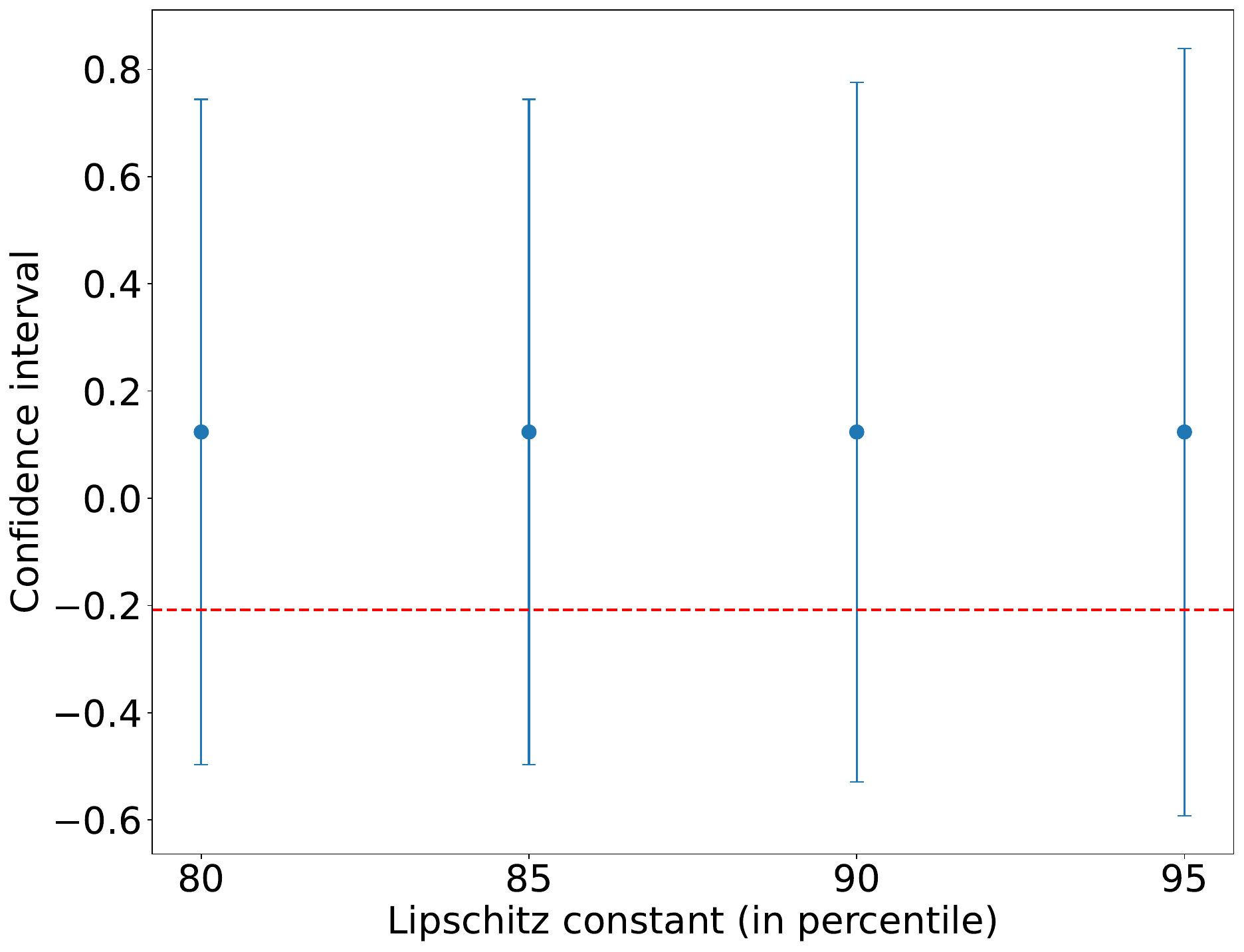}
\end{minipage}
\caption{Real data results: confidence intervals for
$\MP_{\epsilon\opt}$ (left)  and $\MC$ (right)
at different $L$ for $\overrange = 0.01$. % and $\epsilon\opt = 0.07$. 
The red dotted points are $\tau_-$ and $\tau$ respectively.
The data is the same as in Figure~\ref{fig:RCT-AIPW-viz}.
}
\label{fig:RCT-M-sensitivity}
\end{figure}

\paragraph{Failure of asymptotic estimators}
In Figure~\ref{fig:RCT-AIPW-viz}, we observe that
$\AIPWP$ does not cover $\tau$ as
the lack of overlap becomes more severe.  The output of $\AIPWP$ significantly
overestimates the treatment effect, and we see the upper bound of the interval
deviates significantly from the true treatment effect.  (Similar to our running
simulation example, we assume the analyst chooses the truncation threshold
from the set $\mc{E} = \set{0.01,0.02,0.03,0.04,0.05,0.06,0.07}$.)

\paragraph{Sensitivity framework} In contrast, our sensitivity approach
provides a conservative and valid analysis of the bias induced by the
truncation.  In Figure~\ref{fig:RCT-M-sensitivity}, we show $\MP_\epsilon$
across different values of the Lipschitz constant $L$ and see that it
successfully covers $\tau_-$ for $L = \hat{L}_{0.8}$ estimated from the
overlap region~\eqref{eqn:contextualize-L-def-uniform}.  The confidence
intervals $\MP_\epsilon$ are wider than those produced by $\AIPWP$, reflecting
the conservative nature of minimax-type procedures, which explicitly account
for the worst-case uncertainty due to extrapolation from the overlap region.

We also visualize the combined interval $\MC$ on the right side of
Figure~\ref{fig:RCT-M-sensitivity}.  As expected, $\MC$ intervals are wider
than those from $\MP$ alone, since they aim to cover the full ATE $\tau$
rather than only $\tau_-$ in the non-overlap region.  This increased width is
a necessary trade-off to ensure coverage of the entire target estimand.
Still, $\MC$ remains substantially more efficient than the fully conservative
minimax confidence interval $\M$, as it exploits the strengths of both the
asymptotic method (in regions with good support) and the minimax method (in
regions with poor overlap).

\begin{figure}[t]
\centering
\begin{minipage}[b]{0.49\textwidth}
\centering
\includegraphics[width=\textwidth]{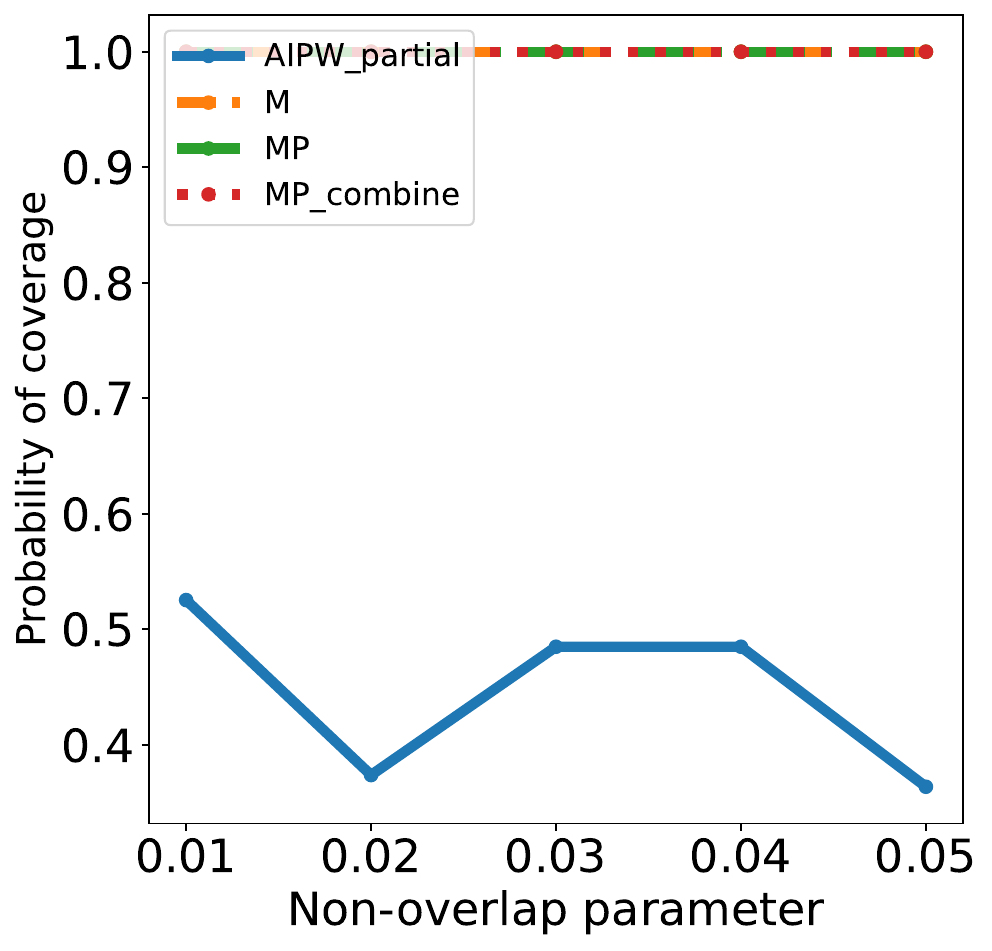}
\end{minipage}
\hfill
\begin{minipage}[b]{0.49\textwidth}
\centering
\includegraphics[width=\textwidth]{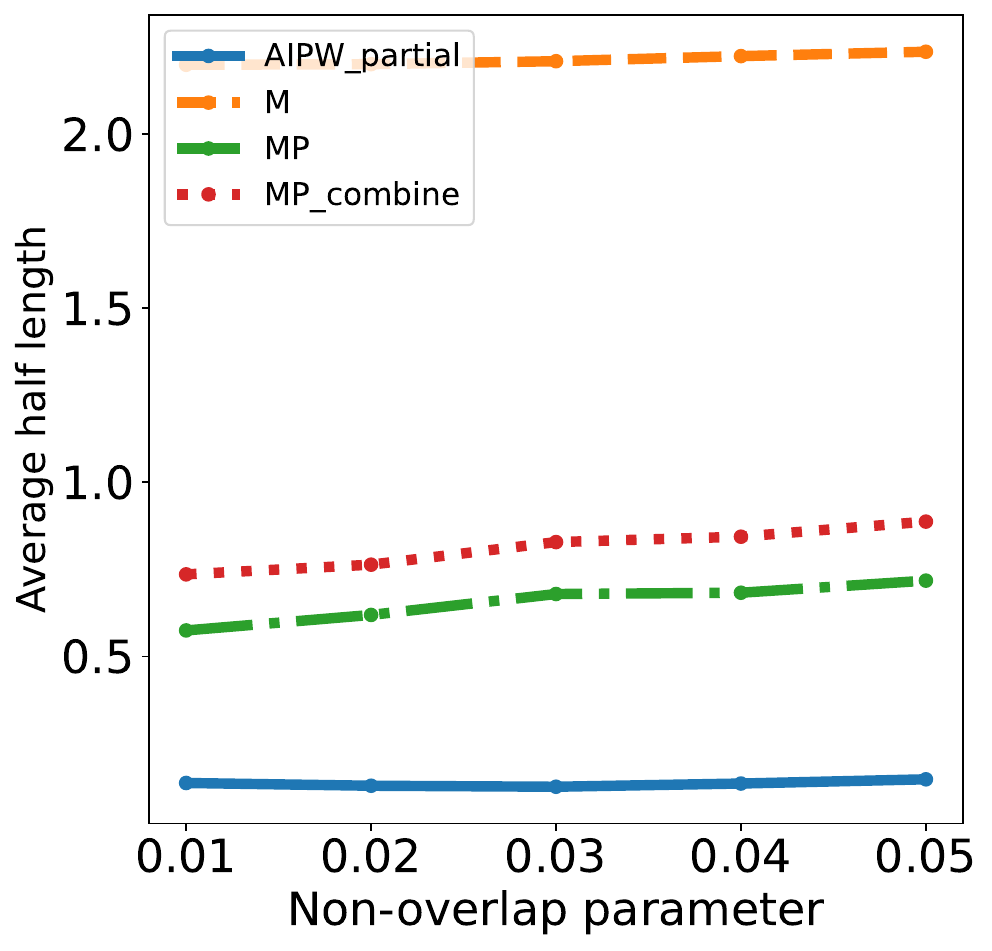}
\end{minipage}
\caption{Probability of coverage and half length for $\MP$ and $\AIPWP$ with
  $L = 0.32$, as we vary the non-overlap parameter $\overrange$.  Results are averaged
  over 100 runs.  The larger the non-overlap parameter, the larger the non-overlap
  region is.  The truncation threshold is always chosen to be the one such
  that $\AIPWP$ has the smallest length.  We can see how unreliable $\AIPWP$
  is and how our method $\MP$ can help mitigate it.  }
\label{fig:RCT-results}
\end{figure}

\paragraph{Coverage probabilities of confidence intervals}
To assess extrapolability, we contextualize the Lipschitz constant $L$ in the
right panel of Figure~\ref{fig:sim_contextualize}, and observe that the
underlying CATE function is much smoother than in our simulation example
due to the higher-dimensional nature of the covariates.  Taking the 95th
percentile
Lipschitz parameter, we study the coverage probability of specific confidence
intervals in Figure~\ref{fig:RCT-results}.  Probability of coverage for
$\AIPWP$ is highly sensitive to the degree of overlap in the data.  As the
degree of non-overlap increases, its coverage probability deteriorates
substantially, indicating a higher risk of undercoverage and, consequently,
misleading inference.  On the other hand, our method $\MP$, which leverages
the minimax framework over the non-overlap region, offers a more robust
estimate of the bias of $\AIPWP$.  Across all scenarios tested, $\MP$ achieves
full (100\%) coverage of the non-overlap region estimand $\tau_-$.  Similarly,
$\MC$ consistently achieves valid coverage for the full ATE $\tau$, while
remaining considerably narrower than the fully conservative interval generated
by the naive minimax method $\M$.  This illustrates a key strength of our
framework: it balances the need for robust inference in low-overlap regions
with the desire for efficiency where data support is strong.

%%% Local Variables:
%%% mode: latex
%%% TeX-master: "main"
%%% End:

\section{Guiding data collection}
\label{sec:continual-sampling-framework}

The notion of overlap studied in this work can potentially have implications
beyond observational analysis.  We explore one such scenario where the analyst
may   collect additional outcome data from units in the non-overlap
region, so that the non-overlap region shrinks over time.  Given some budget
to assign treatments $z_i$ to certain sample $i$, the realization of the
updated treatment $\Tilde{Z}_i(r)$ that depends on the sampling option $r$
yields the updated outcome $\Tilde{Y}_i$.  We provide preliminary ideas where
the notion of overlap we study can guide which batch of samples to collect.

 \begin{figure}[t] 
\centering \includegraphics[width=0.31\textwidth]{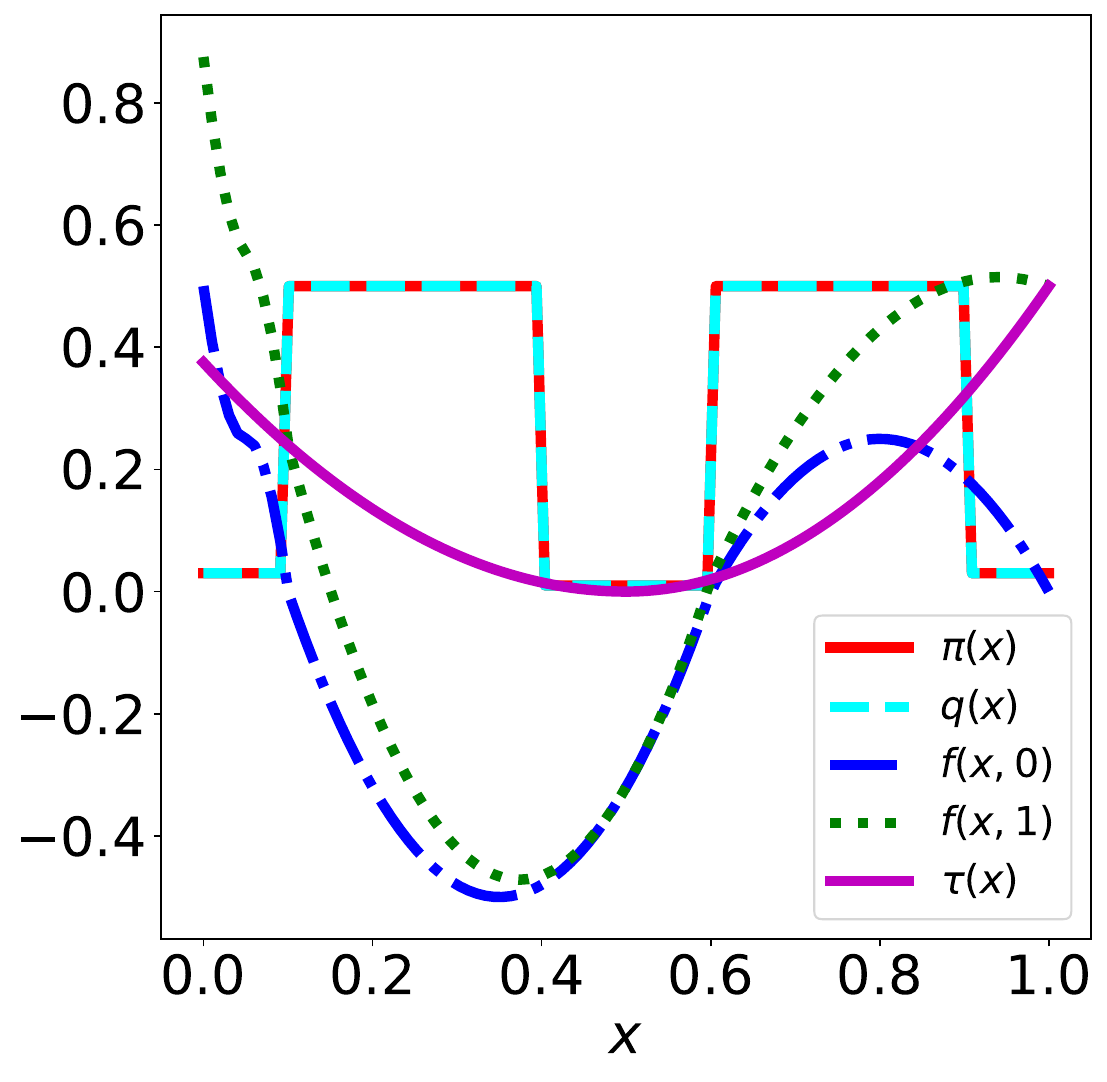}
\caption{Simulation data collection setup}
\label{fig:simulation-data-collection-dgp}
\end{figure}

\subsection{Data collection}
\label{sec:data-collection-two-option}

\begin{figure}[t]
\centering
\begin{minipage}[b]{0.49\textwidth}
\centering
\includegraphics[width=0.8\textwidth, height=6cm]{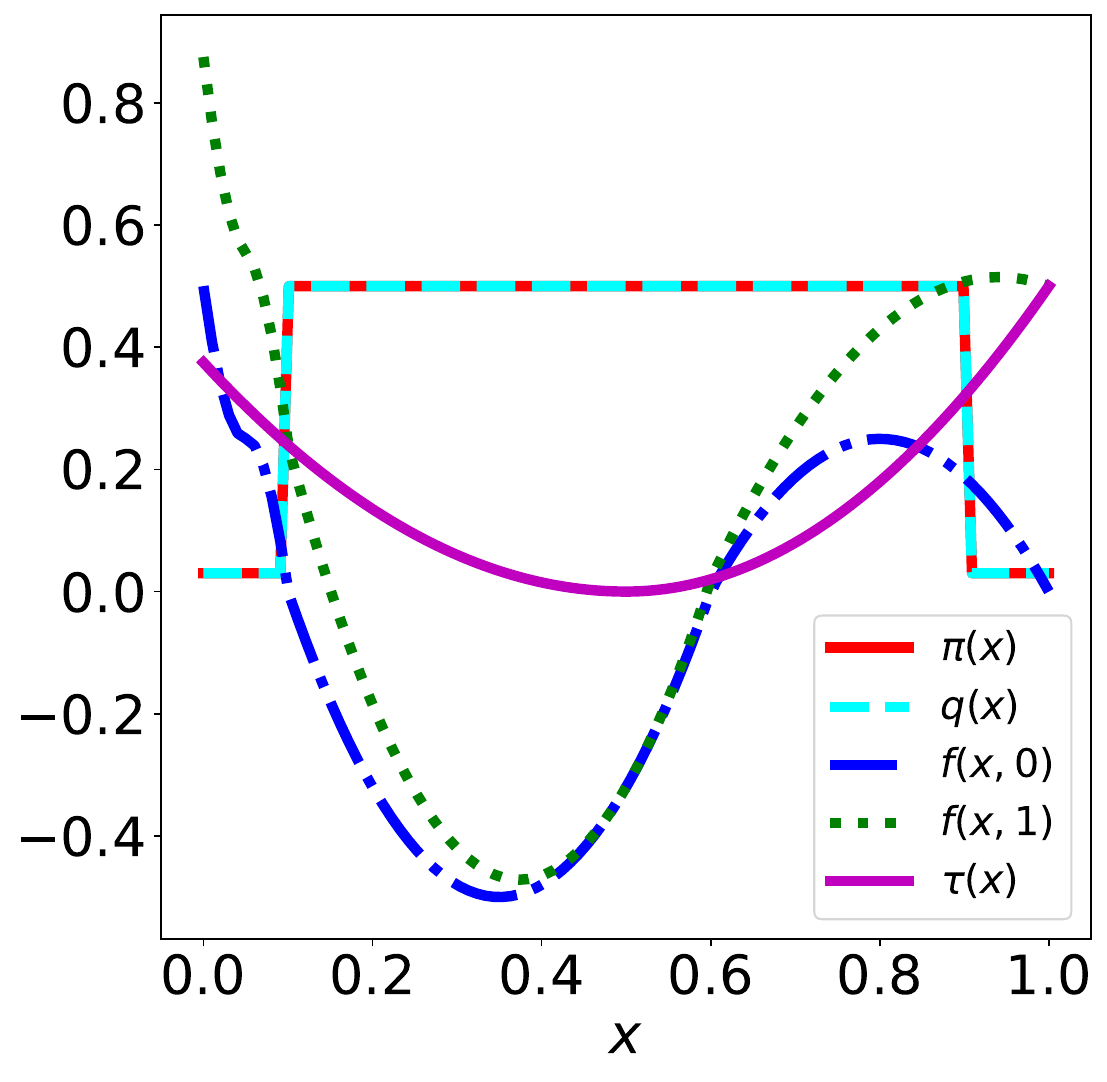}
\end{minipage}
\hfill
\begin{minipage}[b]{0.49\textwidth}
\centering \includegraphics[width=0.8\textwidth, height=6cm]{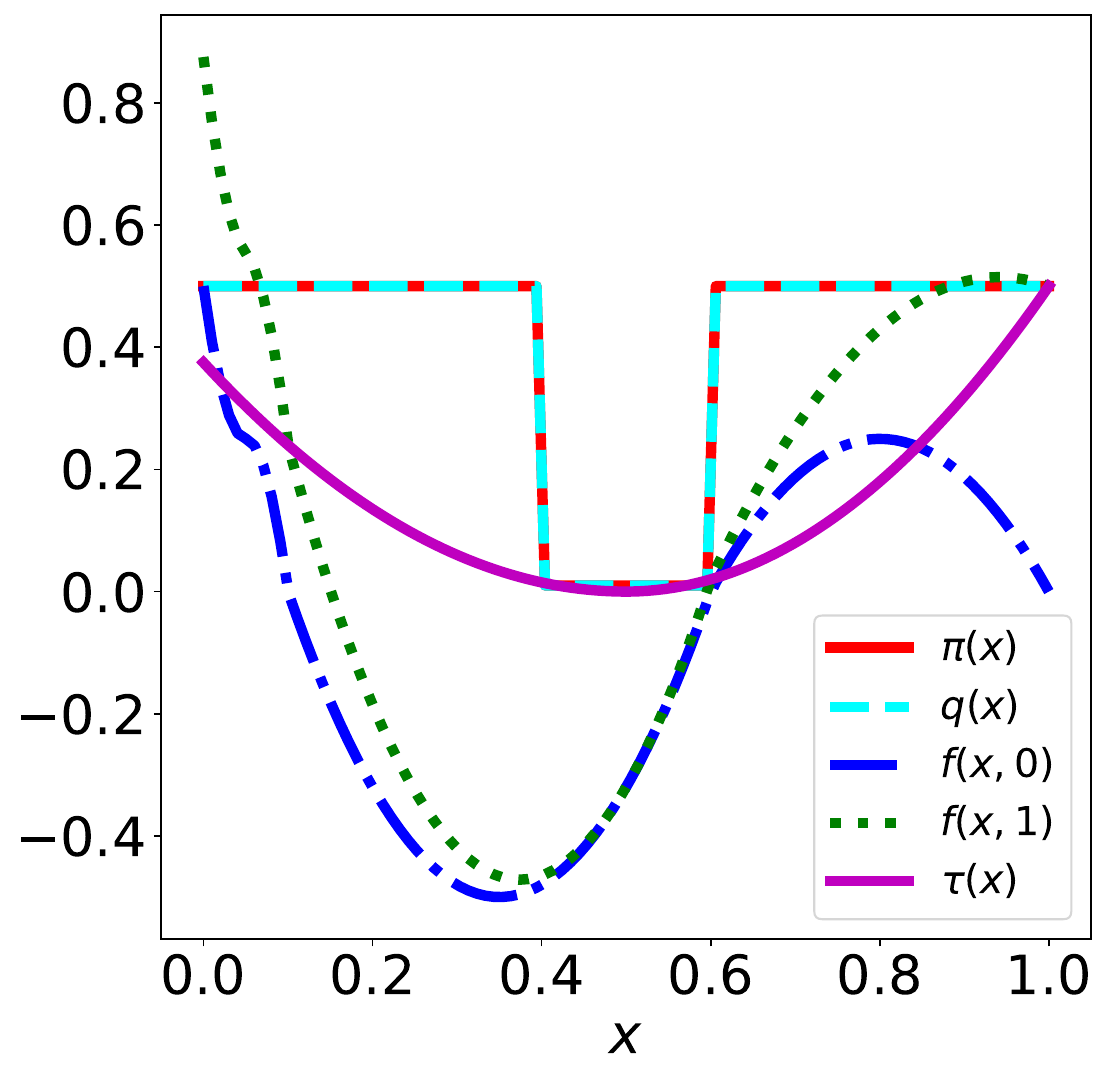}
\end{minipage} 
\caption{Left/right: after \texttt{Option\_1}/\texttt{Option\_2}, the
  distribution of $\pi(X)$.  \texttt{Option\_1} is to collect data if
  $x \in \mc{X}_2 =(0.4,0.6)$ and \texttt{Option\_2} is to collect data if
  $x \in \mc{X}_2=(0,0.1) \cup (0.9,1)$.  \texttt{Option\_1} appears better
  from the propensity score perspective (after sampling, the propensity score is
  0.03 for \texttt{Option\_1} vs 0.01 for \texttt{Option\_2}).  As we can see
  later, \texttt{Option\_2} is the better option.}
\label{fig:sim-data-collection-viz}
\end{figure}

As a starting point, we focus on a simple scenario where the analyst already
has two options in mind and needs to decide which one to proceed with.  For each
sampling option $r$, we can define the data
\begin{align*}
    \widetilde{\mc{D}}(r)  \defeq \set{x_i, \Tilde{\pi}_i(r) ,\Tilde{Z}_i(\Tilde{\pi}_i(r))}_{i=1}^n,
\end{align*}
as the realized data containing the updated treatment information,
where   $\Tilde{\pi}_i(r)$ denotes the updated propensity score for unit $i$
under option $r$.
Note that
the dataset $\widetilde{\mc{D}}(r)$ does not involve $Y$ and therefore can be
computed prior to the actual sampling.  This is because our approach
explicitly considers all possible outcome functions instead of relying on the
actual observed $Y$ data.  We can then compare different sampling options
through
\begin{align}
    \E_Z \left[T\left(\MP_\epsilon\paran{\widetilde{\mc{D}}(r)} \right) \right]\label{eqn:m-partial-sampling-metric}
\end{align}  
for some $\epsilon$, where $T(\cdot)$ is the length 
of the
interval   $\MP$  and
the expectation is taken over the randomness in $Z$.  For a fixed $\epsilon$,
the metric~\eqref{eqn:m-partial-sampling-metric} estimates the expected
uncertainty of regions with $q(x) < \epsilon$ after the additional sampling.

The simulation setup for this setting is depicted in
Figure~\ref{fig:simulation-data-collection-dgp}.  We randomly sample $n=500$
points as the initial data.  We imagine a scenario, with certain operational
constraints being present, the analyst faces two feasible sampling options to
collect new data, as depicted in Figure~\ref{fig:sim-data-collection-viz}:
\texttt{Option\_1} is to collect data if $x \in \mc{X}_2 =(0.4,0.6)$ and
\texttt{Option\_2} is to collect data if
$x \in \mc{X}_2=(0,0.1) \cup (0.9,1)$.  In addition, we consider another
policy \texttt{Oracle} that samples all data $x \in \mc{X}_1 \cup \mc{X}_2$,
which is the entire region without overlap, with probability 50\% for each point.
Therefore, the total number of data points
that can be sampled are roughly the same for all
three options.  As we see later, the policy \texttt{Oracle} performs well
since it minimizes the region of limited overlap, but such random sampling may
not always be feasible in practice.  
If the analyst decides to  collect some
data $x_i$, we assume she can set $\pi(x_i) = \half$.

\texttt{Option\_1} and \texttt{Option\_2}, the two sampling options,
are nearly
identical solely based on propensity scores as the propensity scores in these
regions are nearly the same.  This makes the options hard to distinguish,
especially in high-dimensional feature spaces.  However, we know intuitively
from Figure~\ref{fig:sim-data-collection-viz} that \texttt{Option\_2}
is better, as \texttt{Option\_1} samples data that can be extrapolated more
easily from existing data: both the data from $x \in [0.1,0.4]$ and
$x \in [0.6,0.9]$ can be used to infer the values for $x \in [0.4,0.6]$, so
the extra value of sampling these points is smaller compared to that in
\texttt{Option\_2}.  Therefore, under \texttt{Option\_2}, we need to perform
less unreliable extrapolation, which is further reduced under \texttt{Oracle}.

\begin{table}[h!]
\centering
\begin{tabular}{|l|c|c|c|c|}
\hline
\textbf{Option} & $L = 4.28$ & $L = 5.48$ & $L = 6.65$ & $L = 13.11$ \\
\hline
\texttt{Option\_1} & $0.0958$ & $0.1003$ & $0.1046$ & $0.1283$ \\ \hline
\texttt{Option\_2} & $0.0213$ & $0.0242$ & $0.0271$ & $0.0430$ \\ \hline
\texttt{Oracle}    & $0.0117$ & $0.0126$ & $0.0134$ & $0.0179$ \\
\hline
\end{tabular}
\caption{Comparing three sampling options fixing $\epsilon = 0.04$.
The Lipschitz
constants correspond to
$\Tilde{L}_{p}$~\eqref{eqn:contextualize-L-def-uniform}
for $p \in \set{0.8,0.85,0.9,0.95}$.
Given the data is coming from Figure~\ref{fig:simulation-data-collection-dgp}, the choice of $\epsilon = 0.04$ is not important, as  any $\epsilon < 0.1$ yields the same weight $w_\epsilon$, and hence the same result.
We conduct 10 runs of $Z$ 
to estimate~\eqref{eqn:m-partial-sampling-metric} and only report the mean as the variance is nearly zero. We observe consistent rankings across different options for different choices of $L$.}
\label{table:sim-sample-options-before}
\end{table}

To formalize the above intuition, we compute our
metric~\eqref{eqn:m-partial-sampling-metric} for 
each of the three options
and report them
in Table~\ref{table:sim-sample-options-before} over different Lipschitz
constants.  Similar to Figure~\ref{fig:sim_contextualize}, we perform
contextualization and use Lipschitz constants taking values
$\Tilde{L}_{p}$~\eqref{eqn:contextualize-L-def-uniform} for
$p \in \set{0.8,0.85,0.9,0.95}$.  Our approach formally considers the
extrapolability of outcome functions and designs sampling evaluation based on
this.  Therefore, our metric predicts that \texttt{Option\_2} is better.  To
complement the  worst-case analysis in Table~\ref{table:sim-sample-options-before}, we then evaluate these
two options by repeatedly sampling according to these options and report the
error (the distance between the CI and the estimand) of the resulting CI
generated by $\AIPWP$ in Table~\ref{table:sim-sample-options-after}.

\begin{table}[h!]
\centering
\begin{tabular}{|c|c|c|}
\hline
\textbf{Option} & \textbf{Coverage of $\AIPWP$} & \textbf{Distance of $\AIPWP$} \\
\hline
\texttt{Option\_1}  & $0.051 \pm 0.0136$ & $0.0182 \pm 0.0006$ \\
\hline
\texttt{Option\_2}  & $0.901 \pm 0.0185$ & $0.0009 \pm 0.0002$ \\
\hline
\texttt{Oracle} & $0.985 \pm 0.0075$ & $0.00009 \pm 0.00006$ \\
\hline
\end{tabular}
\caption{
Same setting as Table~\ref{table:sim-sample-options-before}.
Comparing two sampling options fixing $\epsilon = 0.04$.
The error of $\AIPWP$ is also computed using 200   runs and  also include standard errors. 
Compared with Table~\ref{table:sim-sample-options-before}, our metric successfully predicts the ranking of different sampling options.
}
\label{table:sim-sample-options-after}
\end{table}

\subsection{Confidence sequence}

Our framework is also applicable when the analyst conducts continual
sampling/data collection with an arbitrary stopping policy, and we defer those
results to Appendix~\ref{sec:experimenta-results-continual-sampling}.  From
Figure~\ref{fig:sim-conf-seq}, we observe how the uncertainty over $\tau_-$
keeps reducing as the level of limited overlap reduces, and 
we offer the analyst
a flexible way to specify the stopping rule for sampling subject to a given
sampling budget.

Theoretically, guiding continual data collection requires generalizing the
coverage guarantee in~\eqref{eqn:coverage-def} to a \emph{uniform coverage
  property}:
\begin{align}
\inf_{f \in \mathcal{F}} \mathbb{P}_f \left( \tau_{\bm{w}_t}(f) \in \mathcal{C}_t, \, \forall t \geq 1 \right) \geq 1 - \alpha.
\label{eqn:coverage-def-uni}
\end{align}
In particular, this setting includes problems where $\tau_{\bm{w}_t}(f)$
varies over time as the non-overlap region evolves with new data. 
Since the stopping time is not fixed in advance, 
a naive union bound based on a
pre-specified horizon $T$ does not provide valid coverage for arbitrary,
data-dependent stopping rules. 
Instead, we require a confidence
\emph{sequence}.
We assume the analyst conducts continual sampling and that,
at every time $t$, new
data will be collected.  
At each time $t$, one solves
the problem~\eqref{eqn:opt-prob-delta} and finds an estimator
$\hl_t = \bm{w}_t^\top \bm{y}_t$ at $t$ to estimate 
$\tau_{\bm{w}_t}(f)$, which is a
linear combination of the data, and $\bm{y}_t$ is the outcome data at time $t$.  
As we will show in Lemma~\ref{lemma:gaussian-conf-seq-cv},
   $\mc{C}_t$, which generalizes~\eqref{eqn:CI-expression},
   is a  valid   confidence sequence:
\begin{align}
\mc{C}_t
= \hl_t \pm \sd(\hl_t)\,
\cv_{\alpha_t}\!\left(
  \frac{\maxbias_t(\hl_t)}{\sd(\hl_t)}
\right),
\qquad
\alpha_t := \frac{6\alpha}{\pi^2 t^2}.
\label{eqn:conf-seq}
\end{align}
In addition,
$\sd(\hl_t)$ can be computed by noting that
$\hl_t = \bm{w}_t^\top \bm{y}_t$ 
and $\maxbias_t(\hl_t)$ is given by~\eqref{eqn:max-bias}.

%%% Local Variables:
%%% mode: latex
%%% TeX-master: "main"
%%% End:

\section{Discussion}
\label{section:discussion}

This paper develops a practical minimax-based approach for reliable treatment effect inference in scenarios with limited covariate overlap between treated and control groups. By extrapolating the outcome function over the covariate space, our framework produces robust confidence intervals for treatment effects in regions without overlap and mitigates the bias introduced by common practices such as data trimming. Additionally, it serves as a diagnostic tool for identifying subpopulations with limited effective sample sizes and offers principled guidance for addressing challenges associated with poor overlap. Experimental results, across both simulated and real-world datasets, demonstrate the framework’s effectiveness in improving the reliability and interpretability of treatment effect inference.

A key feature of our approach is
a \emph{sensitivity framework} that allows analysts to systematically quantify the uncertainty associated with extrapolating from well-supported (overlap) regions to sparse (non-overlap) regions. 
By parameterizing the smoothness of the outcome function class via a Lipschitz constant (see~\eqref{eqn:def-F-L-Lip-finite-sample}), our method enables users to assess how assumptions on function regularity influence the resulting confidence intervals. This sensitivity analysis is operationalized through a data-driven \emph{contextualization procedure}, which calibrates the Lipschitz constant based on observed covariates and outcomes in the overlap region. The result is a transparent, interpretable trade-off between bias and extrapolation strength—allowing analysts to probe, validate, and communicate the robustness of their causal conclusions.

In practice, analysts often have 
more domain-specific knowledge about outcome smoothness or structural relationships in the data. 
Incorporating this knowledge into the function class can sharpen inference and reduce conservativeness in extrapolated intervals.
Moreover, the sensitivity framework highlights how minimal deviations in smoothness assumptions can affect bias due to trimming,
helping practitioners avoid spurious certainty in under-supported regions.

An exciting direction for future work is extending this framework to off-policy evaluation tasks with richer action spaces (e.g., multi-armed treatments). Such settings require imputing multiple counterfactuals per unit, increasing both the complexity of the function class and the difficulty of computing the associated modulus of continuity. Developing scalable and efficient methods to handle this will be crucial for broader adoption of the minimax approach
in high-stakes applications.

% Acknowledgments---Will not appear in anonymized version

%% ========================== Bibliography =========================  = %%

\bibliographystyle{abbrvnat}

\ifdefined\useorstyle
\setlength{\bibsep}{.0em}
\else
\setlength{\bibsep}{.7em}
\fi
\bibliography{bib}

\begin{thebibliography}{49}
\providecommand{\natexlab}[1]{#1}
\providecommand{\url}[1]{\texttt{#1}}
\expandafter\ifx\csname urlstyle\endcsname\relax
  \providecommand{\doi}[1]{doi: #1}\else
  \providecommand{\doi}{doi: \begingroup \urlstyle{rm}\Url}\fi

\bibitem[Abadie and Imbens(2006)]{AbadieIm06}
A.~Abadie and G.~W. Imbens.
\newblock Large sample properties of matching estimators for average treatment
  effects.
\newblock \emph{Econometrica}, 74\penalty0 (1):\penalty0 235--267, 2006.

\bibitem[Abadie and Imbens(2011)]{AbadieIm11}
A.~Abadie and G.~W. Imbens.
\newblock Bias-corrected matching estimators for average treatment effects.
\newblock \emph{Journal of Business \& Economic Statistics}, 29\penalty0
  (1):\penalty0 1--11, 2011.

\bibitem[Armstrong and Koles{\'a}r(2018)]{ArmstrongKo18}
T.~B. Armstrong and M.~Koles{\'a}r.
\newblock Optimal inference in a class of regression models.
\newblock \emph{Econometrica}, 86\penalty0 (2):\penalty0 655--683, 2018.

\bibitem[Armstrong and Koles{\'a}r(2021)]{ArmstrongKo21}
T.~B. Armstrong and M.~Koles{\'a}r.
\newblock Finite-sample optimal estimation and inference on average treatment
  effects under unconfoundedness.
\newblock \emph{Econometrica}, 89\penalty0 (3):\penalty0 1141--1177, 2021.

\bibitem[Beliakov(2006)]{Beliakov06}
G.~Beliakov.
\newblock Interpolation of {L}ipschitz functions.
\newblock \emph{Journal of Computational and Applied Mathematics}, 196\penalty0
  (1):\penalty0 20--44, 2006.

\bibitem[Busso et~al.(2014)Busso, DiNardo, and McCrary]{BussoDiMc14}
M.~Busso, J.~DiNardo, and J.~McCrary.
\newblock New evidence on the finite sample properties of propensity score
  reweighting and matching estimators.
\newblock \emph{The Review of Economics and Statistics}, 96\penalty0
  (5):\penalty0 885--897, 2014.

\bibitem[Cai and Low(2004)]{CaiLo04}
T.~T. Cai and M.~G. Low.
\newblock An adaptation theory for nonparametric confidence intervals.
\newblock \emph{The Annals of statistics}, 32\penalty0 (5):\penalty0
  1805--1840, 2004.

\bibitem[Cole and Hern{\'a}n(2008)]{ColeHe08}
S.~R. Cole and M.~A. Hern{\'a}n.
\newblock {Constructing Inverse Probability Weights for Marginal Structural
  Models}.
\newblock \emph{American Journal of Epidemiology}, 168\penalty0 (6):\penalty0
  656--664, 2008.

\bibitem[Crump et~al.(2006)Crump, Hotz, Imbens, and Mitnik]{CrumpHoImMi06}
R.~K. Crump, V.~J. Hotz, G.~W. Imbens, and O.~A. Mitnik.
\newblock Moving the goalposts: Addressing limited overlap in the estimation of
  average treatment effects by changing the estimand.
\newblock Technical report, National Bureau of Economic Research, 2006.

\bibitem[Crump et~al.(2009)Crump, Hotz, Imbens, and Mitnik]{CrumpRiHoImMi09}
R.~K. Crump, V.~J. Hotz, G.~W. Imbens, and O.~A. Mitnik.
\newblock {Dealing with limited overlap in estimation of average treatment
  effects}.
\newblock \emph{Biometrika}, 96\penalty0 (1):\penalty0 187--199, 2009.

\bibitem[Cui(2021)]{Cui21}
Y.~Cui.
\newblock Individualized {Decision}-{Making} {Under} {Partial}
  {Identification}: Three {Perspectives}, {Two} {Optimality} {Results}, and
  {One} {Paradox}.
\newblock \emph{Harvard Data Science Review}, \penalty0 (3), 2021.

\bibitem[Dehejia and Wahba(1999)]{DehejiaWa99}
R.~H. Dehejia and S.~Wahba.
\newblock Causal effects in nonexperimental studies: Reevaluating the
  evaluation of training programs.
\newblock \emph{Journal of the American Statistical Association}, 94\penalty0
  (448):\penalty0 1053--1062, 1999.

\bibitem[Donoho(1994)]{Donoho94}
D.~L. Donoho.
\newblock Statistical estimation and optimal recovery.
\newblock \emph{Annals of Statistics}, 22\penalty0 (1):\penalty0 238--270,
  1994.

\bibitem[D’Amour et~al.(2021)D’Amour, Ding, Feller, Lei, and
  Sekhon]{DAmourDiFeLeSe21}
A.~D’Amour, P.~Ding, A.~Feller, L.~Lei, and J.~Sekhon.
\newblock Overlap in observational studies with high-dimensional covariates.
\newblock \emph{Journal of Econometrics}, 221\penalty0 (2):\penalty0 644--654,
  2021.

\bibitem[Fr{\"o}lich(2004)]{Frolich04}
M.~Fr{\"o}lich.
\newblock Finite-sample properties of propensity-score matching and weighting
  estimators.
\newblock \emph{Review of Economics and Statistics}, 86\penalty0 (1):\penalty0
  77--90, 2004.

\bibitem[Glynn et~al.(2019)Glynn, Lunt, Rothman, Poole, Schneeweiss, and
  St{\"u}rmer]{GlynnLuRoPoScSt19}
R.~J. Glynn, M.~Lunt, K.~J. Rothman, C.~Poole, S.~Schneeweiss, and
  T.~St{\"u}rmer.
\newblock Comparison of alternative approaches to trim subjects in the tails of
  the propensity score distribution.
\newblock \emph{Pharmacoepidemiology and drug safety}, 2019.

\bibitem[Heckman et~al.(1997)Heckman, Ichimura, and Todd]{HeckmanIcTo97}
J.~J. Heckman, H.~Ichimura, and P.~E. Todd.
\newblock {Matching As An Econometric Evaluation Estimator: Evidence from
  Evaluating a Job Training Programme}.
\newblock \emph{The Review of Economic Studies}, 64\penalty0 (4):\penalty0
  605--654, 10 1997.

\bibitem[Heiler and Kazak(2021)]{HeilerKa21}
P.~Heiler and E.~Kazak.
\newblock Valid inference for treatment effect parameters under irregular
  identification and many extreme propensity scores.
\newblock \emph{Journal of Econometrics}, 222\penalty0 (2), 2021.

\bibitem[Hong et~al.(2019)Hong, Leung, and Li]{HongLeLi19}
H.~Hong, M.~P. Leung, and J.~Li.
\newblock {Inference on finite-population treatment effects under limited
  overlap}.
\newblock \emph{The Econometrics Journal}, 23\penalty0 (1):\penalty0 32--47,
  2019.

\bibitem[Hussain et~al.(2022)Hussain, Oberst, Shih, and
  Sontag]{HussainObShSo22}
Z.~Hussain, M.~Oberst, M.-C. Shih, and D.~Sontag.
\newblock Falsification before extrapolation in causal effect estimation.
\newblock In \emph{Proceedings of the 36th International Conference on Neural
  Information Processing Systems}, 2022.
\newblock ISBN 9781713871088.

\bibitem[Hussain et~al.(2023)Hussain, Shih, Oberst, Demirel, and
  Sontag]{HussainShObDeSo23}
Z.~Hussain, M.-C. Shih, M.~Oberst, I.~Demirel, and D.~Sontag.
\newblock Falsification of internal and external validity in observational
  studies via conditional moment restrictions.
\newblock In \emph{Proceedings of The 26th International Conference on
  Artificial Intelligence and Statistics}, pages 5869--5898, 2023.

\bibitem[Imbens(2004)]{Imbens04}
G.~Imbens.
\newblock Nonparametric estimation of average treatment effects under
  exogeneity: a review.
\newblock \emph{The Review of Economics and Statistics}, 86\penalty0
  (1):\penalty0 4--29, 2004.

\bibitem[Imbens and Manski(2004)]{ImbensMa04}
G.~W. Imbens and C.~F. Manski.
\newblock Confidence intervals for partially identified parameters.
\newblock \emph{Econometrica}, 72\penalty0 (6):\penalty0 1845--1857, 2004.

\bibitem[Ju et~al.(2019)Ju, Schwab, and van~der Laan]{JuScVa19}
C.~Ju, J.~Schwab, and M.~J. van~der Laan.
\newblock On adaptive propensity score truncation in causal inference.
\newblock \emph{Statistical methods in medical research}, 2019.

\bibitem[Juditsky and Nemirovski(2009)]{JuditskyNe09}
A.~B. Juditsky and A.~S. Nemirovski.
\newblock {Nonparametric estimation by convex programming}.
\newblock \emph{The Annals of Statistics}, 37:\penalty0 2278 -- 2300, 2009.

\bibitem[Kallus and Zhou(2021)]{KallusZh21}
N.~Kallus and A.~Zhou.
\newblock Minimax-optimal policy learning under unobserved confounding.
\newblock \emph{Management Science}, 67\penalty0 (5):\penalty0 2870--2890,
  2021.

\bibitem[Kang and Schafer(2007)]{KangSc07}
J.~D.~Y. Kang and J.~L. Schafer.
\newblock {Demystifying Double Robustness: A Comparison of Alternative
  Strategies for Estimating a Population Mean from Incomplete Data}.
\newblock \emph{Statistical Science}, 22\penalty0 (4):\penalty0 523 -- 539,
  2007.

\bibitem[Khan and Nekipelov(2024)]{KhanNe22}
S.~Khan and D.~Nekipelov.
\newblock On uniform inference in nonlinear models with endogeneity.
\newblock \emph{Journal of Econometrics}, 240\penalty0 (2):\penalty0 105261,
  2024.

\bibitem[Khan and Tamer(2010)]{KhanTa10}
S.~Khan and E.~Tamer.
\newblock Irregular identification, support conditions, and inverse weight
  estimation.
\newblock \emph{Econometrica}, 78\penalty0 (6):\penalty0 2021--2042, 2010.

\bibitem[Khan et~al.(2024)Khan, Saveski, and Ugander]{KhanSaUg23}
S.~Khan, M.~Saveski, and J.~Ugander.
\newblock Off-policy evaluation beyond overlap: Sharp partial identification
  under smoothness.
\newblock In \emph{Proceedings of the 41st International Conference on Machine
  Learning}, volume 235, pages 23734--23757, 2024.

\bibitem[Koenker and Bilias(2002)]{BiliasKo02}
R.~Koenker and Y.~Bilias.
\newblock Quantile regression for duration data: A reappraisal of the
  pennsylvania reemployment bonus experiments.
\newblock In \emph{Economic applications of quantile regression}, pages
  199--220. Springer, 2002.

\bibitem[LaLonde(1986)]{LaLonde86}
R.~J. LaLonde.
\newblock Evaluating the econometric evaluations of training programs with
  experimental data.
\newblock \emph{The American Economic Review}, 76\penalty0 (4):\penalty0
  604--620, 1986.

\bibitem[Landgren et~al.(2018)Landgren, Siegel, Auclair, Chari, Boedigheimer,
  Welliver, Mezzi, Iskander, and Jakubowiak]{LandgrenSiAuetal18}
O.~Landgren, D.~S. Siegel, D.~Auclair, A.~Chari, M.~Boedigheimer, T.~Welliver,
  K.~Mezzi, K.~Iskander, and A.~Jakubowiak.
\newblock Carfilzomib-lenalidomide-dexamethasone versus
  bortezomib-lenalidomide-dexamethasone in patients with newly diagnosed
  multiple myeloma: results from the prospective, longitudinal, observational
  commpass study.
\newblock \emph{Blood}, 132:\penalty0 799, 2018.

\bibitem[Lee et~al.(2011)Lee, Lessler, and Stuart]{LeeLeSt11}
B.~K. Lee, J.~Lessler, and E.~A. Stuart.
\newblock Weight trimming and propensity score weighting.
\newblock \emph{PLOS ONE}, 6\penalty0 (3):\penalty0 1--6, 03 2011.

\bibitem[Li et~al.(2018{\natexlab{a}})Li, Ren, Shen, Hou, Su, Di~Bacco, Hong,
  Galaznik, Dash, Crossland, Dolin, and Szalma]{LiReShetal18}
B.~Li, K.~Ren, L.~Shen, P.~Hou, Z.~Su, A.~Di~Bacco, J.-L. Hong, A.~Galaznik,
  A.~B. Dash, V.~Crossland, P.~Dolin, and S.~Szalma.
\newblock Comparing bortezomib-lenalidomide-dexamethasone (vrd) with
  carfilzomib-lenalidomide-dexamethasone (krd) in the patients with newly
  diagnosed multiple myeloma (ndmm) in two observational studies.
\newblock \emph{Blood}, 132:\penalty0 3298, 2018{\natexlab{a}}.

\bibitem[Li et~al.(2018{\natexlab{b}})Li, Thomas, and Li]{LiThLi18}
F.~Li, L.~E. Thomas, and F.~Li.
\newblock {Addressing Extreme Propensity Scores via the Overlap Weights}.
\newblock \emph{American Journal of Epidemiology}, 188\penalty0 (1):\penalty0
  250--257, 2018{\natexlab{b}}.

\bibitem[Ma and Wang(2020)]{MaWa20}
X.~Ma and J.~Wang.
\newblock Robust inference using inverse probability weighting.
\newblock \emph{Journal of the American Statistical Association}, 115\penalty0
  (532):\penalty0 1851--1860, 2020.

\bibitem[Ma et~al.(2023)Ma, Sant'Anna, Sasaki, and Ura]{MaSaSaUr23}
Y.~Ma, P.~H. Sant'Anna, Y.~Sasaki, and T.~Ura.
\newblock Doubly robust estimators with weak overlap.
\newblock \emph{arXiv:2304.08974 [stat.ME]}, 2023.

\bibitem[Manski(1990)]{Manski90}
C.~F. Manski.
\newblock Nonparametric bounds on treatment effects.
\newblock \emph{The American Economic Review}, 80\penalty0 (2):\penalty0
  319--323, 1990.

\bibitem[NIH(2016)]{NIH16}
NIH.
\newblock Relating clinical outcomes in multiple myeloma to personal assessment
  of genetic profile(com-mpass), 2016.
\newblock URL \url{https://clinicaltrials.gov/study/NCT01454297}.

\bibitem[Petersen et~al.(2012)Petersen, Porter, Gruber, Wang, and van~der
  Laan]{PetersenPoGrWaVa12}
M.~L. Petersen, K.~E. Porter, S.~Gruber, Y.~Wang, and M.~J. van~der Laan.
\newblock Diagnosing and responding to violations in the positivity assumption.
\newblock \emph{Statistical Methods in Medical Research}, 21\penalty0
  (1):\penalty0 31--54, 2012.

\bibitem[Rosenbaum and Rubin(1983)]{RosenbaumRu83b}
P.~R. Rosenbaum and D.~B. Rubin.
\newblock The central role of the propensity score in observational studies for
  causal effects.
\newblock \emph{Biometrika}, 70\penalty0 (1):\penalty0 41--55, 1983.

\bibitem[Rothe(2017)]{Rothe17}
C.~Rothe.
\newblock Robust confidence intervals for average treatment effects under
  limited overlap.
\newblock \emph{Econometrica}, 85\penalty0 (2):\penalty0 645--660, 2017.

\bibitem[Sasaki and Ura(2022)]{SasakiUr22}
Y.~Sasaki and T.~Ura.
\newblock Estimation and inference for moments of ratios with robustness
  against large trimming bias.
\newblock \emph{Econometric Theory}, 38\penalty0 (1):\penalty0 66–112, 2022.

\bibitem[Schneider and Zenios(1990)]{Schneiderze90}
M.~H. Schneider and S.~A. Zenios.
\newblock A comparative study of algorithms for matrix balancing.
\newblock \emph{Operations Research}, 38\penalty0 (3), 1990.

\bibitem[Smith and Todd(2005)]{SmithTo05}
J.~A. Smith and P.~E. Todd.
\newblock Does matching overcome {LaLonde's} critique of nonexperimental
  estimators?
\newblock \emph{Journal of Econometrics}, 125:\penalty0 305--353, 2005.

\bibitem[Stoye(2009)]{Stoye09}
J.~Stoye.
\newblock More on confidence intervals for partially identified parameters.
\newblock \emph{Econometrica}, 77\penalty0 (4):\penalty0 1299--1315, 2009.

\bibitem[St{\"u}rmer et~al.(2010)St{\"u}rmer, Rothman, Avorn, and
  Glynn]{Sturmer10}
T.~St{\"u}rmer, K.~J. Rothman, J.~Avorn, and R.~J. Glynn.
\newblock Treatment effects in the presence of unmeasured confounding: dealing
  with observations in the tails of the propensity score distribution—a
  simulation study.
\newblock \emph{American Journal of Epidemiology}, 172\penalty0 (7):\penalty0
  843--854, 2010.

\bibitem[Yang and Ding(2018)]{YangDi18}
S.~Yang and P.~Ding.
\newblock {Asymptotic inference of causal effects with observational studies
  trimmed by the estimated propensity scores}.
\newblock \emph{Biometrika}, 105\penalty0 (2):\penalty0 487--493, 2018.

\end{thebibliography}

\ifdefined\useorstyle
 
\ECSwitch 
\ECHead{Appendix}

\else
\newpage
\appendix
\fi

\section{Proof of Lemma~\ref{lemma:minimax-as-matching}}

\label{sec:additional-details-minimax}

In this section, we closely follow the notation
in~\citep{ArmstrongKo21}.  We let $n_1 = |\set{i:z_i = 1}|$ and
$n_0 = |\set{i:z_i = 0}|$ denote the total number of samples with $z=1$ and
$z=0$.  With slight abuse of notation, we let $i \in [n_1]$ denote the $i$-th
sample for those with $z_i=1$ and similarly write $j \in [n_0]$ to denote the
$j$-th sample for those with $z_j=0$.

We first define $m_i=(2z_i-1) f(x_i, z_i)$, and
$r_i=(1-2z_i)f(x_i,1-z_i)$. Then~\citet{ArmstrongKo21} has shown that the
modulus of continuity~\eqref{eqn:def-omega} can be written as
\begin{equation}
\label{eq:dual-problem}
\min_{f\in \mc{F}_L}\frac{1}{2}\sum_{i=1}^{n}\frac{f^2(x_i, z_i)}{\sigma^2(x_i, z_i)}
- \mu \tau_{\bm{w}}(f)/L
= \min_{f\in\mc{F}_1}
\frac{L^2}{2}
\sum_{i=1}^{n}\frac{f^2(x_i, z_i)}{\sigma^2(x_i, z_i)}
- \mu \tau_{\bm{w}}(f),
\end{equation}
which can be further written  as
\begin{align}
 & \frac{1}{2}\sum_{i=1}^{n}\frac{m_i^2}{\sigma^2(z_i)}
  -\mu\left(\sum_{i=1}^{n}w_i(m_i+r_i)
  \right) \nonumber \\
 & +\sum_{i, j\colon z_i=1-z_j=1}\left[\Lambda^{0}_{ij}(r_i-m_j
    - \norm{x_i-x_j}) +\Lambda^{1}_{ij}(r_j- m_i-
    \norm{x_i-x_j})\right].\label{eq:lagrangian-mu}
\end{align} where $\Lambda^{0}_{ij}, \Lambda^{1}_{ij} \ge 0$ 
are the optimal dual variables 
for the
Lipschitz constraints.

Now we are ready 
to prove a stronger version of Lemma~\ref{lemma:minimax-as-matching}.
\begin{lemma}
  Let $\mu \ge 0$ and $\Lambda \ge 0$ be optimal dual variables corresponding
  to the Lipschitz constraints~\eqref{eqn:def-F-L-Lip-finite-sample}. We can
  rewrite $\LFLCI(\bm{w})$~\eqref{eqn:minimax-w-interlval-expression} as
  \begin{align}
    \LFLCI(\bm{w})  & =   \sum_{k=1}^n \Paran{
                      \I{z_k= 0} \paran{\sum_{i: z_i = 1}  \frac{\Lambda_{ik}^1}{\mu} y_i  - w_k y_k}
                      +\I{z_k=1} \paran{w_k y_k - \sum_{j: z_j = 0}  \frac{\Lambda_{kj}^0}{\mu} y_j }
                      } \I{w_k > 0},
  \end{align} 
  Thus, the minimax estimator will impute the counterfactual values as
\begin{align}
  w_k \hat{f}(x_k,1 -z_k) = \sum_{j: z_j = 1- z_k} W_{jk} y_j
  ~~~~\mbox{where}~~~~W_{jk} = 
\begin{cases}
\frac{\Lambda_{jk}^1}{\mu}, & \text{if } z_k = 0 \\ 
 \frac{\Lambda_{kj}^0}{\mu}, & \text{if } z_k = 1.
\end{cases}
\label{eqn:impute-conterfactual-more-specific}
\end{align} 
If $z_k = 0$, then $\sum_{i}  \frac{\Lambda_{ik}^1}{\mu} = w_k$; if $z_k = 1$, then $\sum_{j}  \frac{\Lambda_{kj}^0}{\mu} = w_k$.
\label{lemma:minimax-as-matching-stronger}
\end{lemma}

\begin{proof}[Proof of Lemma~\ref{lemma:minimax-as-matching-stronger}]
We let $f_i = f\opt(x_i,z_i)$ denote the optimal solution 
of~\eqref{eqn:def-omega} and 
$m_i = (2z_i-1) f_i, r_i =(1-2z_i) f_i$.
We also let $\mu, \Lambda^{0}_{ij},\Lambda^{1}_{ij}$ denote the optimal dual variable.
Note that at the optimal solution, the conic constraint must be binding so that 
  $\mu > 0$.
Fix $k$ and
suppose $z_k=0$, from~\eqref{eqn:minimax-w-interlval-expression},
we need to show that 
\begin{align}
 - w_k - \sum_{i \in [n_1]} \frac{\Lambda_{ik}^0}{\mu}
 =  \paran{\sum_{i=1}^n w_i} \frac{f_k}{\sum_{i \in [n_1]} f_i}.
 \label{eqn:matching-lemma-need-to-show}
\end{align} The proof for the case with $z_k=1$ is similar thus omitted.
 From (S1) of Section D.2 
 in~\citep{ArmstrongKo21}, we have
\begin{align*}
\frac{m_j}{\sigma^2} = \mu w_j + \sum_{i \in [n_1]} \Lambda_{ij}^0, \forall j \in [n_0]
%\label{eqn:S1}
\end{align*}
so that
\begin{align}
\frac{m_j}{\sigma^2 \mu} = w_j + \sum_{i \in [n_1]} \frac{\Lambda_{ij}^0}{\mu} , \forall j \in [n_0].
\label{eqn:matching-lemma-part-1}
\end{align} 
From~\citep{ArmstrongKo21}, we know
\begin{align}
\frac{m_i}{\sigma^2} = \mu w_i +  \sum_{j \in [n_0]} \Lambda_{ij}^1,  \forall i \in [n_1].
\label{eqn:S2}
\end{align} 
Summing over~\eqref{eqn:S2},
we have 
\begin{align*}
\sum_{i \in [n_1]} \frac{m_i}{\sigma^2} =  \sum_{i \in [n_1]}  w_i \mu + \sum_{j \
\in [n_0]} \sum_{i \in [n_1]} \Lambda_{ij}^1.
\end{align*}
Since $\sum_{i \in [n_1]} \Lambda_{ij}^1 = \mu w_j$ from~\citep{ArmstrongKo21},
we have 
\begin{align}
    \sum_{i \in [n_1]} \frac{m_j}{\sigma^2}
    =  \sum_{i \in [n_1]}  w_i \mu  + \sum_{i \in [n_1]} \sum_{j \in [n_0]} \Lambda_{ij}^1
    =  \sum_{i \in [n_1]}  w_i \mu  +  \sum_{j \in [n_0]}  w_j \mu  =   
    \paran{\sum_{i=1}^n w_i} \mu.  
    \label{eqn:matching-lemma-part-2}
\end{align}    
Therefore, we arrive at
\begin{align*}
- w_k - \sum_{i \in [n_1]} \frac{\Lambda_{ik}^0}{\mu}
=-\frac{m_k}{\sigma^2 \mu}
= \frac{f_k}{\sigma^2 \mu}
 =  \paran{\sum_{i=1}^n w_i} \frac{f_k}{\sum_{i \in [n_1]} f_i},
\end{align*} where the first 
equality 
follows from~\eqref{eqn:matching-lemma-part-1},
the second one follows from $m_k = -f_k$ as $z_k=0$,
and the last one follows from~\eqref{eqn:matching-lemma-part-1}.
Thus, we show~\eqref{eqn:matching-lemma-need-to-show} and finish the proof.
\end{proof}

\section{Computation of 
$\omega_{\mc{F};\mc{D}}(\delta)$}

\label{sec:algo}

We follow~\citep{ArmstrongKo21}'s
procedure to compute
$\omega_{\mc{F};\mc{D}}(\delta)$.
Most of the components there are directly applicable.
However,
one key component in~\citep{ArmstrongKo21}'s procedure 
is to solve a problem that 
involves finding a    matrix with given row and column sum constraints and certain entries of the matrix must be zero, i.e., Step 2(b) in 
Section A.3 of~\citep{ArmstrongKo21}.
The implementation approach in~\citep{ArmstrongKo21}, which uses matrix inversion,
may be
slow to ensure dual feasibility
$\Lambda^d \ge 0$ 
and thus may need many iterations 
to arrive at a   solution.  
Instead, we solve the problem by converting it into a
max flow algorithm (Algorithm~\ref{algo:max-flow})
instance following~\cite{Schneiderze90}, 
which can be   more efficient than using matrix inversion.
 
\begin{algorithm}
\caption{\texttt{MAX\_FLOW}($p, q, E$)}
\label{algo:max-flow}
\begin{algorithmic}[1]
\State \textbf{Input:} Row sums $p = [p_1, p_2, \ldots, p_m]$, Column sums $q = [q_1, q_2, \ldots, q_n]$, Allowed indices $E$
\State \textbf{Output:} A matrix $A \ge 0$ that satisfies $A_{ij} = 0$ if $(i,j) \not\in E$ and $\sum_j A_{ij} = p_i$ for all $i$ and $\sum_i A_{ij} = q_j$ for all $j$.
\State Create a directed graph $G_f$ with a source node $s$ and a sink node $t$
\For{each row $i$}
    \State Create a node $n_i$
    \State Add an edge from $s$ to $n_i$ with capacity $p_i$
\EndFor
\For{each column $j$}
    \State Create a node $n_j$
    \State Add an edge from $n_j$ to $t$ with capacity $q_j$
\EndFor
\For{each index $(i, j) \in E$}
    \State Add an edge from $n_i$ to $n_j$ with capacity $\infty$
\EndFor
\State Compute the maximum flow from $s$ to $t$ in the graph $G_f$ with costs $c= 1$
\State \Return Matrix $A$ with entries $A_{ij} = f_{n_i,n_j}$, where $f_{n_i,n_j}$ is the flow from node $n_i$ to $n_j$.
\end{algorithmic}
\end{algorithm}

%%% Local Variables:
%%% mode: latex
%%% Tex-master: "main"
%%% End:

\section{Details of   analytical insights}
\label{sec:analytic-example-details}

In this section, we provide details and proofs of results in
Section~\ref{section:analytic}.
We start by proving Lemma~\ref{lemma-extrapolate-start-from-near}.

% \begin{proof}[Proof of Lemma~\ref{lemma-extrapolate-start-from-near}]
\subsection{Proof of Lemma~\ref{lemma-extrapolate-start-from-near}}

Fix $L$ and $\delta$. Let $f\opt$ be an optimal solution of the program~\eqref{eqn:toy-program}. We write
\begin{align*}
f_{i,1} \defeq f\opt(x_i,1), \qquad f_{i,0} \defeq f\opt(x_i,0),
\end{align*}
and define
\begin{align*}
d_{i,j} \defeq L |x_i - x_j| .
\end{align*}
Let
\begin{align*}
W_+ \defeq \{ i : w_i > 0 \}, \qquad W_- \defeq \{ i : w_i = 0 \}
\end{align*}
denote the sets of points with and without weights, respectively.

\medskip
\noindent\textbf{Step 1: Nonnegativity of $f_{i,1}$.}
We first show that for every $i$ with $z_i = 1$,
we can assume that
\begin{align}
f_{i,1} \ge 0 .
\label{eq:f-i-positive-if-z=1}
\end{align}
Indeed, suppose there exists an index $i$ with $z_i = 1$ and $f_{i,1} < 0$. Define a new solution $\hat f$ by
\begin{align*}
\hat f_{i,1} \defeq \max\{f_{i,1},0\} \quad \text{for all } i \text{ with } z_i=1,
\end{align*}
and keep all other coordinates (including all $f_{i,0}$) unchanged. 
By construction, and since the map $x \mapsto \max\{x,0\}$ is 1-Lipschitz,
\begin{align*}
|\hat f_{i,1} - \hat f_{j,1}|
= \abs{ \max\{f_{i,1},0\} -  \max\{f_{j,1},0\}}
\le |f_{i,1} - f_{j,1}|
\le d_{i,j}
\quad\text{for all }i,j.
\end{align*} 
Thus, the Lipschitz constraints remain feasible. Moreover, because the objective of~\eqref{eqn:toy-program} only depends on $f_{i,1}$ at weighted points $i \in W_+$ and these have $w_i>0$ with $z_i=1$, increasing a negative $f_{i,1}$ to $0$ weakly \emph{improves} the objective (while keeping feasibility). 
 So we may assume~\eqref{eq:f-i-positive-if-z=1} without loss of generality.

\medskip
\noindent\textbf{Step 2: Ordering the zero--weight points by distance.}
For $i \in W_-$, define its distance to the positive--weight region by
\begin{align*}
\overrange_i \defeq \min\{\,|x_i - x_j| : j \in W_+\,\}.
\end{align*}
Order the indices in $W_-$ so that
\begin{align*}
\overrange_1 \le \overrange_2 \le \cdots \le \overrange_m ,
\end{align*}
where $\{1,\dots,m\}=W_-$ after relabeling. Our goal is to construct another optimal solution $g$ 
such that
\begin{align}
i<j \;\Longrightarrow\; \overrange_i \le \overrange_j \text{ and } g_{i,1} \ge g_{j,1},
\label{eq:monotone-goal}
\end{align}
i.e., among points in $W_-$, the map $i \mapsto g_{i,1}$ is weakly \emph{decreasing} as the distance~$\overrange_i$ grows.

\medskip
\noindent\textbf{Step 3: Trouble points and one-step modification.}
Starting from the fixed optimal solution $f\opt$, we work only with the coordinates $f_{i,1}$ for $i\in W_-$; all other coordinates ($i\in W_+$ and all $f_{i,0}$) will never be changed.

Call an index $k \in \{1,\dots,m-1\}$ a \emph{trouble point} if
\begin{align}
\overrange_k \le \overrange_{k+1} \quad\text{but}\quad f_k < f_{k+1} .
\label{eq:trouble-def}
\end{align}
If there is no trouble point, then~\eqref{eq:monotone-goal} already holds with $g_{i,1} = f_{i,1}$ and we are done. Otherwise, let $i^*$ be the smallest index that satisfies~\eqref{eq:trouble-def}.

Given such an $i^*$, consider all zero--weight points that are at least as far as $i^*$ and whose value is \emph{too large}, namely
\begin{align}
\mathcal I \defeq \bigl\{\, j \in W_- : \overrange_j \ge \overrange_{i^*},\ f_{j,1} > f_{i^*,1} \,\bigr\}.
\label{eq:I-def}
\end{align}
Define a modified solution $g$ by
\begin{align*}
g_j &\defeq
\begin{cases}
f_{i^*,1}, & j \in \mathcal I,\\ 
f_{j,1}, & j \notin \mathcal I,
\end{cases}
\end{align*}
and keep $g_{i,0} = f_{i,0}$ for all $i$. In words, we leave every weighted point $i\in W_+$ unchanged, and among $W_-$ we replace the values at the indices in $\mathcal I$ by $f_{i^*,1}$ (which is smaller), while leaving all other coordinates unchanged.

\medskip
\noindent\textbf{Step 4: Feasibility and optimality of the modified solution.}
We check that $g$ is feasible and optimal for~\eqref{eqn:toy-program}.

\smallskip
\emph{(i) Objective value.}
By definition of $W_+$ and $W_-$, the objective depends only on $f_{i,1}$ with $i \in W_+$. Since we never changed those coordinates, $g$ has exactly the same objective value as $f\opt$ and is therefore also optimal.

\smallskip
\emph{(ii) Lipschitz constraints.}
For any pair $(p,q)$ we must show that
\begin{align*}
|g_{p,1} - g_{q,1}| \le d_{p,q}.
\end{align*}
The original $f$ satisfied $|f_{p,1} - f_{q,1}|\le d_{p,q}$ for all $p,q$. We compare old and new values case by case:

\begin{itemize}
\item If $p,q \notin \mathcal I$, then $g_{p,1} = f_{p,1}$ and $g_{q,1} = f_{q,1}$, so the constraint is unchanged.
\item If $p,q \in \mathcal I$, then $g_{p,1} = g_{q,1} = f_{i^*,1}$ and 
therefore
\begin{align*}
|g_{p,1} - g_{q,1}| = 0 \le d_{p,q}.
\end{align*}
\item If $p \in \mathcal I$ and $q \notin \mathcal I$  , then
\begin{align*}
|g_{p,1} - g_{q,1}|
= \bigl|f_{i^*,1} - f_{q,1}\bigr|.
\end{align*}
By construction of $\mc{I}$~\eqref{eq:I-def},
$f_{i^*,1} \le f_{p,1}$ for every $p\in\mathcal I$.
There are two cases:
\begin{enumerate}
\item If $q \in W_+$, 
then we have
$d_{i^*,q}
\le  d_{p,q}$ so that
\begin{align*}
|g_{p,1} - g_{q,1}|
= \bigl|f_{i^*,1} - f_{q,1}\bigr|
\le  d_{i^*,q}
\le  d_{p,q}.
\end{align*}
\item If  $q \in W_-$,  $\eta_q \le \eta_{i^*}$, i.e., $i^*$ is in the middle of $p$ and $q$,
then we also have 
$d_{i^*,q}
\le  d_{p,q}$ so that
\begin{align*}
|g_{p,1} - g_{q,1}|
= \bigl|f_{i^*,1} - f_{q,1}\bigr|
\le  d_{i^*,q}
\le  d_{p,q}.
\end{align*}
\item  If  $q \in W_-$,  $\eta_q > \eta_{i^*}$,  
in this case,  $q \not\in \mc{I}$
implies
$f_{q,1} \le f_{i^*,1} \le f_{p,1}$
so 
\begin{align*}
    |g_{p,1} - g_{q,1}|
= \bigl|f_{i^*,1} - f_{q,1}\bigr|
\le \bigl|f_{p,1} - f_{q,1}\bigr|
 \le d_{p,q}.
\end{align*}
\end{enumerate}

\item If $p \not\in \mathcal I$ and $q \in \mathcal I$, 
this is similar to the previous case.
\end{itemize}
Therefore $g$ is feasible and optimal.
In summary, we have constructed a solution $g$ with
\begin{align*}
i<j \le i^* + 1
\;\Longrightarrow\; \overrange_i \le \overrange_j \text{ and } g_{i,1} \ge g_{j,1},
\end{align*}
Proceeding in the same way as above, we can remove all trouble points
and
we complete proving~\eqref{eq:monotone-goal}.

\subsection{The example
visualized 
in Section~\ref{section:analytic}}
  
In Lemma~\ref{lemma:toy-program-non-overlap-k}, we provide the details of
 the example
visualized 
in Section~\ref{section:analytic}, which is summarized in
Figure~\ref{fig:example-toy-analytic}
and
Figure~\ref{fig:example-toy-vs-eta}.
Before discussing it, we present
the following technical lemma which will be useful.

\begin{lemma}
Let $\bm{K}\in\mathbb{R}^{n\times m}$ have full column rank and suppose
the feasible set
\[
S := \left\{\bm{w}\in\mathbb{R}^n : \|\bm{w}\| \le R,\ \bm{K}^\top \bm{w} = \bm{d}\right\}
\]
is nonempty. Define
\[
\bm{A} := \bm{K}^\top \bm{K}, \qquad 
\bm{b} := \bm{K}^\top \bm{u}, \qquad
\gamma := \|\bm{u}\|^2,
\]
and
\[
\bm{w}_0 := \bm{K}\bm{A}^{-1}\bm{d}, \qquad
\rho^2 := \|\bm{w}_0\|^2 = \bm{d}^\top \bm{A}^{-1}\bm{d}.
\]
Let 
\[
\bm{P} := \bm{I} - \bm{K}\bm{A}^{-1}\bm{K}^\top
\]
denote the orthogonal projector onto $\ker(\bm{K}^\top)$. Then $\rho \le R$ and
\[
\max_{\bm{w}\in S} \set{\bm{u}^\top \bm{w}}
= \bm{u}^\top \bm{w}_0 + \|\bm{P}\bm{u}\| \sqrt{R^2 - \rho^2}
= \bm{b}^\top \bm{A}^{-1}\bm{d}
  + \sqrt{\gamma - \bm{b}^\top \bm{A}^{-1}\bm{b}}\,
    \sqrt{R^2 - \bm{d}^\top \bm{A}^{-1}\bm{d}}.
\]
\label{lemma:ball-linear}
\end{lemma}

\begin{proof}[Proof of Lemma~\ref{lemma:ball-linear}]
Define $\bm{w}_0$ as the minimum--norm solution of the affine constraint:
\[
\bm{w}_0 
:= \arg\min\{\|\bm{w}\| : \bm{K}^\top \bm{w} = \bm{d}\}.
\]
Solving the normal equations gives $\bm{w}_0 = \bm{K}\bm{A}^{-1}\bm{d}$ with
$\bm{A} = \bm{K}^\top \bm{K}$.  
 Let
\[
S_{\mathrm{aff}} := \{\bm{w} : \bm{K}^\top \bm{w} = \bm{d}\}.
\] 
Thus, every feasible
$\bm{w} \in S_{\mathrm{aff}} $ can be written uniquely as
\[
\bm{w} = \bm{w}_0 + \bm{z}, \qquad \bm{z} \in \ker(\bm{K}^\top).
\] 
Using this decomposition, the original problem is equivalent to
\[
\max_{\bm{z}\in\ker(\bm{K}^\top)}
\left\{\bm{u}^\top(\bm{w}_0+\bm{z}) : \|\bm{w}_0+\bm{z}\| \le R\right\}.
\]
Since
$\bm{w}_0 \perp \ker(\bm{K}^\top)$,
for every $\bm{z}\in\ker(\bm{K}^\top)$,
\[
\|\bm{w}_0 + \bm{z}\|^2
= \|\bm{w}_0\|^2 + \|\bm{z}\|^2
= \rho^2 + \|\bm{z}\|^2.
\]
The constraint $\|\bm{w}_0+\bm{z}\| \le R$ becomes
\[
\|\bm{z}\|^2 \le R^2 - \rho^2,
\]
so feasibility implies $\rho \le R$.
Moreover, for $\bm{z}\in\ker(\bm{K}^\top)$,  
\[
\bm{u}^\top \bm{z} = (\bm{P}\bm{u})^\top \bm{z},
\]
since $\bm{P}$ is the orthogonal projector onto $\ker(\bm{K}^\top)$.
Thus the problem reduces to
\[
\max_{\bm{z}\in \ker(\bm{K}^\top)}
\left\{\bm{u}^\top \bm{w}_0 +  \bm{u}^\top \bm{z},
  \|\bm{z}\| \le \sqrt{R^2 - \rho^2}\right\} 
  = 
\max_{\bm{z}\in \R^n }
\left\{\bm{u}^\top \bm{w}_0 + (\bm{P}\bm{u})^\top \bm{z},
  \|\bm{z}\| \le \sqrt{R^2 - \rho^2}\right\}.
\]
Since
\[
\max_{\|\bm{z}\| \le \sqrt{R^2 - \rho^2}}
(\bm{P}\bm{u})^\top \bm{z} = \|\bm{P}\bm{u}\| \sqrt{R^2 - \rho^2},
\]
\[
\max_{\bm{w}\in S} \bm{u}^\top \bm{w}
= \bm{u}^\top \bm{w}_0 + \|\bm{P}\bm{u}\| \sqrt{R^2 - \rho^2}.
\]

Finally, using $\bm{w}_0 = \bm{K}\bm{A}^{-1}\bm{d}$ and
$\bm{P} = \bm{I} - \bm{K}\bm{A}^{-1}\bm{K}^\top$,
we obtain
\[
\bm{u}^\top \bm{w}_0 = \bm{b}^\top \bm{A}^{-1}\bm{d},
\quad
\|\bm{P}\bm{u}\|^2
= \|\bm{u}\|^2 - \bm{b}^\top \bm{A}^{-1}\bm{b}
= \gamma - \bm{b}^\top \bm{A}^{-1}\bm{b},
\]
and $\rho^2 = \bm{d}^\top \bm{A}^{-1}\bm{d}$, which gives the explicit
closed form in the statement.
\end{proof}

\begin{lemma}
Assume the variance $\sigma^2 = 1$.
Let $\overrange, \xi > 0$, and $k>1$.
Suppose there are
$n$ control samples with $Z=0$ at $x=-\xi$
and $n$ treatment samples with $Z=1$ at $x=\xi$; these form the
overlap region.
In addition, suppose there are
$kn$ control samples with $Z=0$ at $x=-\xi-\overrange$
and
$kn$ treatment samples with $Z=1$ at $x=\xi+\overrange$.
These latter points are in the non–overlap region, and their distance
$\overrange$ to the overlap region measures the level of overlap.
Thus,  the modulus of
continuity is
\begin{align*}
    \omega(\delta)
:= \frac{2k}{k+1} 
\max_{f \in \mc{F}_L, \sum_i f_{i,z_i}^2 \le \frac{\delta^2}{4}} 
\bigl[ (f_{x_1,1}-f_{x_1,0}) + (f_{x_4,1}-f_{x_4,0}) \bigr].
\end{align*}
Define the constants
\begin{align*}
\gamma
&:= \frac{2}{n}\Bigl(1+\frac{1}{k}\Bigr), \\
C
&:= \frac{2kn}{k+1} L^2 \overrange^2, \\
\delta_c
&:= \frac{2\sqrt{2nk(k+1)}}{k-1}\,L\overrange, \\
s(\delta)
&:= \sqrt{ \frac{\delta^2}{4} - C }.
\end{align*}

Then the worst–case bias and standard deviation of $\hat{\tau}_\delta$
are given by:
\begin{align}
\begin{split}
\maxbias(\hat{\tau}_\delta)
&=
\begin{cases}
\displaystyle
\frac{2L(2\xi + \overrange) k}{k+1},
&
\text{if } \delta \le \delta_c
\;\;\bigl(\text{i.e.\ } 
L\overrange \ge \frac{(k-1)\,\delta}{2\sqrt{2nk(k+1)}}\bigr), \\ 
\displaystyle
k
\left[
\frac{2L(2\xi + \overrange)}{(k+1)}
+
\frac{2(k-1)}{(k+1)^2}\,L\overrange 
-
\frac{C}{(k+1)}\,
\frac{\sqrt{\frac{8}{n(k+1)}}}{s(\delta)}
\right],
&
\text{if } \delta > \delta_c
\;\;\bigl(\text{i.e.\ } 
L\overrange < \frac{(k-1)\,\delta}{2\sqrt{2nk(k+1)}}\bigr),
\end{cases} \\ 
\sd(\hat{\tau}_\delta)
&=
\begin{cases}
\displaystyle 
\frac{k}{k+1}\sqrt{\frac{2}{n}\Bigl(1+\frac{1}{k}\Bigr)},
&
\text{if } \delta \le \delta_c, \\ 
\displaystyle 
\frac{\sqrt{2} k}{(k+1) \sqrt{n(k+1)}}\,
\frac{\delta}{s(\delta)},
&
\text{if } \delta > \delta_c.
\end{cases}
\end{split}
\label{eqn:worst-case-bias-var-kn}
\end{align}
\label{lemma:toy-program-non-overlap-k}
\end{lemma}

\begin{proof}[Proof of Lemma~\ref{lemma:toy-program-non-overlap-k}]
Let
\begin{align*}
x_1 &=   - \xi - \eta, &
x_2 &=   - \xi, &
x_3 &=   \xi, &
x_4 &=   \xi + \eta .
\end{align*}
Write
\begin{align*}
a_i := f(x_i,0), \qquad
b_i := f(x_i,1), \qquad i = 1,\dots,4 .
\end{align*}
The distances between the points satisfy
\begin{align*}
|x_2 - x_1| = |x_4 - x_3| = \eta,
\qquad
|x_3 - x_2| = 2\xi,
\qquad
|x_4 - x_2| = 2\xi + \eta .
\end{align*}

We now consider the non-uniform quadratic weights
\begin{align*}
(n_1,n_2,n_3,n_4) = (kn, n, n, kn),
\end{align*}
so that the quadratic constraint is
\begin{align}
kn\,a_1^2 + n\,a_2^2 + n\,b_3^2 + kn\,b_4^2 \le \frac{\delta^2}{4}.
\label{eqn:quad-knnnk}
\end{align}
The modulus of continuity for this configuration is
\begin{align*}
\omega(\delta)
&= \max_f \frac{2k}{k+1} \big[(f_{x_1,1} - f_{x_1,0}) + (f_{x_4,1} - f_{x_4,0})\big] \\
&= \max_{(a_i,b_i)} \frac{2k}{k+1} \big[(b_1 - a_1) + (b_4 - a_4)\big]
\end{align*}
subject to the quadratic constraint~\eqref{eqn:quad-knnnk} and the Lipschitz constraints
\begin{align*}
|a_i - a_j| &\le L |x_i - x_j|, &
|b_i - b_j| &\le L |x_i - x_j|,
\qquad i \neq j .
\end{align*}

\paragraph{Step 1: Reduction to a problem in $\bm{v}$.}
Introduce
\begin{align*}
\bm{v} := (a_1,a_2,b_3,b_4)^\top,
\qquad
\bm{w} := (a_3,a_4,b_1,b_2)^\top.
\end{align*}
The quadratic constraint~\eqref{eqn:quad-knnnk} involves only $\bm{v}$ and can
be written as
\begin{align*}
\bm{v}^\top \bm{Q} \bm{v} \le \frac{\delta^2}{4},
\qquad
\bm{Q} := \diag(kn,n,n,kn).
\end{align*}

For each fixed $\bm{v}$, define
\begin{align}
F(\bm{v})
:= \max \Bigl\{ (b_1 - a_1) + (b_4 - a_4)
  : (a_i,b_i) \text{ satisfy all Lipschitz constraints} \Bigr\} .
  \label{eqn:F-v-def-knnnk}
\end{align}
Let
\begin{align}
\mathcal{V}
:= \Bigl\{ \bm{v} : |a_1 - a_2| \le L\eta,\ |b_3 - b_4| \le L\eta \Bigr\}. \label{eqn:mc-V-def-knnnk}
\end{align}
These are exactly the Lipschitz conditions that involve only the coordinates
in $\bm{v}$. Then the original problem can be written as
\begin{align}
\omega(\delta)
= \frac{2k}{k+1}
\max_{\bm{v} \in \mathcal{V},\ \bm{v}^\top \bm{Q} \bm{v} \le \delta^2/4}
F(\bm{v}). \label{eqn:omega-in-F-knnnk}
\end{align}

\paragraph{Step 2: Remove redundant Lipschitz constraints.}

Assume $|a_1 - a_2| \le L\eta$. Suppose we choose $a_3,a_4$ such that
\begin{align*}
|a_3 - a_2| &\le L (2\xi), \\
|a_4 - a_3| &\le L \eta.
\end{align*}
Then all Lipschitz constraints for the $a_i$ follow by the triangle inequality. 
Thus, \emph{conditional on $|a_1-a_2|\le L\eta$}, it suffices to enforce
$|a_2 - a_3|  \le L (2\xi)$ and
$|a_3 - a_4| \le L \eta$.
 
Similarly, \emph{conditional on $|b_3-b_4|\le L\eta$}, it suffices to check
the Lipschitz bounds $|b_3 - b_2|  \le L (2\xi)$ and
$|b_2 - b_1| \le L \eta$.

\paragraph{Step 3: Explicit computation of $F(\bm{v})$~\eqref{eqn:F-v-def-knnnk}.}

Fix $\bm{v} = (a_1,a_2,b_3,b_4)^\top \in \mathcal{V}$~\eqref{eqn:mc-V-def-knnnk}.

\medskip 
We can verify that
\begin{align*}
a_4^\star(\bm{v}) 
&= \min\Bigl\{a_4: |a_2 - a_3|  \le L (2\xi),\ |a_3 - a_4| \le L \eta\Bigr\}
 = a_2 - L(2\xi + \eta),
\end{align*}
and similarly,
\begin{align*}
b_1^\star(\bm{v})
&= \max\Bigl\{b_1: |b_3 - b_2|  \le L (2\xi),\ |b_2 - b_1| \le L \eta\Bigr\}
 = b_3 + L(2\xi + \eta).
\end{align*}
 
\medskip
Combining the values of $a_4^\star(\bm{v})$ and $b_1^\star(\bm{v})$, we have
\begin{align*}
F(\bm{v})
&= (b_1^\star - a_1) + (b_4 - a_4^\star) \\
&= \big(b_3 + L(2\xi+\eta) - a_1\big)
   + \big(b_4 - (a_2 - L(2\xi+\eta))\big) \\
&= -a_1 - a_2 + b_3 + b_4 + 2L(2\xi + \eta).
\end{align*}
If we define
\begin{align*}
\bm{c} := (-1,-1,1,1)^\top,
\end{align*}
then this can be written compactly as
\begin{align*}
F(\bm{v}) = \bm{c}^\top \bm{v} + 2L(2\xi + \eta).
\end{align*}

From~\eqref{eqn:omega-in-F-knnnk},
\begin{align}
\omega(\delta)
&= \frac{2k}{k+1}
\max_{\bm{v} \in \mathcal{V},\ \bm{v}^\top \bm{Q} \bm{v} \le \delta^2/4}
\bigl( \bm{c}^\top \bm{v} + 2L(2\xi+\eta) \bigr).
\label{eqn:w-in-F-v-step-2-knnnk}
\end{align}

\paragraph{Step 4: Solving the outer problem over $\bm{v}$.}

The constant term $2L(2\xi+\eta)$ does not affect the optimizer, so
we focus on maximizing $\bm{c}^\top \bm{v}$ subject to
\begin{align*}
\bm{v}^\top \bm{Q} \bm{v} \le \frac{\delta^2}{4}, \qquad
\bm{v} \in \mathcal{V},
\end{align*}
where $\mathcal{V}$ is in~\eqref{eqn:mc-V-def-knnnk}.
Let
\begin{align*}
\bm{Q}^{-1} = \diag\Bigl(\frac{1}{kn},\frac{1}{n},\frac{1}{n},\frac{1}{kn}\Bigr).
\end{align*}
Introduce the vectors
\begin{align*}
\bm{h}_1 := (1,-1,0,0)^\top, \qquad
\bm{h}_2 := (0,0,1,-1)^\top,
\end{align*}
so that $|a_1-a_2|\le L\eta$ and $|b_3-b_4|\le L\eta$ can be written as
\begin{align*}
|\bm{h}_1^\top \bm{v}| \le L\eta, \qquad |\bm{h}_2^\top \bm{v}| \le L\eta.
\end{align*}
Define the scalars
\begin{align*}
\gamma &:= \bm{c}^\top \bm{Q}^{-1} \bm{c}
= \frac{2}{n}\Bigl(1+\frac{1}{k}\Bigr), \\
\beta &:= \bm{h}_1^\top \bm{Q}^{-1} \bm{h}_1 = \bm{h}_2^\top \bm{Q}^{-1} \bm{h}_2
= \frac{1}{n}\Bigl(1+\frac{1}{k}\Bigr), \\
\alpha &:= \bm{h}_1^\top \bm{Q}^{-1} \bm{c} = \bm{h}_2^\top \bm{Q}^{-1} \bm{c}
= \frac{1}{n}\Bigl(\frac{k-1}{k}\Bigr).
\end{align*}
Since we assume $k>1$, we have $\alpha>0$.

The quadratic program is
\begin{align}
\max_{\bm{v} \in \mathbb{R}^4} \; \bm{c}^\top \bm{v}
\quad\text{subject to}\quad
\bm{v}^\top \bm{Q} \bm{v} \le \frac{\delta^2}{4},
\qquad
|\bm{h}_1^\top \bm{v}| \le L\eta,\quad |\bm{h}_2^\top \bm{v}| \le L\eta,
\label{eqn:quadratic-in-Q}
\end{align}
with
\(
\bm{Q} = \mathrm{diag}(kn,n,n,kn)
\),
\(\bm{c} = (-1,-1,1,1)^\top\),
and \(\bm{h}_1 = (1,-1,0,0)^\top\), \(\bm{h}_2 = (0,0,1,-1)^\top\).
Using a symmetry argument (as in the main text), we may assume that among
the maximizers there exists $\bm{v}^\star$ such that
\(|\bm{h}_1^\top \bm{v}^\star| = |\bm{h}_2^\top \bm{v}^\star|\).
For such an optimizer, either both Lipschitz constraints are slack
(\(|\bm{h}_1^\top \bm{v}^\star| = |\bm{h}_2^\top \bm{v}^\star| < L\eta\)),
or both are binding (\(|\bm{h}_1^\top \bm{v}^\star| = |\bm{h}_2^\top \bm{v}^\star| = L\eta\)).
This yields the two regimes we analyze below.

\medskip
\emph{Regime~I: No Lipschitz constraints active.}

Ignore the Lipschitz constraints and maximize $\bm{c}^\top \bm{v}$ over the
ellipsoid $\{\bm{v} : \bm{v}^\top \bm{Q} \bm{v} \le \delta^2/4\}$.
Writing
\(\bm{w} = \bm{Q}^{1/2}\bm{v}\) and
\(\bm{u} = \bm{Q}^{-1/2} \bm{c}\), this becomes
\begin{align*}
\max_{\|\bm{w}\| \le \delta/2} \bm{u}^\top \bm{w},
\end{align*}
whose maximizer is attained at
\(\bm{w}^\star = (\delta/2)\,\bm{u}/\|\bm{u}\|\).
Therefore
\begin{align*}
\max_{\bm{v}^\top \bm{Q} \bm{v} \le \delta^2/4} \bm{c}^\top \bm{v}
= \max_{\|\bm{w}\| \le \delta/2} \bm{u}^\top \bm{w}
= \frac{\delta}{2} \|\bm{u}\|
= \frac{\delta}{2} \sqrt{\gamma}.
\end{align*}
The corresponding optimizer (in $\bm{v}$-coordinates) has
\begin{align*}
\bm{h}_1^\top \bm{v}_{\mathrm{unc}}(\delta)
= \bm{h}_2^\top \bm{v}_{\mathrm{unc}}(\delta)
= \frac{\delta}{2} \frac{\alpha}{\sqrt{\gamma}}.
\end{align*}
The Lipschitz constraints remain slack as long as
\(\bigl|\bm{h}_1^\top \bm{v}_{\mathrm{unc}}(\delta)\bigr|\le L\eta\),
i.e.
\begin{align*}
\frac{\delta}{2} \frac{\alpha}{\sqrt{\gamma}} \le L\eta.
\end{align*}
This defines the threshold
\begin{align*}
\delta_c
:= \frac{2L\eta\sqrt{\gamma}}{\alpha}
= \frac{2\sqrt{2nk(k+1)}}{k-1}\,L\eta,
\end{align*}
where we used the explicit expressions for $\gamma$ and $\alpha$ and the
assumption $k>1$.
For $0\le \delta \le \delta_c$ the Lipschitz constraints do not bind and
the unconstrained ellipsoid solution is feasible. Plugging
\(\bm{c}^\top \bm{v}_{\mathrm{unc}}(\delta) = (\delta/2)\sqrt{\gamma}\) into
\eqref{eqn:w-in-F-v-step-2-knnnk} yields
\begin{align}
\omega(\delta)
&= \frac{2k}{k+1}
\Bigl( \frac{\delta}{2}\sqrt{\gamma} + 2L(2\xi+\eta) \Bigr),
\qquad 0 \le \delta \le \delta_c.
\label{eqn:omega-regime-1}
\end{align}
Substituting $\gamma = \frac{2}{n}\bigl(1+\frac{1}{k}\bigr)$ gives the explicit
Regime~I expression.

\medskip
\emph{Regime~II: both Lipschitz constraints active.}
We now analyze the case where the quadratic constraint is loose enough that both
Lipschitz constraints bind at the maximizer. 
From \eqref{eqn:quadratic-in-Q},
we want to maximize $\bm{c}^\top \bm{v}$ subject to
\begin{align*}
\bm{v}^\top \bm{Q} \bm{v} \;\le\; \frac{\delta^2}{4},
\qquad
|\bm{h}_1^\top \bm{v}| \le L\eta,
\qquad
|\bm{h}_2^\top \bm{v}| \le L\eta.
\end{align*}
 By symmetry between the ``left'' and ``right'' coordinates, the maximizer must satisfy
\begin{align*}
\bm{h}_1^\top \bm{v} =
\bm{h}_2^\top \bm{v} = L\eta.
\end{align*}  

\emph{Reformulation.}
Introduce the variables
\begin{align*}
\bm{w} := \bm{Q}^{1/2} \bm{v}, \qquad \bm{u} := \bm{Q}^{-1/2} \bm{c},
\end{align*}
so that
\begin{align*}
\bm{c}^\top \bm{v} = \bm{u}^\top \bm{w}, \qquad \|\bm{w}\|^2 = \bm{v}^\top \bm{Q} \bm{v} \le \frac{\delta^2}{4}.
\end{align*}
Similarly, define
\begin{align*}
\bm{k}_j := \bm{Q}^{-1/2} \bm{h}_j, \quad j=1,2,
\qquad
K := [\,\bm{k}_1\; \bm{k}_2\,] \in \mathbb{R}^{4\times 2},
\qquad
\bm{d}_{\mathrm{vec}} := (L\eta, L \eta)^\top.
\end{align*}
The Lipschitz equalities become
\begin{align*}
\bm{h}_j^\top \bm{v} =  L \eta
\quad\Longleftrightarrow\quad
\bm{k}_j^\top \bm{w} =  L \eta
\quad\Longleftrightarrow\quad
K^\top \bm{w} = \bm{d}_{\mathrm{vec}}.
\end{align*}

Hence the Regime~II problem is equivalent to
\begin{equation}\label{eq:regime2-whitened}
\max_{\bm{w}\in\mathbb{R}^4} \; \bm{u}^\top \bm{w}
\quad\text{s.t.}\quad
\|\bm{w}\|^2 \le \frac{\delta^2}{4},
\;\;
K^\top \bm{w} = \bm{d}_{\mathrm{vec}}.
\end{equation}
This is exactly of the form in Lemma~\ref{lemma:ball-linear}
with $m=4$, $p=2$, $\bm{u}$ as above, $K$ as above, $\bm{d}=\bm{d}_{\mathrm{vec}}$,
and radius $R = \delta/2$.
Lemma~\ref{lemma:ball-linear} therefore yields
\begin{equation}\label{eq:regime2-master}
\max \bm{u}^\top \bm{w}
= \bm{b}^\top A^{-1} \bm{d}_{\mathrm{vec}}
+ \sqrt{\gamma - \bm{b}^\top A^{-1} \bm{b}}\;
  \sqrt{ \frac{\delta^2}{4} - \bm{d}_{\mathrm{vec}}^\top A^{-1} \bm{d}_{\mathrm{vec}} },
\end{equation}
where
\begin{align*}
A := K^\top K,\qquad
\bm{b} := K^\top \bm{u},\qquad
\gamma := \bm{u}^\top \bm{u} = \bm{c}^\top \bm{Q}^{-1}\bm{c}.
\end{align*}

\emph{Explicit evaluation.}
From the definitions of $\bm{h}_1,\bm{h}_2,\bm{Q}$ and $\bm{u}$, a direct computation shows
\begin{align*}
A
=
\begin{pmatrix}
\beta & 0\\
0 & \beta
\end{pmatrix},
\qquad
\bm{b} =
\begin{pmatrix}
\alpha\\
\alpha
\end{pmatrix},
\end{align*}
with
\begin{align*}
\gamma = \frac{2}{n}\Bigl(1+\frac{1}{k}\Bigr),\qquad
\beta = \frac{1}{n}\Bigl(1+\frac{1}{k}\Bigr),\qquad
\alpha = \frac{1}{n}\Bigl(\frac{k-1}{k}\Bigr).
\end{align*}
Using this structure,
\begin{align*}
\bm{b}^\top A^{-1} \bm{d}_{\mathrm{vec}}
= \frac{2\alpha  L \eta}{\beta}
= \frac{2(k-1)}{k+1}L\eta,
\end{align*}
\begin{align*}
\bm{b}^\top A^{-1} \bm{b}
= \frac{2\alpha^2}{\beta}
= \frac{2(k-1)^2}{k n (k+1)},
\qquad
\bm{d}_{\mathrm{vec}}^\top A^{-1} \bm{d}_{\mathrm{vec}}
= \frac{2(L \eta)^2}{\beta}
= \frac{2kn}{k+1}L^2\eta^2.
\end{align*}
Moreover
\begin{align*}
\gamma - \frac{2\alpha^2}{\beta}
= \frac{8}{n(k+1)}.
\end{align*}
Plugging these into \eqref{eq:regime2-master} yields
\begin{align*}
\max \bm{u}^\top \bm{w}
= \frac{2(k-1)}{k+1}L\eta
+ \sqrt{\frac{8}{n(k+1)}}\;
\sqrt{\frac{\delta^2}{4} - \frac{2kn}{k+1}L^2\eta^2}.
\end{align*}
Recalling that $\bm{c}^\top \bm{v} = \bm{u}^\top \bm{w}$ 
and substituting the above equation back into \eqref{eqn:w-in-F-v-step-2-knnnk} gives
the Regime~II modulus
\begin{align}
\omega(\delta)
= \frac{2k}{k+1}\Biggl(
  2L(2\xi+\eta)
  + \frac{2(k-1)}{k+1} L\eta
  + \sqrt{\frac{8}{n(k+1)}}
    \sqrt{\frac{\delta^2}{4} - \frac{2kn}{k+1}L^2\eta^2}
  \Biggr),
\qquad \delta \ge \delta_c.
\label{eqn:omega-regime-2}
\end{align}

\medskip
Combining the expressions from Regime~I~\eqref{eqn:omega-regime-1} and Regime~II~\eqref{eqn:omega-regime-2}, we obtain the claimed
piecewise closed-form expression for the modulus $\omega(\delta)$ with
weights $(n_1,n_2,n_3,n_4) = (kn,n,n,kn)$.

Using~\eqref{eqn:max-bias-equiv-formula}, we have  $\maxbias(\hat{\tau}_\delta) 
= \half(\omega(\delta) - \delta \omega'(\delta))$.
In addition, from~\eqref{eqn:std-formula},
$\sd(\hat{\tau}_\delta) = \omega'(\delta)$, and we 
finish the proof.
\end{proof}

\subsection{Derivation of the confidence sequence}
\label{section:confidence-sequence-proof}

\begin{lemma} 
\label{lemma:gaussian-conf-seq-cv}
Fix $\alpha \in (0,1)$ and a sequence of nonnegative numbers
$\{\alpha_t\}_{t\ge 1}$ such that
\begin{align*}
  \sum_{t=1}^\infty \alpha_t \le \alpha.
\end{align*}
For each $t \ge 1$, let $\tau_t$ denote the target estimand and  write
\begin{align*}
  \hl_t
  = \tau_t + \bias_t(f) + \sd(\hl_t)\,Z_t,
\end{align*}
where $Z_t \sim \mathcal N(0,1)$ and the bias satisfies
\begin{align*}
  \bigl|\bias_t(f)\bigr|
  \le \maxbias(\hl_t)
  \quad\text{for all } f \in \mc{F}.
\end{align*}
Define the standardized worst-case bias
\begin{align*}
  y_t := \frac{\maxbias(\hl_t)}{\sd(\hl_t)}.
\end{align*}
Let $\cv_{\alpha}(z)$ be    the $(1-\alpha)$-quantile of $\lvert N(z,1)\rvert$.
Consider the confidence sequence
\begin{align}
  \mc{C}_t
  := \hl_t \pm
     \sd(\hl_t)\,
     \cv_{\alpha_t}\bigl(y_t\bigr),
  \label{eqn:conf-seq-cv}
\end{align}
for $t = 1,2,\dots$.
Then
\begin{align*}
  \inf_{f \in \mc{F}}
  \P_f\bigl(\forall t \ge 1:\ \tau_t \in \mc{C}_t\bigr)
  \ge 1 - \alpha.
\end{align*}
In particular, the choice
\begin{align*}
  \alpha_t := \frac{6\alpha}{\pi^2 t^2}
\end{align*}
satisfies $\sum_{t=1}^\infty \alpha_t = \alpha$, and therefore
$\{\mc{C}_t\}_{t\ge 1}$ in \eqref{eqn:conf-seq-cv} is an always-valid
$(1-\alpha)$ confidence sequence.
\end{lemma}

\begin{proof}[Proof of Lemma~\ref{lemma:gaussian-conf-seq-cv}]
Fix an arbitrary data-generating function $f$ and write $\P_f$ for the
corresponding probability measure.
We will show that
\begin{align*}
  \P_f\bigl(\forall t \ge 1:\ \tau_t \in \mc{C}_t\bigr)
  \ge 1 - \alpha,
\end{align*}
and since $f$ was arbitrary this implies the claimed uniform bound over $f$.

For each $t\ge 1$, define the failure event
\begin{align*}
  F_t
  := \bigl\{\tau_t \notin \mc{C}_t\bigr\}
  = \Bigl\{
      \bigl|\hl_t - \tau_t\bigr|
      > \sd(\hl_t)\,\cv_{\alpha_t}(y_t)
    \Bigr\}.
\end{align*}
By the decomposition
\begin{align*}
  \hl_t - \tau_t
  = \bias_t(f) + \sd(\hl_t)\,Z_t,
\end{align*}
on $F_t$ we have
\begin{align*}
  \bigl|\bias_t(f) + \sd(\hl_t)\,Z_t\bigr|
  > \sd(\hl_t)\,\cv_{\alpha_t}(y_t).
\end{align*}
Dividing by $\sd(\hl_t)$ and setting
\begin{align*}
  \mu_t(f) := \frac{\bias_t(f)}{\sd(\hl_t)},
\end{align*}
we can write
\begin{align*}
  F_t
  = \Bigl\{
      \bigl|Z_t + \mu_t(f)\bigr|
      > \cv_{\alpha_t}(y_t)
    \Bigr\},
  \qquad
  \bigl|\mu_t(f)\bigr|
  \le y_t,
\end{align*}
where the last inequality uses the bias bound
$\bigl|\bias_t(f)\bigr|\le \maxbias(\hl_t)$.

By the definition of $\cv_{\alpha_t}(z)$,
for every $z\ge 0$ and every $|\mu|\le z$,
\begin{align*}
  \P\bigl(\lvert N(\mu,1)\rvert > \cv_{\alpha_t}(z)\bigr)
  \le \alpha_t.
\end{align*}
Applying this with $z = y_t$ and $\mu = \mu_t(f)$ shows that
\begin{align}
  \P_f(F_t)
  = \P_f\Bigl(
        \bigl|Z_t + \mu_t(f)\bigr|
        > \cv_{\alpha_t}(y_t)
      \Bigr)
  \le \alpha_t
  \quad\text{for all } t \ge 1.
  \label{eq:per-time-alpha}
\end{align}

Now consider the event that the confidence sequence ever fails:
\begin{align*}
  F
  := \bigl\{\exists\,t\ge 1:\ \tau_t \notin \mc{C}_t\bigr\}
  = \bigcup_{t=1}^\infty F_t.
\end{align*}
By the union bound and \eqref{eq:per-time-alpha},
\begin{align*}
\P_f(F)
= \P_f\Bigl( \bigcup_{t=1}^\infty F_t \Bigr) 
\le \sum_{t=1}^\infty \P_f(F_t)
\le \sum_{t=1}^\infty \alpha_t 
\le \alpha.
\end{align*}
Therefore
\begin{align*}
\P_f\bigl(\forall t\ge 1:\ \tau_t \in \mc{C}_t\bigr)
= 1 - \P_f(F)
\ge 1 - \alpha.
\end{align*}
Since this holds for every $f$, taking the infimum over $f$ completes the proof.
\end{proof}

\section{Experimental details}
\label{sec:experimental-details}

\subsection{Details of Example~\ref{example:simulation}}
\label{sec:experiment-sim}

First, we provide additional details and analysis of
Example~\ref{example:simulation}.  We let $\P(Z=1|X=x) = \pi(x)$ given
in~\eqref{eqn:sim-prop}, and we also set the true outcome function to be
$f(x,z) = F(x) + \I{z = 1} \cdot h(x)$ given by~\eqref{eqn:sim-ate}.  Both
$\pi(x)$ and $f(x,z)$ are parameterized by two parameters: $\overlap$ and $\overrange $ that
govern the size of the limited overlap region.  
A smaller $\overlap$ and $\overrange$ means
the limited overlap region is smaller, while $H$ controls the value of the ATE
$\tau$.  Given the true outcome function, we generate data
$y_i= f(x_i,z_i) + \epsilon_i$ where $\epsilon_i \sim N(0,\sigma^2)$ with
$\sigma = 0.06$.  Despite the low noise, we see the unreliability of typical
asymptotic estimators.  We take $\overlap = 0.05, \overrange = 0.05,H=0.25$ unless
specified
otherwise. 

\begin{align}
\pi(x) &= 
\begin{cases} 
\frac{\overlap}{\overrange }x, & \text{if } 0 \leq x \leq \overrange, \\
-\frac{\overlap}{\overrange }(x-\overrange) + (1-\overlap), & \text{if } \overrange 
< x \leq 2\overrange, \\
\frac{1-2\overlap}{1-2\overrange}(x-2\overrange) + \overlap, & \text{if } x > 2\overrange,
\end{cases}\label{eqn:sim-prop}
\end{align}
\begin{align}
h(x) = 8H\left(x - \frac{1}{2}\right)^2,
 ~~
F(x) &= 
\begin{cases} 
\frac{4H}{\overrange^2}\left(x - \frac{\overrange}{2}\right)^2 + H, & \text{if } 0 \leq x \leq \frac{\overrange}{2}, \\
-\frac{4H}{\overrange^2}x(x - \overrange), & \text{if } \frac{\overrange}{2} < x \leq \overrange, \\
32H\left(x - \overrange\right)\left(x - \overrange - \frac{1}{2}\right), & \text{if } \overrange < x \leq \overrange+\frac{1}{2}, \\
-\frac{16H}{(2\overrange -1)^2}\left(x - \overrange - \frac{1}{2}\right)\left(x - 1\right), & \text{if } \overrange +\frac{1}{2} < x \leq 1. \\
\end{cases} \label{eqn:sim-ate}
\end{align}

We simulate $n=1000$ points of $X$ according to the 
data-generating process
described above.  As discussed, for $\AIPWP$, after truncation, one needs to
analyze the ATE in the non-overlap region
$\tau_-$~\eqref{eqn:ATE-decomposition}, and our minimax approach $\MP$ can
build a minimax CI that guarantees coverage of $\tau_-$.  On the other hand,
as we see earlier in Figure~\ref{fig:simulation-results}, the naive minimax
approach $\M$ does not fully exploit the data in the overlap region and
produces an overly conservative interval that is driven by the worst-case bias
of the ATE applied on the entire data.

\subsection{Details of the case study in Section~\ref{section:experiment} }

We first present a lemma that allows us to 
generate an observational dataset given a propensity score function 
from a randomized controlled trial (RCT) dataset.
We treat the ATE estimated from the RCT dataset as the ground truth treatment effect
and generate observational datasets using the approach described in
Lemma~\ref{lemma:rct-sampling}.

\begin{lemma} \label{lemma:rct-sampling}
Let the dataset be $\mc{D} = \set{x_i,y_i,z_i}_{i=1}^n$ and assume $z_i \sim \Ber(p)$ for some $p \in (0,1)$.
For every $i$, let $C_i \sim \Ber(\pi(x_i))$.
Define a random variable $O$ as follows:
if $p \le \half$, then
\begin{align*}
O_i \defeq \begin{cases}
1 & \text{ if } z_i= 1, \\ 
\Ber(\frac{p}{1-p}) & \text{ if } z_i= 0, 
     \end{cases}
\end{align*} 
and if $p > \half$, let
\begin{align*}
O_i \defeq \begin{cases}
1 & \text{ if } z_i= 0, \\ 
\Ber(\frac{1-p}{p}) & \text{ if } z_i= 1.
\end{cases}
\end{align*} 
And let $I_i = \I{z_i = C_i \text{ and } O_i = 1}$. We define another dataset $\mc{D}'$ as 
\begin{align*}
\mc{D}' \defeq \set{x_i,y_i,z_i}_{i \in \set{j \in [n]: I_j =  1}},
\end{align*}
then 
$P(z_i=1 | I_i = 1)=\pi(x_i)$ so that
$\mc{D}'$ is an observational 
dataset with propensity score $\pi$.
In addition, we have 
$\P(I_i=1) = \min\set{p,1-p}$.
%which does not depend on $i$, implying $\mc{D}(X) \deq \mc{D}'(X)$. 
\end{lemma}

\begin{proof}[Proof of Lemma~\ref{lemma:rct-sampling}]
We  assume $p \le \half$ (the other case can be proved similarly).
Then
we have
\begin{align*}
\P(I_i = 1| z_i = 1) &= \P(O_i=1 | z_i=1) \cdot \P(C_i = 1)= \pi(x_i),  \\
\P(I_i = 1| z_i = 0) &= \P(O_i=1 | z_i=0) \cdot \P(C_i = 0) =  \frac{p}{1-p} (1-\pi(x_i)), 
\end{align*}
so that
\begin{align*}
    \P(I_i=1) = 
    \P(I_i = 1 | z_i=1) \P(z_i=1) 
    + 
    \P(I_i = 1 | z_i=0) \P(z_i=0)
    =p \pi(x_i) + (1-p) \frac{p}{1-p} (1-\pi(x_i))
    =p,
\end{align*}
\begin{align*}
\P(z_i = 1|I_i=1) = \frac{\P(I_i = 1 | z_i=1) \P(z_i=1)}{\P(I_i=1)}
=  \frac{p \pi(x_i)}{p} = \pi(x_i),
\end{align*} and we finish the proof.
\end{proof}

We now describe the propensity score we use. We first use the entire RCT dataset to train a   \textsf{lightgbm} regressor
to estimate  the conditional treatment effect $\tilde{\tau}(x)$, then construct the propensity score $\pi(x)$ based on the quantile of $\tilde{\tau}(x)$, see~\eqref{eqn:RCT-prop-E-k-def}.

%\ym{TBD: RCT multiple run results}

% \subsection{Other results for methods based on~\citep{CrumpHoImMi06}}

%%% Local Variables:
%%% mode: latex
%%% TeX-master: "main"
%%% End:

 \subsection{Experimental results for continual sampling}
 \label{sec:experimenta-results-continual-sampling}

We use the same data as in Section~\ref{sec:data-collection-two-option}.
 Using the confidence sequence~\eqref{eqn:conf-seq}, we simulate the confidence sequence assuming the analyst can sample 6 epochs of new data \\ 
 ($x  \in (0.40, 0.47),\ (0.47, 0.53),\ (0.53, 0.60),\ (0, 0.03),\ (0.03, 0.07),\ (0.07, 0.10)$)
 in the non-overlap region with a fixed truncation threshold and $L = 5.48$. 
From Figure~\ref{fig:sim-conf-seq}, 
we observe that the length of the confidence interval decreases and 
that the interval becomes closer to the origin as the number of iterations increases, indicating that the uncertainty over the non-overlap region gradually decreases.

 \begin{figure}\centering
\begin{minipage}[b]{0.49\textwidth}
\centering
\includegraphics[width=\textwidth, height=6cm]{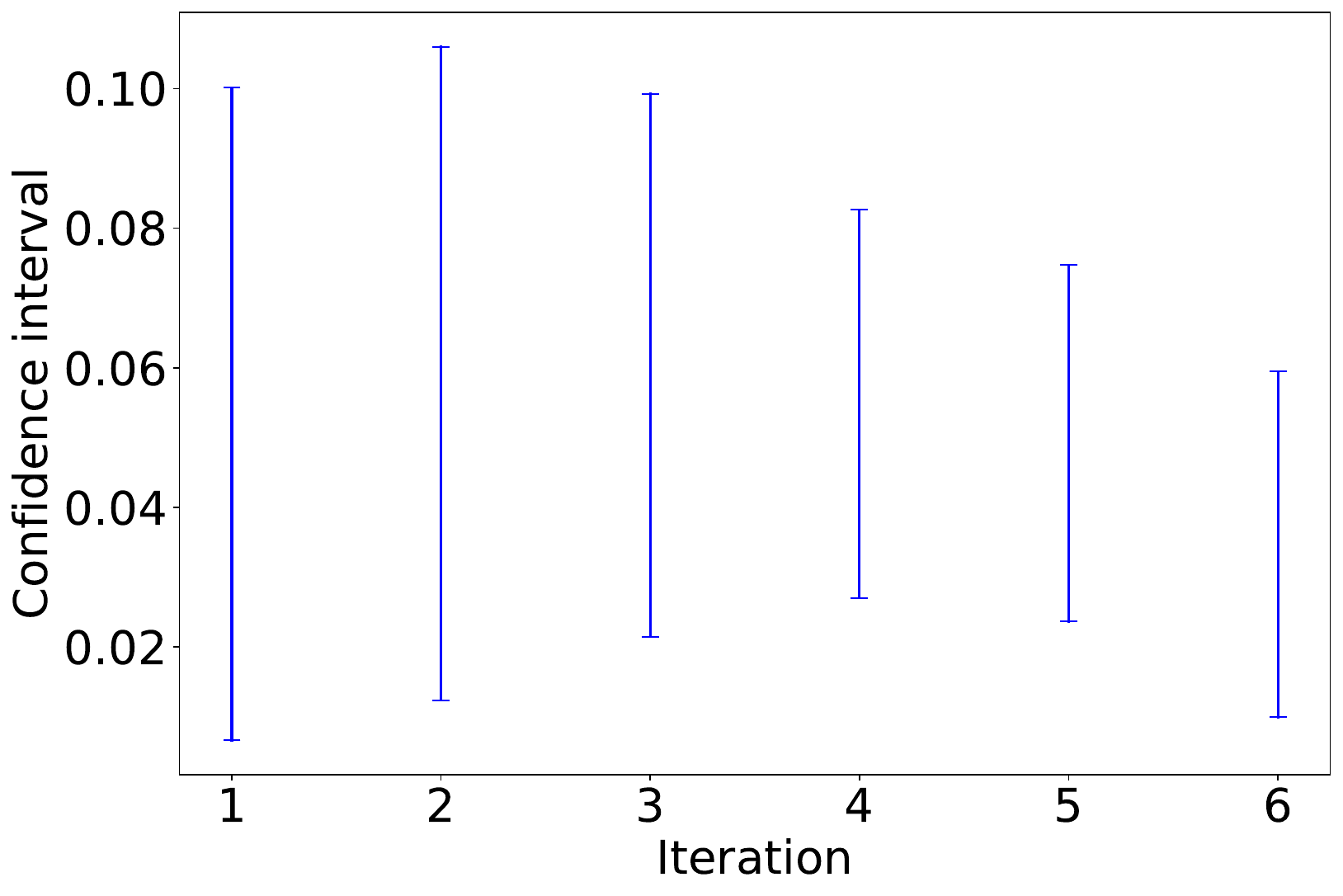}
\end{minipage}
\hfill
\begin{minipage}[b]{0.49\textwidth}
\centering \includegraphics[width=\textwidth, height=6cm]{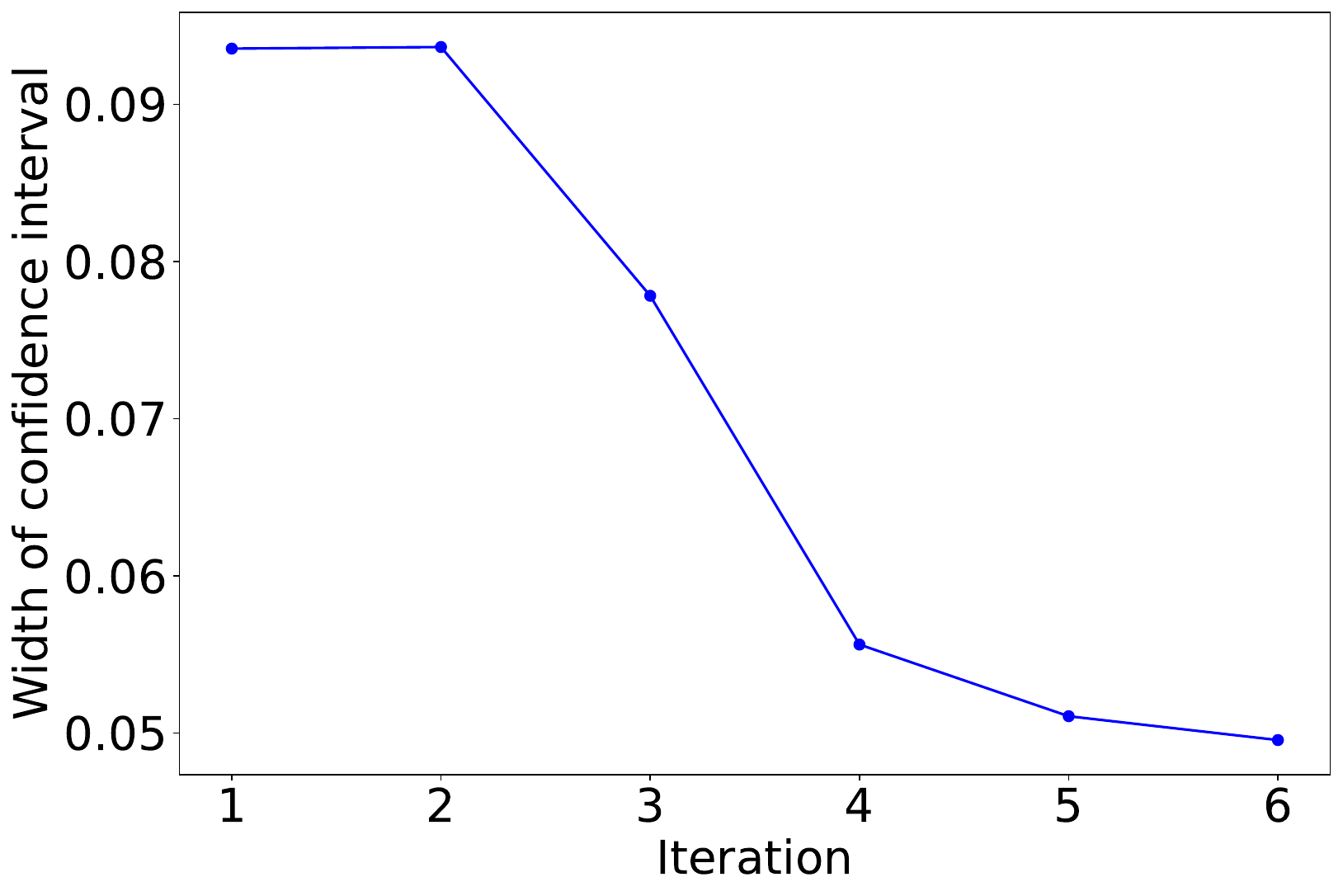}
\end{minipage}  
\caption{A confidence sequence generated during continual sampling}
\label{fig:sim-conf-seq}
\end{figure}

%%% Local Variables:
%%% mode: latex
%%% TeX-master: "main"
%%% End:

%include appendix sections

\end{document}